\documentclass[12pt,english,twoside]{report}
\usepackage{mathptmx}

\usepackage[T1]{fontenc}
\usepackage[latin9]{inputenc}
\usepackage[a4paper]{geometry}
\usepackage{verbatim}
\usepackage{pdfpages}
\usepackage{graphicx}

\usepackage{subcaption} 

\usepackage{setspace}
\usepackage{arabtex}
\usepackage[numbers]{natbib}
\usepackage{nomencl}
\usepackage{paralist}
\usepackage{color}
\usepackage{amsthm}
\usepackage[cmex10]{amsmath}
\usepackage{bbold}
\usepackage{multirow}
\usepackage{dsfont}
\usepackage{booktabs}
\usepackage{eufrak}
\usepackage{amssymb}

\usepackage{float}
\usepackage{apxproof}
\usepackage{bbm}
\usepackage{algorithm}
\usepackage{algpseudocode}
\usepackage{hyperref}
\usepackage{algpseudocode}
\usepackage{array, makecell}
\usepackage{mathtools}
\usepackage{amsmath}

\usepackage{amsmath,amsfonts,bm}

\usepackage{soul}

\def\thetagenie{\hat{\theta}(\mathcal{D}_N;x,y)}
\def\probthetagenie{p_{\thetagenie}(y|x)}
\def\probthetagenietag{p_{\thetagenietag}(y'|x)}
\def\thetagenie{\hat{\theta}(\mathcal{D}_N;x,y)}
\def\thetagenietag{\hat{\theta}(\mathcal{D}_N;x,y')}
\def\thetageniein#1{\hat{\theta}(\mathcal{D}_N;x,y=#1)}
\def\probthetagenie{p_{\thetagenie}(y|x)}
  
\newcommand{\norm}[1]{\left\lVert#1\right\rVert}
\def\Theoref#1{Theorem~\ref{#1}}
\def\appref#1{appendix~\ref{#1}}

\def\lemmaref#1{lemma~\ref{#1}}
\def\Lemmaref#1{Lemma~\ref{#1}}
\def\tableref#1{table~\ref{#1}}
\def\Tableref#1{Table~\ref{#1}}

\def\DN{\mathcal{D}_N} 
\def\Plamb{P_\lambda}
\def\Klamb{K_\lambda}

\def\Klpnml{K_\textit{LpNML}}
\def\Klpnml2{K_\textit{LpNML2}}
\def\mulpnml{\hat{\mu}_\textit{\tiny{LpNML}}}
\def\sigmalpnml{\hat{\sigma}^2_\textit{\tiny{LpNML}}}

\def\thetay{\hat{\theta}_{y}}
\def\thetaytag{\hat{\theta}_{y'}}
\def\probthetay{p_{\thetay}(y|x)}
\def\probthetaytag{p_{\thetaytag}(y'|x)}
\def\thetahat{\hat{\theta}}
\def\thetalamb{\hat{\theta}_\lambda}

\newcommand{\bnorm}[1]{\big\lVert#1\big\rVert}
\definecolor{green}{rgb}{0.05, 0.5, 0.06}
\definecolor{pink}{rgb}{0.91, 0.33, 0.5}

\def\Dc{\mathcal{D}_c}
\def\Du{\mathcal{D}_u}

\algnewcommand\algorithmicforeach{\textbf{for each}}
\algdef{S}[FOR]{ForEach}[1]{\algorithmicforeach\ #1\ \algorithmicdo}

\def\figref#1{figure~\ref{#1}}
\def\Figref#1{Figure~\ref{#1}}

\def\secref#1{section~\ref{#1}}

\def\eqref#1{(\ref{#1})}

\def\1{\bm{1}}

\def\eps{{\epsilon}}

\DeclareMathAlphabet{\mathsfit}{\encodingdefault}{\sfdefault}{m}{sl}
\SetMathAlphabet{\mathsfit}{bold}{\encodingdefault}{\sfdefault}{bx}{n}

\newcommand{\normlone}{L^1}
\newcommand{\normltwo}{L^2}

\DeclareMathOperator*{\argmax}{arg\,max}
\DeclareMathOperator*{\argmin}{arg\,min}

\newcommand\tabref{table~\ref}
\newcommand\Tabref{Table~\ref}

\newtheorem{theorem}{Theorem}[section]
\newtheorem{lemma}[theorem]{Lemma}

\theoremstyle{definition}

\theoremstyle{remark}

\usepackage{etoolbox}
\newtoggle{edit-mode}
\togglefalse{edit-mode}  

\providecommand{\printnomenclature}{\printglossary}
\providecommand{\makenomenclature}{\makeglossary}
\makenomenclature
\makeatletter

\usepackage{tauthesis}
\iftoggle{edit-mode}{
\geometry{verbose,tmargin=2cm,bmargin=2cm,lmargin=2cm,rmargin=6cm,headheight=1cm,headsep=1cm,footskip=1cm, marginparwidth=5cm}
}

\usepackage[font={small,bf}, labelfont={small,bf}, margin=1cm]{caption}
\usepackage{titlesec}
\newcommand{\hsp}{\hspace{20pt}}
\titleformat{\chapter}[hang]{\Huge\bfseries}{\thechapter\hsp}{0pt}{\Huge\bfseries}

\University{Tel Aviv University}
\Faculty{The Iby and Aladar Fleischman Faculty of Engineering}
\School{The Zandman-Slaner School of Graduate Studies}
\Department {Department of Electrical Engineering - Systems}
\Title{\textbf{ \uppercase {Quantifying the Prediction Uncertainty of Machine Learning Models for Individual Data}}}
\Degree{Doctor of Philosophy}
\Author{Koby Bibas}
\Year{\today}
\Supervisor{Prof. Meir Feder}
\SecondSupervisor{<Dr. second Supervisor>}

\makeatother

\usepackage{babel}

\begin{document}

\pagenumbering{gobble}

\titlepage

\secondtitlepage

\acknowledgments{
\thispagestyle{empty}

Over the course of the past five years, I have had the privilege of embarking on an incredible journey that has left an indelible mark on my life. Through this period, I have been blessed to have had the guidance and support of Meir Feder, who has been my supervisor and mentor throughout my journey.
Meir's open-mindedness, innovative and visionary approach has enabled me to take this work to new heights and achieve milestones that I never thought were possible. I cannot express enough gratitude for Meir's belief in me, which has allowed me to grow and develop as a researcher. He has been a source of inspiration and a constant reminder that anything is possible with hard work and dedication.

I would also like to extend my gratitude to Tal Hassner, who has been instrumental in guiding me on how to apply my research to the computer vision field. Our discussions have been incredibly insightful and have provided me with new perspectives that have helped me to approach my research from different angles.

Furthermore, I owe a debt of gratitude to my colleagues and friends in the 102 lab: Yaniv Fogel, Shahar Shayovitz, Uriya Pess, Ido Lublinsky, and Assaf Daouber. They have been an integral part of my journey and not only provided me with a sounding board for my ideas, but have also made this period fun and exciting.

At the conclusion of this journey, it is with a heart full of gratitude that I express my deepest appreciation to those who have been my pillars of strength. My parents and my wife, Batzi, have been my unwavering support system, providing me with endless love and encouragement that have been the bedrock of my success. 

}

\cleardoublepage
\pagenumbering{roman}
\setcounter{page}{1} 

\chapter*{Abstract}
\vspace{-40pt}
Machine learning models have exhibited exceptional results in various domains. The most prevalent approach for learning is the empirical risk minimizer (ERM), which adapts the model's weights to reduce the loss on a training set and subsequently leverages these weights to predict the label for new test data. Nonetheless, ERM makes the assumption that the test distribution is similar to the training distribution, which may not always hold in real-world situations.
In contrast, the predictive normalized maximum likelihood (pNML) was proposed as a min-max solution for the individual setting where no assumptions are made on the distribution of the tested input. 
This study investigates pNML's learnability for linear regression and neural networks, and demonstrates that pNML can improve the performance and robustness of these models on various tasks. 
Moreover, the pNML provides an accurate confidence measure for its output, showcasing state-of-the-art results for out-of-distribution detection, resistance to adversarial attacks, and active learning.

\setcounter{tocdepth}{1}
\tableofcontents{}

\cleardoublepage

\addcontentsline{toc}{chapter}{Nomenclature}

\printnomenclature

\cleardoublepage

\listoffigures

\cleardoublepage

\listoftables

\textpages

\nomenclature{MSE}{Mean Square Error}
\nomenclature{RLS}{Recursive Least Square}
\nomenclature{OOD}{Out-of-Distribution}
\nomenclature{IND}{In-Distribution}
\nomenclature{CCR}{Correct Classification Rate}
\nomenclature{OSCR}{Open-Set Classification Rate}
\nomenclature{NN}{Neural Network}
\nomenclature{ML}{Machine Learning}
\nomenclature{pNML}{Predictive Normalized Maximum Likelihood}
\nomenclature{PAC}{Probably Approximately Correct}
\nomenclature{VC}{Vapnik-Chervonenkis}
\nomenclature{DNN}{Deep Neural Network}
\nomenclature{MN}{Minimum Norm}
\nomenclature{PAC}{Probably Approximately Correct}
\nomenclature{Lpnml}{Predictive Normalized Maximum Likelihood with Luckiness}
\nomenclature{ERM}{Empirical Risk Minimizer}

\chapter{Introduction}

Machine learning (ML) models possess the remarkable capacity to learn from experience without the requirement of explicit programming. Among the prevalent techniques of ML, supervised batch learning represents a quintessential approach where the experience is represented by a training set. The training set encompasses $N$ examples, which are represented as $\mathcal{D}_N=\{(x_n,y_n)\}_{n=1}^N$. Here, $x_n \in \mathcal{X}$ denotes the $n$-th feature vector and $y_n \in \mathcal{Y}$ signifies its corresponding label. 
The objective of the learning task is to come up with algorithm that, faced with a new test feature $x$, predicts its corresponding label $y$. The prediction is typically represented as a probability assignment $q(\cdot|x)$.

The prediction performance is evaluated using a loss function $\ell(q;x,y)$ which is sometimes termed the generalization error. 
In our information theoretic analysis, we consider the log-loss function~\cite{Fogel2017,Fogel2018,merhav1998universal} 
\begin{equation}
 \ell(q;x,y) = - \log q(y|x).
\end{equation}
Minimizing this loss is an ill-posed problem unless additional assumptions are made regarding the class of possible models (hypotheses) that are used to find the relation between $x$ and $y$.
Denote $\Theta$ as a general index set, the hypothesis class is a set of conditional probability distributions $P_\Theta = \{ p_\theta(y|x),\;\;\theta\in\Theta\}$.

The solution to the learning problem also depends on the assumptions made regarding how the data is generated.
In the stochastic setting, it is assumed that there is a true probabilistic relation between $x$ and $y$ given by an (unknown) model from the class $P_\Theta$. 
A more general setting is the Probably Approximately Correct (PAC)~\cite{valiant1984theory}. 
In PAC setting, $x$ and $y$ are assumed to be generated by some source $P(x,y)=P(x)P(y|x)$, but unlike the stochastic setting, $P(y|x)$ is not necessarily a member of the hypothesis class.
In the PAC setting there are several suggested measures of the class learnability such as the Vapnik-Chervonenkis (VC) dimension \cite{VC}. These measures successfully explain the learnability of rather simple hypotheses classes, yet they fail to explain the learnability of modern hypothesis classes such as deep neural networks (DNN)~\cite{DBLP:conf/iclr/ZhangBHRV17}. 


In this study, we focus on \textit{the individual setting}, which is deemed the most general setting. This setting entails no assumptions regarding the way the examples are generated: both the training set and the test sample are specific individual values, which can potentially be determined by an adversary. In this setting, the typical approach is to employ a learner capable of competing reasonably with a reference learner, known as a genie. This genie is characterized by three key properties: (i) it has knowledge of the test label value, (ii) it is constrained to employ a model from the given hypotheses set $P_\Theta$, and (iii) it has no knowledge of which of the samples is the test. The genie selects a model that achieves the minimum loss over both the training set and the test sample
\begin{equation} \label{eq:genie_pnml} 
\hat{\theta}(\mathcal{D}_N;x,y)  = \arg\min_{\theta \in \Theta} \left[\ell\left(p_\theta;x,y\right) +\sum_{n=1}^N \ell\left(p_\theta;x_n,y_n\right) \right].
\end{equation}
The generalization error, which is referred to as the regret, can be described as the difference in log-loss between the learner $q$ and the genie
\begin{equation} \label{eq:regret}
R(q;\mathcal{D}_N;x,y) = - \log q(y|x) - \left[ -\log p_{\hat{\theta}(\mathcal{D}_N;x,y)}(y|x) \right].
\end{equation} 

The \textit{predictive normalized maximum likelihood} (pNML)~\cite{Fogel2018} has been proposed as the optimal solution for min-max regret, where the minimum is over the learner choice and the maximum is for any possible outcome.
\begin{equation} \label{eq:minmax_regret_objective}
\Gamma = R^*(\mathcal{D}_N,x) = \min_q \max_{y \in \mathcal{Y}} R(q;\mathcal{D}_N;x,y).
\end{equation}
The pNML probability assignment and regret are
\begin{equation} \label{eq:pnml} 
q_{\mbox{\tiny{pNML}}}(y|x)= \frac{\probthetagenie}{\sum_{y' \in \mathcal{Y}} \probthetagenietag},
\qquad
\Gamma = \log \sum_{y' \in \mathcal{Y}} \probthetagenietag.
\end{equation}
The solution can be established by recognizing that the regret remains constant across all choices of $y$. Should a different probability assignment be considered, it must assign a smaller probability to at least one of the outcomes. Since the true label is with individual value (and can be determined by adversary), if the true label is set to one of those outcomes it leads to a higher regret.
The pNML min-max regret can be used as an information-theoretic learnability measure that depends on the specific training set $\mathcal{D}_N$ and the specific test features $x$ and may let the learner {\em know when it does not know.}.

In this thesis, the pNML and its associated learnability measure are examined for linear regression and neural networks hypothesis classes along with various applications to the pNML learnability measure.
Chapter~\ref{chap:ordinary_least_squares_regression} shows an explicit analytical expression of the pNML for the linear regression hypothesis class. The analysis conducted in this chapter highlights that the pNML regret is minimal when the test sample resides in the subspace that corresponds to the large eigenvalues of the empirical correlation matrix of the training data. This finding suggests that even in situations where the model is over-parameterized, i.e., the number of parameters exceeds the training sample size, successful generalization can occur if the test data originates from a ``learnable space''.

In Chapter~\ref{chap:overparameterized_linear_regression}, an upper bound for the pNML regret is derived for over-parameterized linear regression: We construct a hypothesis set that comprises hypotheses whose norm is not greater than the minimum norm (MN) solution norm. Utilizing this hypothesis set, the pNML prediction coincides with the MN solution. As a result, the derived regret upper bound can serve as a measure of the generalization error of the MN solution. 

In Chapter~\ref{chap:luckiness}, we review the luckiness concepts~\cite{shtar1987universal}: A luckiness function $w(\theta)$ is designed such that on sequences with small $w(\theta)$, we are prepared to incur large regret~\cite{grunwald2007minimum}. 
We design a luckiness function to be proportional to the $\ell_2$ linear regression model norm and derive the corresponding pNML. We show its prediction differs from the ridge regression: When the test sample lies within the subspace associated with the small eigenvalues of the empirical correlation matrix of the training data, the prediction is shifted toward 0.

Next, we focus on the neural network hypothesis class. In Chapter~\ref{chap:neural_networks}, we present an analytical solution of the pNML learner and derive its min-max regret for a single layer neural network. We analyze the obtained regret and demonstrate that it achieves low values under two conditions: when the test input either (i) resides in a subspace related to the larger eigenvalues of the empirical correlation matrix of the training data or (ii) is located far from the decision boundary. The results are applicable to the last layer of DNNs without altering the network architecture or the training process. 

Chapter~\ref{chap:adversarial_defense} investigates the use of the pNML learner as a defense mechanism against adversarial attacks. The pNML approach is suitable for this scenario since it does not assume any data generation process, making it resilient to adversarial manipulations of the input. We introduce the Adversarial pNML scheme as an adversarial defense method which restricts the genie learner to perform minor adjustments to the adversarial examples: Each member of the class assumes a different label for the adversarial test sample and performs a targeted adversarial attack accordingly. We then compare the resulting hypothesis probabilities to predict the true label.

Finally, in Chapter~\ref{chap:active_learning}, we present a variant of pNML for active learning, where training samples are selectively chosen to minimize the regret of the test set. In addition, we provide an approximate version of this criterion that enables faster inference for DNNs. We demonstrate that for the same accuracy level our criterion needs 25\% less labeled samples compared to recent leading methods.


\chapter{Ordinary Least Squares Regression}  \label{chap:ordinary_least_squares_regression}
\section{Introduction} \label{sec:linear_regression:introduction}
Linear regression is a statistical technique that is widely used for modeling the relationship between a dependent variable and one or more independent variables~\cite{lawson1995solving}. It is a simple and powerful method for predicting numerical outcomes and understanding the nature of the relationship between variables. 
The empirical risk minimizer (ERM) is a standard learner used in linear regression~\cite{vapnik1992principles}.
In ERM, given a training set and hypothesis class $\{p_\theta(y|x),\ \theta \in \Theta\}$, a learner that minimizes the loss over the training set is chosen:
\begin{equation}
q_{\mbox{\tiny{ERM}}}(y|x) = \underset{p_\theta}{\textit{argmin }} \frac{1}{N}\sum_{i=1}^{N}  \mathcal{L}(p_\theta; x_i, y_i).
\end{equation}

One of the fundamental assumptions of linear regression is that the number of training samples must be greater than the number of features in order to achieve accurate generalization~\cite{james2013introduction}.
However, recent advancements in deep neural networks (DNNs) have challenged this long-held assumption. DNNs are powerful machine learning models that can handle complex patterns in large datasets, and they have achieved remarkable success in various applications. Unlike traditional linear regression models, DNNs can have several orders of magnitude more learnable parameters than the size of the feature space, and yet they are still capable of achieving excellent generalization performance.

By utilizing the predictive normalized maximum likelihood (pNML) learner, this research proposes a novel approach to analyzing when linear regression is capable of generalization. Specifically, we derive explicate the pNML learner and its corresponding regret for the linear regression hypothesis class. This includes also the regularized case where the norm of the coefficients vector is constrained. Based on the analysis of the learnability measure, it is shown that if the test data comes from a ``learnable space'' successful generalization occurs.
This phenomenon may explain why other over-parameterized models such as DNNs are successful for ``learnable'' data.

\section{Formal Problem Definition} \label{sec:linear_regression:formal_problem_def}
Given N pairs of data and labels $\mathcal{D}_N=\{x_i, y_i\}_{i=1}^{N}$ where $x_i \in R^M, y_i \in R$, the model takes the form of
\begin{equation}
y_i =x_i^\top \theta + e_i \qquad i \leq n \leq N
\end{equation}
where $\theta \in R^M$ are the learnable parameters and the $e_i \in R$ are zero mean, Gaussian, independent with variance of $\sigma^2$. 
The goal is to predict $y$ based on a new data sample $x$. 
Under the assumptions $y$, conditioned on $x$, has a normal distribution that depends on the learnable parameters $\theta$ 
\begin{equation}
p_{\theta}(y) 
=\frac{1}{\sqrt[]{2\pi\sigma^2}}\exp\left\{-\frac{1}{2\sigma^2}\big(y- x^\top \theta \big)^2\right\}.
\end{equation}
The unknown parameter vector $\theta$ belongs to a set $\Theta$, which in the general case is the entire $R^M$. In the regularized version (leading to Ridge regression \cite{ridgeregression}), $\Theta$ is the sphere $|\theta|\leq A$. 
In ordinary least squares, the mean square error (MSE) is used.
The ERM solution in this case is
\begin{equation}
    \theta_N = \left(X_N^\top X_N \right)^{-1} X_N^\top Y_N.
\end{equation}

In order to obtain the pNML the following procedure is executed: assuming the label of the test data is known, find the best model that fits it with the training samples, and predict the assumed label by this model. 
Repeat the process for all possible labels. 
Then, normalize to get a valid probability distribution which is the pNML learner.
Recall that the pNML learner is given by:
\begin{equation} \label{eq:pNML_def}
q_{\mbox{\tiny{pNML}}}(y|x;\mathcal{D}_N) = \frac{1}{K} p_{\hat{\theta}(\mathcal{D}_N;x,y)}(y|x).
\end{equation}
where $K$ is the the normalization factor:
\begin{equation} \label{eq:linear_regression:gamma}
K = \int_R  p_{\hat{\theta}(\mathcal{D}_N;x,y)}(y|x)dy,   
\end{equation}
The goal is to find an analytic expression for (\ref{eq:pNML_def}) and for the learnability measure $\Gamma = \log K$, the minmax regret value.

\section{The pNML solution} \label{sec:linear_regression:under_parameterized:linear_pNML_eval}
In considering linear regression, $y\in {\cal R}$ is the scalar label, $x\in {\cal R}^M$ is the feature vector (sometimes the first component of $x$ is set to $1$ to formulate affine linear relation), and the model class is the following hypothesis set
\begin{equation} \label{regression_model}
P_\Theta = 
\left\{ p_{\theta}(y|x) 
=\frac{1}{\sqrt[]{2\pi\sigma^2}}\exp\left\{-\frac{1}{2\sigma^2}\big(y- x^\top\theta \big)^2\right\}, \;\; \theta \in {\cal R}^M \right\}. 
\end{equation}
That is, the label $y$ is a linear combination of the components of $x$, within a Gaussian noise. As shown below, in this case the pNML and its min-max regret can be evaluated explicitly.

The genie learner, that knows the true test label, minimizes the following MSE objective 
\begin{align}
\thetagenie= \arg\min_{\theta\in\Theta} \left[ \sum_{I=1}^N \left(y_I - x_I^\top \theta \right)^2 + \left(y-x^\top \theta\right)^2 \right].
\end{align}

Denote $X \in R^{N \times M}$ as the matrix that contains all the training data and $Y_N \in R^{N}$ as the training label vector 
\begin{equation}
X_N = \begin{bmatrix} x_1 & x_2 & \dots & x_N \end{bmatrix}^\top , \qquad
Y_N = \begin{bmatrix} y_1 & y_2 & \dots & y_N \end{bmatrix}^\top,
\end{equation}
assuming that the test label $y$ is given, the optimal solution with the Recursive Least Squares (RLS) formulation~\cite{hayes19969} is
\begin{equation} \label{eq:linear_regression:under_parameterized:linear_rls_update}
\thetagenie = \theta_{N} + P_N x (y - \hat{y})
\end{equation}
where $\hat{y} = x^\top  \theta_{N}$ is the ERM prediction based on the training set $\mathcal{D}_N$ and
\begin{equation}
P_N = \frac{\left(X_N^\top X_N \right)^{-1}}{1 + x^\top \left(X_N^\top X_N \right)^{-1} x}
\end{equation}
Thus the genie probability assignment is
\begin{equation}
\begin{split}
\probthetagenie 
&=
\frac{1}{\sqrt[]{2\pi\sigma^2}}\exp\left\{-\frac{1}{2\sigma^2}\left(y- x^\top \thetagenie \right)^2\right\} 
\\ &=
\frac{1}{\sqrt[]{2\pi\sigma^2}}\exp\left\{-\frac{1}{2\sigma^2}\left(y - x^\top  \left(\theta_{N} + P_N x (y -x^\top \theta_N) \right) \right)^2\right\} 
\\ &=
\frac{1}{\sqrt[]{2\pi\sigma^2}}\exp\left\{-\frac{1}{2\sigma^2}\left(y - x^\top \theta_{N} - x^\top P_N x (y -x^\top \theta_N) \right)^2\right\} 
\\ &= 
\frac{1}{\sqrt[]{2\pi\sigma^2}}
\exp\left\{-\frac{(1 - x^\top  P_N x )^2 }{2\sigma^2}\left(y- x^\top \theta_N \right)^2\right\}.  
\\
\end{split}
\end{equation}
To get the pNML normalization factor, we integrate over all possible labels
\begin{equation}
K = 
\int_{-\infty}^{\infty} \frac{1}{\sqrt[]{2\pi\sigma^2}}
\exp\left\{-\frac{(1 - x^\top  P_N x )^2 }{2\sigma^2}
\left(y' - x^\top \theta_N \right)^2\right\} dy'
=\frac{1}{1 - x^\top  P_N x }
\end{equation}
Thus, the pNML distribution of $y$ given the test input $x$ is
\begin{equation}
q_{\mbox{\tiny{pNML}}}(y |x) = \frac{1}{K} \probthetagenie
= 
\frac{1 - x^\top  P_N x }{\sqrt[]{2\pi\sigma^2}}
\exp\left\{-\frac{(1 - x^\top P_N x )^2 }{2\sigma^2}\left(y- x^\top \theta_N \right)^2\right\} 
\end{equation}
and its associate learnability measure or regret:
\begin{equation} \label{eq:linear_regression:under_parameterized:linear_regret}
\Gamma = \log K = \log\left(\frac{1}{1 - x^\top P_N x } \right).
\end{equation}

\subsection{pNML with regularization} \label{sec:linear_regression:under_parameterized:linear_pNMLwithReg}
Next, we shall assume that the model class $\Theta$ is constrained to the sphere  $||\theta||\leq A$, for some $A$. 
Using a Lagrange multiplier $\lambda$ we get the Tikhonov regularization (or Ridge regression), where the expression to minimize is now: 
\begin{equation}
\sum_{n=1}^{N}(y_n-x_n^\top  \theta)^2 + \lambda ||\theta||^2
\end{equation}
With the test data, the regularized least square solution is:
\begin{equation}
\probthetagenie = \theta_{N\lambda} + P_{N\lambda}x (y - x^\top \theta_{N\lambda})
\end{equation}
However, now 
\begin{equation}
P_{N\lambda}= \frac{\left(X_N^\top X_N + \lambda I \right)^{-1}}{1 + x^\top \left(X_N^\top X_N  + \lambda I \right)^{-1} x}
\end{equation}
and the ERM learner is
\begin{equation}
\theta_{N\lambda} = \left(X_N^\top X_N + \lambda I \right)^{-1} X_N^\top Y_N.
\end{equation}
The rest of the evaluation is similar to \secref{sec:linear_regression:under_parameterized:linear_pNML_eval}, yielding the following pNML learner:
\begin{equation} \label{eq:linear_regression:under_parameterized:linear_pNML_least_sqaures}
q_{\mbox{\tiny{pNML}}}(y|x)
=
\frac{1 + x^\top  P_{N\lambda} x }{\sqrt[]{2\pi\sigma^2}} 
\exp\left\{-\frac{(1 - x^\top P_{N\lambda} x )^2 }{2\sigma^2}\left(y- x^\top \theta_{N\lambda} \right)^2\right\} 
\end{equation}
and the associated regret or the log-normalization factor:
\begin{equation}
\Gamma = \log K = \log \left( \frac{1}{1 - x^\top  P_{N\lambda} x } \right)
\end{equation}
Note that regularization can help in the case where $X_N^\top X_N $, the unnormalized correlation matrix of the data is
ill conditioned. 
In the next section we find the learnable space for the linear regression problem and observe situations where this regularization is needed.

\section{The learnable space} \label{sec:linear_regression:under_parameterized:linear_regression_learnable_space}
In order to understand for which test sample the trained model generalizes well we need to look at the regret expression (\eqref{eq:linear_regression:under_parameterized:linear_regret}). 
High regret means that the pNML learner is far from the genie and therefore we may not trust its predictions. 
Low regret, on the other hand, means the model is as good as a genie who knows the true test label, and so it is trusted.

By simplying the derived min-max regret, we get:
\begin{equation}
\Gamma  = \log \left(\frac{1}{ 1 - x^\top P_N x} \right) = \log \left( \frac{1}{ 1 - \frac{x^\top \left(X_N^\top X_N\right)^{-1} x}{1 + x^\top \left(X_N^\top X_N\right)^{-1} x}} \right) = \log \left(1 + x^\top \left(X_N^\top X_N\right)^{-1} x \right)
\end{equation}

Consider the matrix $X_N$, composed of the training data, and apply the singular value decomposition (SVD) on it, i.e., $X_N^\top = U \Sigma V^\top $ with $U\in R^{M \times M}$, 
$\Sigma$ is a rectangular diagonal matrix of the singular values $\eta_i$ and $V \in R^{N \times N}$.
Denote $R_N$ as the empirical correlation matrix of the training: 
\begin{equation}
R_N=\frac{1}{N} U \Sigma \Sigma^\top  U^\top ,
\end{equation}
the pNML regret can be written as 
\begin{equation}
\Gamma = \log K = \log \left(1 + \frac{1}{N} \sum_{i=1}^{M} \frac{\left(x^\top u_i\right)^2 }{\eta^2_i}\right).
\end{equation}
If the test sample $x$ lies mostly in the subspace spanned by the eigenvectors with large eigenvalues, then the model can generalize well.

\section{Simulation} \label{sec:linear_regression:under_parameterized:lin_simulation}
In this section, we present some simulations that demonstrate the results above. 
We chose the problem of fitting a polynomial to data, which is a special case of linear regression.
The simulations show prediction and generalization capabilities in a variety of regularization factors and polynomial degrees.

In the first experiment we generated 3 random points, $t_0,t_1,t_2$, uniformly in the interval $[-1, 1]$. These points are the training set and are shown in \figref{fig:linear_regression:under_parameterized:linear_least_squares_with_reg} (top) as red dots. 
The relation between $y$ and $t$ is given by a polynomial of degree two.
Thus, the X matrix of \secref{sec:linear_regression:under_parameterized:linear_pNML_eval} is given by:
\begin{equation} \label{eq:linear_regression:under_parameterized:linear_two_degree_pol}
X = 
\begin{bmatrix}
1 & 1 & 1 \\
t_0 & t_1 & t_2 \\
t_0^2 & t_1^2 & t_2^2 
\end{bmatrix}.
\end{equation}
Based on the training we predict a probability for all t values in the interval [-1,1] using (\ref{eq:linear_regression:under_parameterized:linear_pNML_least_sqaures}) with a regularization factor $\lambda$ of $0$, $0.1$ and $1.0$. 
It is shown in \figref{fig:linear_regression:under_parameterized:linear_least_squares_with_reg} (top) that without regularization ($\lambda=0$), the blue curve fits the data exactly. As $\lambda$ increases the fitted curve becomes less steep but tends to fit less to the training data.

\figref{fig:linear_regression:under_parameterized:linear_least_squares_with_reg} (bottom) shows the regret, given by (\ref{eq:linear_regression:under_parameterized:linear_regret}), for the polynomial model from \eqref{eq:linear_regression:under_parameterized:linear_two_degree_pol} for all $t\in[-1,1]$ where the training $t_i$'s are marked in red on the x axis. 
We can see that around the training data the regret is very low in comparison to areas where there are no training data. 
In addition, models with larger regularization term have lower regret for every point in the interval $[-1,1]$.
For all regularization terms, the regret increases as moving away from the training data.

\begin{figure}[tbh] 
    \centering
    \begin{subfigure}[t]{0.49\linewidth}
    \includegraphics[width=\linewidth]{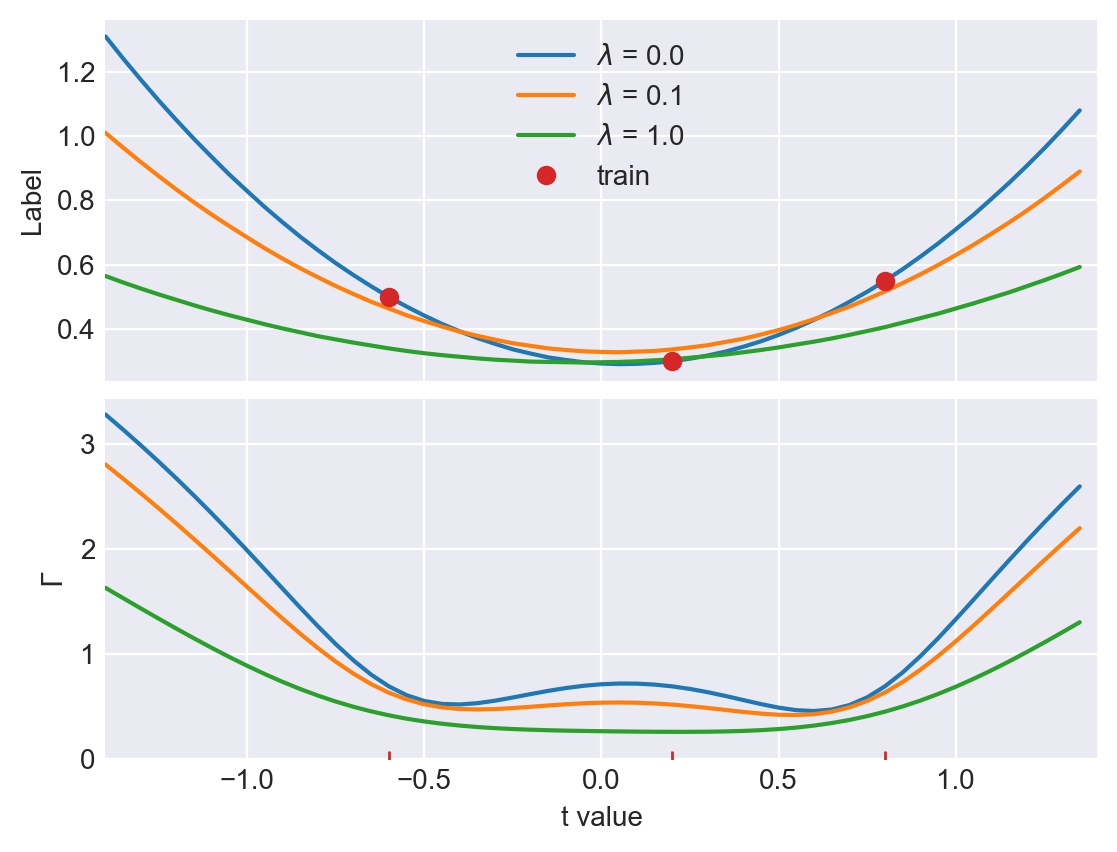} 
    \caption{Varied regularization terms
    \label{fig:linear_regression:under_parameterized:linear_least_squares_with_reg}
    }
    \end{subfigure}
    \begin{subfigure}[t]{0.49\linewidth}
    \includegraphics[width=\linewidth]{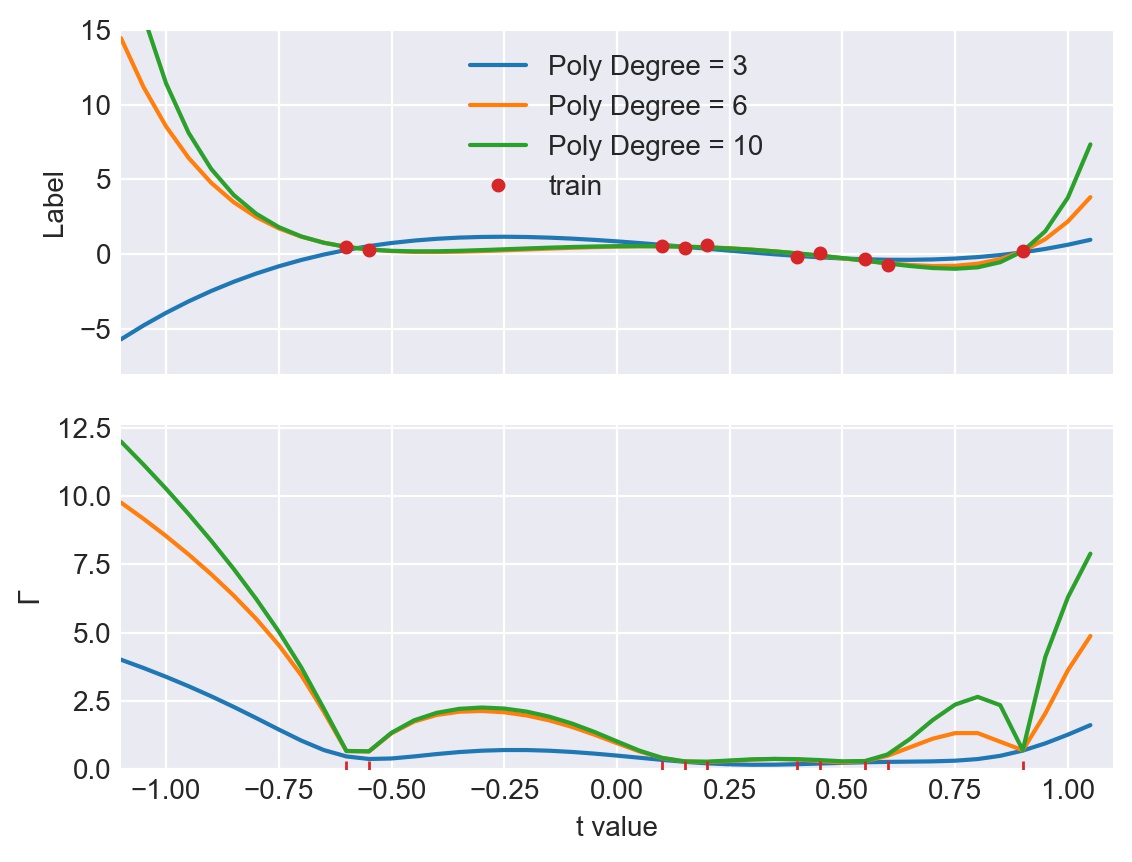}
    \caption{Different polynomial degrees
    \label{fig:linear_regression:under_parameterized:linear_least_squares_with_poly}
    }
    \end{subfigure}
    \caption{The mean pNML least squares estimator and its associated regret}
\end{figure}

Next, we simulate the case of fitting polynomials with different degrees. 
Again, we generated 10 random points in the interval $[-1, 1]$.
The matrix $X$ is now:
\begin{equation}
X = 
\begin{bmatrix}
1 & 1 & 1 & \hdots & 1\\
t_0 & t_1 & t_2 & \hdots & t_{9}\\
\vdots &  \vdots &       & \vdots  \\
t_0^{\textit{Poly Deg}} & t_1^{\textit{Poly Deg}} & t_2^{\textit{Poly Deg}}  & \hdots & t_{9}^{\textit{Poly Deg}}
\end{bmatrix}.
\end{equation}

\Figref{fig:linear_regression:under_parameterized:linear_least_squares_with_poly} (top) shows the predicted label for every $t$ value in $[-1,1]$ for the different polynomial degrees. 
To avoid singularities we used the regularized version with $\lambda=10^{-4}$. 
The training set is shown by red dots in the figure. 
Note that for a polynomial of degree ten, the number of parameters is greater than the training set size. 
Nevertheless, the prediction accuracy near the training samples is similar to that of a degree three polynomial.

\Figref{fig:linear_regression:under_parameterized:linear_least_squares_with_poly} (bottom) shows the regret (or learnability) of the three pNML learners corresponding to model classes of polynomials with the various degrees. 
All the learners have regret values that are small near the training samples and large as $t$ drifts away from these samples.

\section{Concluding remarks}
In this chapter, we provided an explicit analytical solution of the pNML universal learning scheme and its learnability measure for the linear regression hypothesis class. 
Interestingly, the predicted universal pNML assignment is Gaussian with a mean that is equal to that of the ERM, but with a variance that increases by a factor $K$ whose logarithm is the learnability measure $\Gamma$.
Analyzing $\Gamma$ we can observe the ``learnability space'' for this problem. 
Specifically, if a test sample mostly lies in the subspace spanned by the eigenvectors associated with large eigenvalues of the empirical correlation matrix then the learner can generalize well, even in an over-parameterized case where the regression dimension is larger than the number of training samples.
Finally, we provided a simulation of the pNML least squares prediction for polynomial interpolation.

The the next chapters, we'll derive the pNML learner for other hypothesis classes as neural network and over-parameterized linear regressions.
We conjecture that as in linear regression other ``over-parameterized'' model classes are learnable at least locally, that can be inferred from the pNML solution. 

\chapter{Overparameterized Linear Regression}  
\label{chap:overparameterized_linear_regression}
\section{Introduction}\label{sec:linear_regression:overparameterization:introduction}

Classic learning theory argues that complex models tend to overfit their training set, thus generalizing poorly to unseen ones~\citep{bartlett2020benign,hastie2001friedman}.
This assumption is challenged by modern learning models such as DNN which operate well even with a perfect fit to the training set~\citep{DBLP:conf/iclr/ZhangBHRV17}.
Motivated by this phenomenon, we consider when a perfect fit in training is compatible with an accurate prediction, i.e., when a small \textit{generalization error} is achieved.

We examine \textit{over-parameterized} linear regression, where the number of the learnable parameters is larger than the training set size.
We focus on the \textit{minimum norm} (MN) solution. This solution has the following unique property: It is the solution with the minimal norm that attains a perfect fit on the training set.
Recent work show that the MN solution generalizes well in the over-parameterized regime and approximated its generalization error~\citep{belkin2018overfitting,hastie2019surprises,liang2020just,ma2019generalization,shah2018minimum}. However, they assume some probabilistic connection between the training and test which may not be valid in a real-life scenario.

A popular approach to deal with over-parameterized models is to find the optimal regularization term for the ridge regression model class.
\citet{nakkiran2020optimal} showed that models with optimally-tuned regularization achieve monotonic test performance as growing the model size.
\citet{dwivedi2020revisiting} used the minimum description length principle to quantify the model complexity and to find the optimal regularization term.

Several studies suggested that for linear regression the generalization is proportional to the model norm~\citep{kakade2009complexity,ma2019generalization,shamir2015sample}.
\citet{muthukumar2020harmless} showed that the generalization error decays to zero with the number of features.
\citet{tsigler2020benign} provided non-asymptotic generalization bounds for over-parameterized ridge regression.
\citet{nichani2020deeper} analyzed the effect of increasing the depth of linear networks on the test error using the MN solution.
\citet{hastie2019surprises} provided a non-asymptotic approximation of the generalization error for the over-parameterized region.
In addition, several authors argued that the MN solution captures the basic behavior of DNN~\citep{allen2019convergence,gunasekar2018implicit}.

However, all mentioned work assume some probability distribution on the training and testing sets or the learnable parameters. This assumption may not apply in a real-life scenario. Moreover, they do not consider the specific test input thus do not provide a point-wise generalization error.

In this chapter, we derive an upper bound of the pNML regret for over-parameterized linear regression.
We design the hypothesis set to contain hypotheses that have a norm that is not larger than the MN norm.
By utilizing this hypothesis set, the pNML prediction equals the MN solution. Thus the derived upper bound of the regret can be used as the generalization error of the MN solution.
We show that if the test vector resides in a subspace spanned by the eigenvectors associated with the large eigenvalues of the empirical correlation matrix of the training data, linear regression can generalize despite its over-parameterized nature.
In addition, we present a recursive formulation of the norm of the MN solution. 
We demonstrate the case where a small deviation from the prediction of the MN solution increases the model norm significantly, which implies high confidence in the MN prediction.

To summarize, we cover the following derivations.
\begin{itemize}
    \item \textbf{Designing the norm constrained hypothesis set.} Introducing the norm constrained hypothesis set for over-parameterized linear regression. Utilizing this set, we create a pNML learner that has a meaningful regret and a prediction that equals the MN solution.
    \item \textbf{Upper bounding the pNML regret.}  Deriving an analytical upper bound of the pNML regret, which is associated with the generalization error. We demonstrate what are the characteristics of the test data for which the regret is small.
    \item \textbf{Deriving a recursive formulation for the norm of the MN solution.} We present a recursive formula for the norm of the MN solution. We show what are the properties of the test data for which a small deviation from the MN prediction increases significantly the norm of the MN solution. In this situation, the MN prediction is considered reliable, as it has a significantly smaller norm than the other predictors that fit the training data.
\end{itemize}

The presented results hold for nearly all settings since the pNML is the min-max solution of the individual setting in which there is no assumption on a probabilistic connection between the training set and the test sample (distribution-free).
In addition, we demonstrate the use of the pNML regret as a confidence measure in a simulation of fitting trigonometric polynomials to synthetic data. Also, we show that the empirically calculated regret and its upper bound are correlated with the test error double-descent phenomenon on sets from the UCI repository~\citep{Dua:2019}.

\section{Notation and preliminaries}
\label{sec:norm_constrained_pnml:preliminaries}

In the supervised machine learning scenario, a training set consisting of $N$ pairs of examples is given
\begin{equation}
\mathcal{D}_N = \{(x_n,y_n)\}_{n=1}^{N}, \quad x_n \in R^{M \times 1}, \quad y_n \in R,
\end{equation}
where $x_n$ is the $n$-th data instance and $y_n$ is its corresponding label.
The goal of a learner is to predict the unknown label $y$ given a new test data $x$ by assigning a probability distribution $q(\cdot|x)$ to the unknown label. 
The performance is evaluated using the log-loss function
\begin{equation}
\ell(q;x,y) = -\log {q(y|x)}.
\end{equation}

For over-parameterization, the MN solution is the solution that attains a perfect fit to the training set and has the lowest norm
\begin{equation} \label{eq:norm_constrained_pnml:mn_solution}
\theta_N^* = X_N^+ Y_N.
\end{equation}
where $X_N^+$ is the Moore-Penrose inverse of $X_N$ is~\citep{ben2003generalized}
\begin{equation} \label{eq:norm_constrained_pnml:pseudo-inverse}
X_N^+ = 
\begin{cases} 
    (X_N^\top X_N)^{-1} X_N^\top & \textit{Rank}(X_N^\top X_N) = M, \\
    X_N^\top (X_N X_N^\top)^{-1} & \textit{otherwise}.   
\end{cases}
\end{equation}

Denote the ridge regression solution as 
\begin{equation} \label{eq:norm_constrained_pnml:ridge}
\theta_N^\lambda \triangleq \left(X_N^\top X_N  + \lambda I \right)^{-1} X_N^\top Y_N 
\end{equation} 
where $\lambda$ is the regularization term, it can be shown that $\lim_{\lambda \xrightarrow{} 0} \theta_N^\lambda = \theta_N^*$~\citep{zhou2002variants}.

\section{The norm of the minimum norm solution} 
\label{sec:norm_constrained_pnml:mn_norm}

We present the behavior of the norm of the MN solution for an over-parameterized linear regression model. 
We show the properties of the test sample for which the MN solution prediction can be trusted.
In~\secref{sec:norm_constrained_pnml:regret_upper_bound}, we use this result to upper bound the regret.

\begin{theorem}  \label{theorem:mn_norm}
Denote the projection of the test sample onto the orthogonal subspace of the training data empirical correlation matrix as 
\begin{equation}
x_\bot \triangleq \left[I - X_N^+ X_N \right] x,
\end{equation}
the norm of the MN solution based on the training set $\mathcal{D}_N$ and the test sample $(x,y)$ is given by 
\begin{equation}  \label{eq:norm_constrained_pnml:mn_solution_norm}
\norm{\theta_{N+1}^*}^2 = \norm{\theta_N^*}^2 + \frac{1}{\norm{x_\bot}^2}(y-x^\top \theta_N^*)^2.
\end{equation}
\end{theorem}
\begin{proof}
Let $c=x^\top \left(I - X_N^+ X_N \right)$.
For $c \neq 0$ the recursive formula to compute the pseudo-inverse of the data matrix is
\begin{equation}
X_{N+1}^+ = 
\begin{bmatrix} 
X_N^+ - c^+ x^\top X_N^+ & c^+ 
\end{bmatrix}.
\end{equation}
Denote the MN solution based on the $N$ training samples by $\theta_N^*$,
given a new sample $(x,y)$ the recursive formulation of the MN solution based on the training set and this new sample is
\begin{equation}
\theta_{N+1}^* 
= X_{N+1}^+ Y_{N+1} = 
\begin{bmatrix} 
X_N^+ - c^+ x^\top X_N^+ & c^+ 
\end{bmatrix}
\begin{bmatrix}
Y_N \\ y \\
\end{bmatrix} 
= 
\theta_N^* +  c^+ (y - x^\top \theta_N^*). 
\end{equation}
The norm of the MN solution based on these $N+1$ samples is
\begin{equation} \label{eq:norm_constrained_pnml:mn_norm_recurisve}
||\theta_{N+1}^*||^2  
= 
||\theta_N^*||^2 + 2 \theta_N^{* \top} c^+ (y - x^\top \theta_N^*) + c^{+ ^\top} c^+ (y - x^\top \theta_N^*)^2.
\end{equation}

Denote $x_\bot = \left(I - X_N^+ X_N \right) x$, the pseudo-inverse of $c$ is
\begin{equation} \label{eq:norm_constrained_pnml:c^+}
c^+ = c^\top (c c^\top)^{-1} 
= 
\left[x^\top \left(I - X_N^+ X_N \right)\right]^\top \frac{1}{x^\top \left[I -X_N^\top (X_N X_N^\top )^{-1} X_N \right] x} 
=
\frac{x_\bot}{||x_{\bot}||^2}.
\end{equation}
The inner product of the MN solution and $c^+$ can be written as
\begin{equation} \label{eq:norm_constrained_pnml:theta_N_dot_c^+}
\theta_N^{* \top} c^+ 
=  
Y_N^\top X_N^{+ \top}
\frac{\left(I - X_N^+ X_N \right)^\top x}{x^\top \left(I - X_N^+ X_N \right)^2 x} 
= 
\frac{ Y_N^\top \left(X_N^+ - X_N^+\right)^\top x}{x^\top \left(I - X_N^+ X_N \right)^2 x} 
= 0.
\end{equation}

Substitute \eqref{eq:norm_constrained_pnml:c^+} and \eqref{eq:norm_constrained_pnml:theta_N_dot_c^+} to \eqref{eq:norm_constrained_pnml:mn_norm_recurisve} gives the final result
\begin{equation} \label{eq:norm_constrained_pnml:mn_norm_behaviour}
\norm{\theta_{N+1}}^2 = \norm{\theta_N^*}^2 + \frac{1}{\norm{x_{\bot}}^2} \left(y - x^\top \theta_N^* \right)^2.
\end{equation}
\end{proof}

If the test sample $x$ lies mostly in the subspace that is spanned by the eigenvectors of the empirical correlation matrix of the training data then $\norm{x_\bot}$ is small: A slight deviation from the MN solution prediction increases significantly the norm of $\theta_{N+1}^*$ (the MN solution that includes the test sample).
On the other hand, if the test sample lies in the orthogonal subspace, $\norm{x_\bot}$ is relatively large and a deviation from the MN solution prediction does not change the norm of $\theta_{N+1}^*$.

If many values of the test label produce MN solution with a low norm, they are all reasonable and therefore none of them can be trusted. On the contrary, if there is just one value of the test label that is associated with MN solution with a low norm, we are confident that it is the right one. 
For confident prediction, we would like that any other prediction will cause a model with high complexity, this is a situation where $\norm{x_\bot}$ is small.
We use this result to upper bound the pNML regret in~\secref{sec:norm_constrained_pnml:regret_upper_bound}.

\section{The pNML for over-parameterized linear regression}
\label{sec:norm_constrained_pnml:pnml_learner}



For linear regression, we assume a linear relationship between the data and labels
\begin{equation}
y_i = x_i^\top \theta + e_i,
\qquad 1 \leq i \leq N, 
\quad x_n \in R^{M \times 1}, 
\quad y_n \in R.
\end{equation}
$e_i$ is a white noise random variable with a variance of $\sigma^2$.
The test label $y$, conditioned on the test data $x$, has a normal distribution that depends on the learnable parameters
\begin{equation} \label{eq:norm_constrained_pnml:p_theta}
p_{\theta}(y|x) 
=\frac{1}{\sqrt[]{2\pi\sigma^2}}\exp\left\{-\frac{1}{2\sigma^2}\big(y- x^\top \theta \big)^2\right\}.
\end{equation}
The unknown vector $\theta$ belongs to a set $\Theta$.
A different perspective that corresponds to the individual setting is to assume that $x$ and $y$ are individual values and the given hypothesis set that the genie can choose from is composed of learners that are defined by \eqref{eq:norm_constrained_pnml:p_theta}.

Recall the pNML solution for linear regression from the previous section:
Denote the data matrix and label vector as
\begin{equation} \label{eq:norm_constrained_pnml:trainset_matrix}
X_N = 
\begin{bmatrix}
x_1 & x_2 & \dots & x_N
\end{bmatrix}^\top \in \mathcal{R}^{N \times M},
\quad
Y_N = \begin{bmatrix} y_1 & y_2 & \dots & y_N \end{bmatrix}^\top \in \mathcal{R}^{N \times 1},
\end{equation}
and let $u_m$ and $h_m$ be the $m$-th eigenvector and eigenvalues of the training set data matrix.
Assuming $X_N^\top X_N$ is invariable ($M \leq N$), the pNML regret and normalization factor are
\begin{equation} \label{eq:norm_constrained_pnml:ols_regret}
\Gamma = \log K_0, \qquad
K_0 = 1 + \frac{1}{N} \sum_{m=1}^{M} \frac{\left(x^\top u_m\right)^2 }{h_m^2}.
\end{equation}

This result deals with under-parameterized linear regression models. It shows that if the test sample $x$ lies in the subspace spanned by the eigenvectors with large eigenvalues, the corresponding regret is low. In this situation the model prediction is similar to the genie's and can be trusted.

Executing the pNML procedure using an over-parameterized hypothesis set would lead to noninformative regret:
Having a large hypothesis set may produce a perfect fit to every test label and therefore the maximal regret.
To reduce the hypothesis set size, we include learners whose $\normltwo$ norm is not larger than the norm of the MN solution
\begin{equation} \label{eq:norm_constrained_pnml:mn_hypotesis_class}
P_\Theta = \left\{p_\theta(y|x) \ | \ \norm{\theta} \leq \norm{\theta_N^*} \ , \ \theta \in R^{M \times 1} \ \right\}.
\end{equation}
Our goal is to find the pNML regret using this hypothesis set as defined in \eqref{eq:norm_constrained_pnml:mn_hypotesis_class}.

\subsection{The pNML regret upper bound} \label{sec:norm_constrained_pnml:regret_upper_bound}

We now show an upper bound of the pNML regret using the hypothesis set that contains only learners that have a norm that is not larger than the MN norm.

The genie that knows the true test label value is the solution of the following minimization objective
\begin{equation}
\thetagenie = \argmin_{\theta \in R^{M \times 1}} \left[ \left(y - x^\top \theta \right)^2 + \sum_{n=1}^N \left(y_n - x_n^\top \theta \right)^2\right] \text{ s.t.} \norm{\theta} = \norm{\theta_N^*}.
\end{equation}
With \eqref{eq:norm_constrained_pnml:ridge}, we write the genie using the recursive least squares formulation~\citep{hayes19969}
\begin{equation} \label{eq:norm_constrained_pnml:theta_genie_rls}
\thetagenie = \theta_N^\lambda + \frac{\left(X_N^\top X_N + \lambda I\right)^{-1} x}{1 + x^\top \left(X_N^\top X_N + \lambda I \right)^{-1} x} \left(y - x^\top \theta_N^\lambda\right).
\end{equation}
Notice that $\lambda$ depends on the test label $y$ such that the norm constraint is fulfilled.

\begin{lemma} \label{lemma:genie_upper_bound}
The upper bound of the genie probability assignment is
\begin{equation}
\probthetagenie
\leq 
\frac{1}{\sqrt{2\pi\sigma^2}}
\exp\left\{
- \frac{\left(y-x^\top \theta_N^\lambda \right)^2}{2\sigma^2 K_0^2 \left(1 + \frac{\norm{x_\bot}^2}{K_0 \lambda} \right)^2} 
\right\},
\end{equation}
\end{lemma}
where $K_0$ is defined in\eqref{eq:norm_constrained_pnml:ols_regret}
and $\lambda$ satisfies the norm constraint
\begin{equation}
||\thetagenie|| = \norm{\theta_N^*}.
\end{equation}
\begin{proof}
The genie probability assignment using the recursive formulation~\eqref{eq:norm_constrained_pnml:theta_genie_rls} is
\begin{equation} \label{eq:norm_constrained_pnml:prob_genie_rls}
\begin{split}
\probthetagenie &=
\frac{1}{\sqrt{2\pi\sigma^2}}
\exp\left\{
-\frac{1}{2\sigma^2} \left(y - x^\top \thetagenie \right)^2
\right\}
\\ &=
\frac{1}{\sqrt{2\pi\sigma^2}}
\exp\left\{
-\frac{\left(y - x^\top \theta_N^\lambda \right)^2}{2\sigma^2\left[1 + x^\top \left(X_N^\top X_N + \lambda I \right)^{-1} x\right]^2} 
\right\}.
\end{split}
\end{equation}
Let $u_m$ and $h_m$ be the $m$-th eigenvector and eigenvalues of the training set data matrix (using SVD decomposition). 
Assuming over-parameterization $N < M$,
\begin{equation} \label{eq:norm_constrained_pnml:1_plus_x_P_x_upper_bound}
1 + x^\top \left(X_N^\top X_N + \lambda I \right)^{-1} x = 
1 +
\sum_{m=1}^N \frac{\left( u_m^\top x \right)^2}{h_m^2 + \lambda} 
+
\sum_{m=N+1}^M \frac{\left( u_m^\top x \right)^2}{\lambda}
\leq
K_0 + \frac{1}{\lambda} \norm{x_\bot}^2.
\end{equation}
where we set $\lambda=0$ for $m\leq N$.
Substitute \eqref{eq:norm_constrained_pnml:1_plus_x_P_x_upper_bound} to \eqref{eq:norm_constrained_pnml:prob_genie_rls} proves the lemma.
\end{proof}
The genie probability distribution is monotonic decreasing with respect to $\lambda$.

\begin{lemma} \label{lemma:lambda_lower_bound}
The lower bound of the regularization term $\lambda$ that satisfies $||\thetagenie||=\norm{\theta_N^*}$ is
\begin{equation}
\lambda 
\geq
\frac{1}{2}
\frac{\frac{1}{\norm{x_\bot}^2}\left(y - x^\top \theta_N^* \right)^2 }{\theta_N^{*\top} X_N^+ X_N^{+ \top} \theta_N^*
+
\frac{\left(y - x^\top \theta_N^* \right)^2}{||x_{\bot}||^2} 
x^\top X_N^+ X_N^{+ \top} x}.
\end{equation}
\end{lemma}
\begin{proof}
The proof is given in \appref{sec:norm_constrained_pnml:lambda_lower_bound}.
\end{proof}
When $y$ equals the MN solution prediction $x^\top \theta_N^*$, the regularization term is zero. As $y$ deviates from the MN solution prediction, the value of $\lambda$ required to satisfy the norm constraint increases. This suggests that the pNML probability assignment is maximized at the MN prediction and exhibits a form of symmetry around it, in a qualitative sense..

For each possible value of the test label $y'$, we wish to find the learner that satisfies the norm constraint $||\thetagenietag|| = ||\theta_N^*||$
and use in the pNML regret calculation
\begin{equation} \label{eq:norm_constrained_pnml:overparam_norm_factor}
\Gamma = \log 
\int_{-\infty}^\infty
\frac{1}{\sqrt{2 \pi \sigma^2}} \exp \left\{-
\frac{1}{2\sigma^2} \left(y- x^\top \thetagenietag \right)^2 \right\} dy'.
\end{equation}
\begin{theorem} \label{theorem:regret_upper_bound}
The norm constrained pNML regret upper bound is
\begin{equation} \label{eq:norm_constrained_pnml:regret_upper_bound}
\Gamma 
\leq
\log \left[
\left(1 + x^\top X_N^+ X_N^{+ ^\top} x \right)
\left(1 + 2\norm{x_\bot}^2 \right)
+ 
3\sqrt[3]{
\frac{1}{\pi \sigma^2}
\norm{x_\bot}^2
\theta_N^{*\top} X_N^+ X_N^{+ \top} \theta_N^*
}
\right]
\end{equation}
\end{theorem}
\begin{proof}
Let $\delta \geq 0$, we introduce a relaxation of the constraint on the norm of $||\thetagenietag||^2$ such that
\begin{equation}
||\thetagenietag||^2 \leq (1+\delta) \norm{\theta_N^*}^2.
\end{equation}
Using \Theoref{theorem:mn_norm}, a perfect fit is attained when the following constraint is satisfied 
\begin{equation}
|y' - x^\top \theta_N^*| \leq |\Tilde{y}|, \qquad 
\Tilde{y} \triangleq  x^\top \theta_N^* + \sqrt{\delta \norm{x_\bot}^2 \norm{\theta_N^*}^2}.
\end{equation}
We upper bound the regret with the relaxed constraint: For all $y'$s up to $\Tilde{y}$ we get a perfect fit, and for $y'$ from $\Tilde{y}$ to infinity we use the upper bound from \lemmaref{lemma:genie_upper_bound}
\begin{equation}
\Gamma \leq \log \left[  
2 \int_{x^\top \theta_N^*}^{\Tilde{y}} \frac{1}{\sqrt{2 \pi \sigma^2}} dy' \ +
2 \int_{\Tilde{y}}^{\infty} 
\frac{1}{\sqrt{2\pi\sigma^2}}
\exp\left\{
- \frac{\left(y'-x^\top \theta_N^\lambda \right)^2}{2\sigma^2 K_0^2 \left(1 + \frac{\norm{x_\bot}^2}{K_0 \lambda} \right)^2}
\right\}
dy' \right].
\end{equation}
Next, we fix $\lambda$ at the point $\Tilde{y}$. We use the lower bound from \lemmaref{lemma:lambda_lower_bound} to further upper bound the expression.
After integrating, we get the regret that depends on $\delta$.
To get a tight bound, we find $\delta$ that minimizes the regret and that proves the theorem.
The complete derivation is given in \appref{sec:norm_constrained_pnml:pnml_regret_upper_bound}.
\end{proof}

In supervised machine learning, the training set is given, and we are interested in identifying the conditions under which the test sample is associated with a low generalization error.
We make the following remarks:
\begin{enumerate}
\item 
Looking at the $x^\top X_N^+ X_N^{+ \top} x $ term and let $u_m$ and $h_m$ represent the $m$-th eigenvector and eigenvalue of the training data matrix $X_N$
\begin{equation}
K_0 =  1 + x^\top X_N^+ X_N^{+ ^\top} x = 1 + \frac{1}{N} \sum_{m=1}^{\min \left(M,N\right)} \frac{\left(x^\top u_m\right)^2 }{h_m^2}.
\end{equation}
This term is small when the test sample lies within the subspace spanned by the eigenvectors of the training set empirical correlation matrix that is associated with the large eigenvalues.
Also, $K_0$ decreases when increasing the training set size.
\item 
According to~\eqref{eq:norm_constrained_pnml:regret_upper_bound}, when $\norm{x_\bot}=0$, the regret upper bound equals the under-parameterized pNML regret $\log K_0$ as in~\eqref{eq:norm_constrained_pnml:ols_regret}
\item 
As more energy of the test sample is found in the orthogonal subspace of the training data correlation matrix, $\norm{x_\bot}$ increases and so does the regret.
\item
The pNML regret is proportional to the term 
\begin{equation}
\theta_N^{*\top} X_N^+ X_N^{+ \top} \theta_N^* = \norm{X_N^{+ \top} X_N^{+} Y_N}^2.  
\end{equation}
This term represents the norm of the MN solution. 
This is similar to the works described in introduction that show that linear regression generalization is proportional to the model norm.
\item
Increasing $\sigma^2$ reduces the pNML regret. This may relate to the genie: Increasing the noise reduces the genie's performance, which makes the pNML log-loss closer to the genie's.
\end{enumerate}

\begin{figure}[tb]
\centering
\begin{subfigure}[t]{0.49\linewidth}
    \includegraphics[width=\columnwidth]{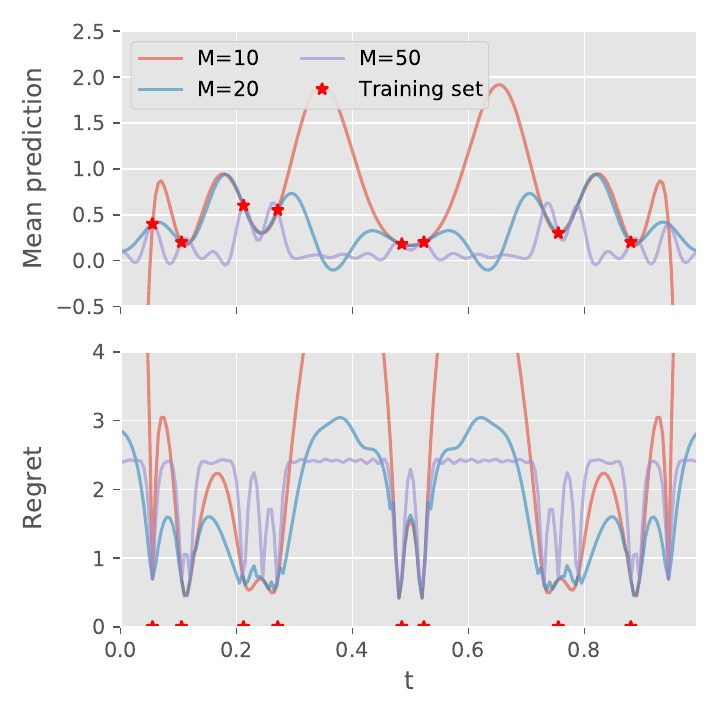}
    \caption{The pNML of different model degrees \label{fig:norm_constrained_pnml:pnml_pred_and_regret}}
\end{subfigure}
\begin{subfigure}[t]{0.49\linewidth}
    \includegraphics[width=\columnwidth]{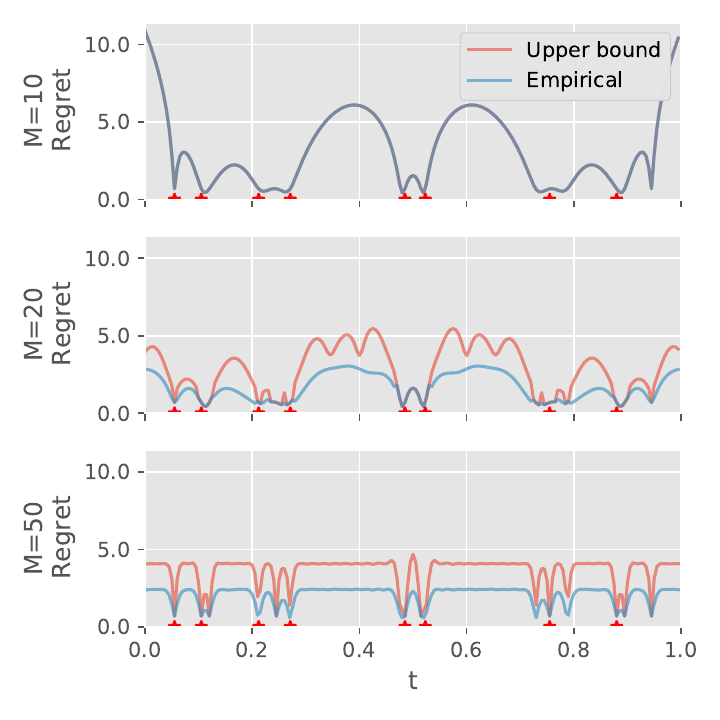}
    
    \caption{The pNML min-max regret \label{fig:norm_constrained_pnml:pnml_syntetic_analytical_vs_empirical}}
\end{subfigure}
\caption{The pNML regret for over-parameterized linear regression
}
\end{figure}

\subsection{Synthetic data experiment} 
\label{sec:norm_constrained_pnml:synthetic_data}

We use a training set that consists of 8 points $\{t_n,y_n\}_{n=0}^7$ in the interval $[0,1]$. These points are shown in \figref{fig:norm_constrained_pnml:pnml_pred_and_regret} (top) as red dots. 
The data matrix was created with the following conversion:
\begin{equation}
X_{N}[n,m] = \cos \left(\pi m t_n +  \frac{\pi}{2} m\right) , \qquad
0 \leq n < N, \quad
0 \leq m <  M, 
\end{equation}
where $M$ and $N$ are the number of learnable parameters and the training set size respectively.
We predict using the pNML learner the labels of all $t$ values in the interval $[0,1]$.
\Figref{fig:norm_constrained_pnml:pnml_pred_and_regret} (top) shows the mean pNML prediction for $M$ values of 10, 20, and 50. 
Since the number of parameters is greater than the training set size, all curves fit perfectly to the training points.

We treat each point in the interval $[0,1]$ as a test point and calculate its pNML min-max regret as shown in \figref{fig:norm_constrained_pnml:pnml_pred_and_regret} (bottom).
The training $t_n$'s are marked in red on the horizontal axis. 
For every $M$, in the training data surroundings the regret is low comparing to areas where there are no training data. 

Surprisingly, the model with $M=10$ has a larger regret than models with a greater number of parameters. 
It may relate to the constraint: In this model, the norm value of the MN solution is 295,552  while for the models with $M=20$ and $M=50$ the MN solution norms are 0.15 and 0.04 respectively.
Having a lower norm constraint means a simpler model and better generalization.
This behaviour is also presented in the regret upper bound \eqref{eq:norm_constrained_pnml:regret_upper_bound} with the term $\theta_N^{*\top} X_N^+ X_N^{+ \top} \theta_N^*$ that is proportional to the norm constraint value.
Furthermore, looking at $t=0.35$ for instance, the model with $M=10$ predicts a label that deviates from 0 much more than the model with $M=50$.

To show that the derived upper bound is informative we plot it along with the empirically calculated pNML min-max regret in~\figref{fig:norm_constrained_pnml:pnml_syntetic_analytical_vs_empirical}.
For $M=10$ the analytical expression and the empirically calculated regret give the same results.
For the other model degrees, the upper bound and the empirical regret have a similar characteristic: In areas where the training data exists, the regret decreases, and as moving away to areas without training points, the regret increases.

\section{Real data: UCI dataset} \label{app:norm_constrained_pnml:real_data_uci}
To evaluate the regret as a generalization measure we use datasets from the UCI repository as proposed by~\citet{hernandez2015probabilistic} with the same test and train splits.

For each dataset, we varied the training set size and fit the pNML and MN learners. We optimize $\sigma^2$ on a validation set for both learners.
We define a regret threshold and check the performance of the pNML taking into account only samples whose regret is lower than this threshold. We also evaluate the logloss of the MN learner of these samples. 

\Figref{fig:norm_constrained_pnml:regret_based_learner} shows the logloss and the Cumulative Distribution Function (CDF) as function of the regret threshold with the 95\% confidence interval that was calculated on different train-test splits.
Both the test logloss of the pNML and MN learners are monotonically increasing functions of the regret threshold.
For 6 out of 10 datasets the pNML test logloss is lower than the MN while for the others the performance is equal.
Using the low regret as an indication for good generalization works the best in the Naval Propulsion dataset: the average test logloss of the 80\% of the samples with the lowest regret is 1.12, while the average logloss over all samples is 1.8.

\begin{figure}[bt]
\centering
\begin{subfigure}[t]{0.315\linewidth}
    \includegraphics[width=1.0\textwidth]{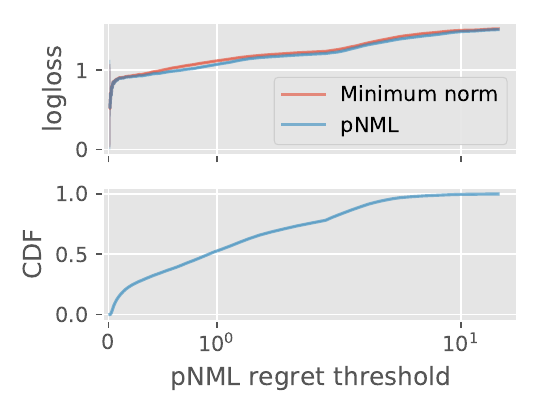}
    \caption{Boston Housing  \label{fig:norm_constrained_pnml:bostonHousing_regret}}
\end{subfigure}
\begin{subfigure}[t]{0.315\linewidth}
    \includegraphics[width=1.0\textwidth]{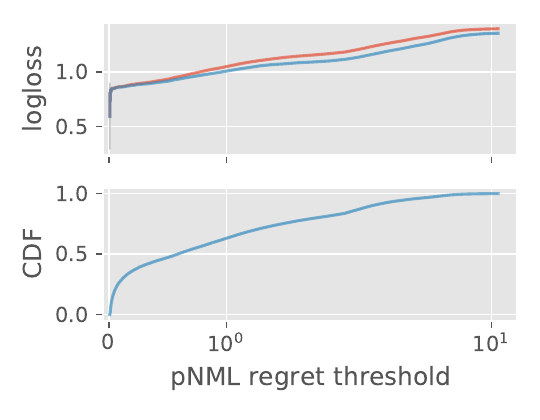}
    \caption{Concrete Strength \label{fig:norm_constrained_pnml:concrete_regret}}
\end{subfigure}
\begin{subfigure}[t]{0.315\linewidth}
    \includegraphics[width=1.0\textwidth]{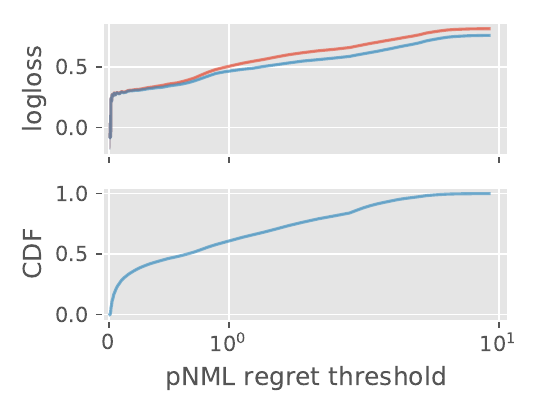}
    \caption{Energy Efficiency \label{fig:norm_constrained_pnml:energy_regret}}
\end{subfigure}
\begin{subfigure}[t]{0.315\linewidth}
    \includegraphics[width=1.0\textwidth]{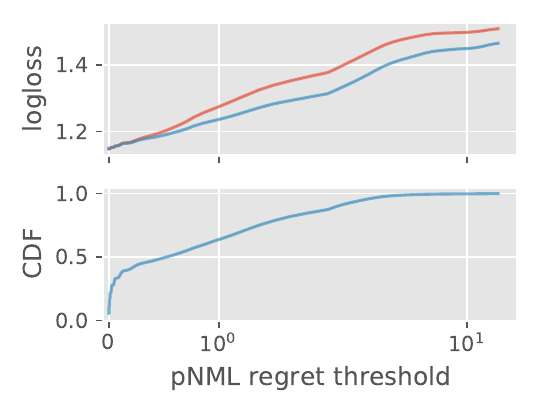}
    \caption{Kin8nm \label{fig:norm_constrained_pnml:kin8nm_regret}}
\end{subfigure}
\begin{subfigure}[t]{0.315\linewidth}
    \includegraphics[width=1.0\textwidth]{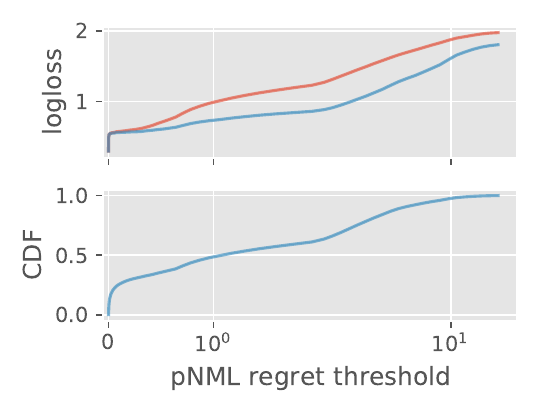}
    \caption{Naval Propulsion \label{fig:norm_constrained_pnml:naval-propulsion-plant_regret}}
\end{subfigure}
\begin{subfigure}[t]{0.315\linewidth}
    \includegraphics[width=1.0\textwidth]{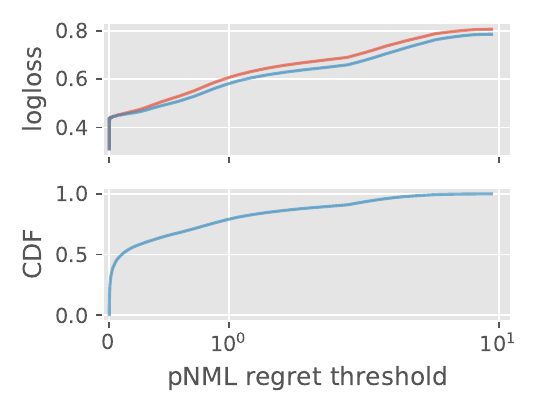}
    \caption{Combined Cycle Plant \label{fig:norm_constrained_pnml:power-plant_regret}}
\end{subfigure}
\begin{subfigure}[t]{0.315\linewidth}
    \includegraphics[width=1.0\textwidth]{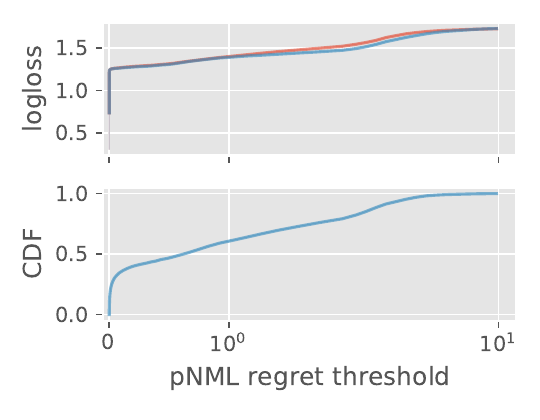}
    \caption{Protein Structure \label{fig:norm_constrained_pnml:protein-structure_regret}}
\end{subfigure}
\begin{subfigure}[t]{0.315\linewidth}
    \includegraphics[width=1.0\textwidth]{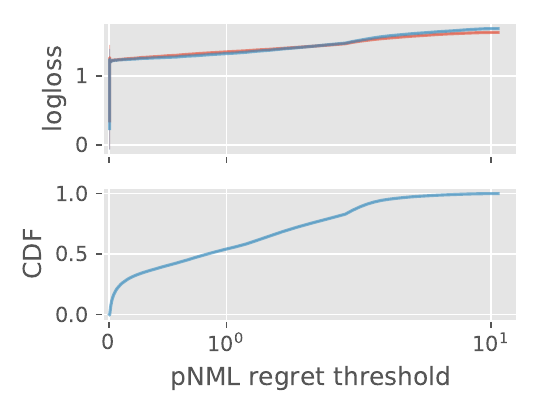}
    \caption{Wine Quality Red\label{fig:norm_constrained_pnml:wine-quality-red_regret}}
\end{subfigure}
\begin{subfigure}[t]{0.315\linewidth}
    \includegraphics[width=1.0\textwidth]{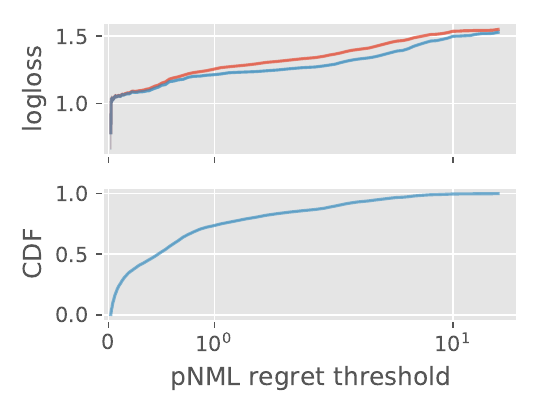}
    \caption{Yacht Hydrodynamics \label{fig:norm_constrained_pnml:yacht_regret}}
\end{subfigure}
\caption{The MN and pNML loss for samples with regrets lower than a threshold}
\label{fig:norm_constrained_pnml:regret_based_learner}
\end{figure}

\subsection{Double-descent with UCI dataset}
\label{sec:norm_constrained_pnml:double_descent}

\begin{table}[tb]
\small
\centering
\caption{UCI set characteristics}
\begin{tabular}{l c c c}
\toprule
Dataset name & $N$ & $M$ & \#Splits\\
\midrule
Boston Housing & 506 & 13 & 20 \\
Concrete Strength & 1,030 & 8 & 20 \\ 
Energy Efficiency & 768 & 8 & 20 \\
Kin8nm & 8,192 & 8 & 20 \\
Naval Propulsion & 11,934 & 16 & 20\\
Cycle Power Plant & 9,568 & 4 & 20\\ 
Protein Structure & 45,730 & 9  & 5 \\
Wine Quality Red & 1599 & 11 & 20\\
Yacht Hydrodynamics & 308 & 6 & 20\\
\bottomrule
\end{tabular}
\label{tab:norm_constrained_pnml:uci_meta_data}
\end{table}

Double descent is referred to as the phenomenon when beyond the interpolation limit,  the test error declines as model complexity increases~\citep{hastie2019surprises}.
We investigate the effect of varying the ratio between the number of parameters to training set size. 
We demonstrate that the pNML regret and its upper bound are correlated with the double-descent behavior of the test log-loss.

We the UCI repository~\citep{Dua:2019}:
We use the sets proposed by \citet{hernandez2015probabilistic} with the same test and train splits.
The training set size, number of features, and the number of train-test splits are shown in~\tableref{tab:norm_constrained_pnml:uci_meta_data}.
For each dataset, we varied the training set size and fit the pNML and MN learners. We optimize $\sigma^2$ on a validation set for both learners.

The test set log-loss as a function of the ratio between the number of parameters to training set size is presented in \figref{fig:norm_constrained_pnml:real_data_double_descent} (top) with the 95\% confidence interval that was calculated on different train-test splits.
Both the pNML and the MN learners behave the same: For a large training set size ($\frac{M}{N}<1$) the test set log-loss increases when removing training samples up to $M=N$. Then the log-loss declines although the training set size decreases.

The empirically calculated pNML regret and its analytical upper bound are shown in \figref{fig:norm_constrained_pnml:real_data_double_descent} (bottom).
Both empirically calculated regret and the derived upper bound present a similar double-descent behavior to the test log-loss. Their peak is for the number of features that equals the training set size and as $\frac{M}{N}$ increases their value decrease.

\begin{figure}[bt]
\centering
\begin{subfigure}[t]{0.32\linewidth}
    \includegraphics[width=\textwidth]{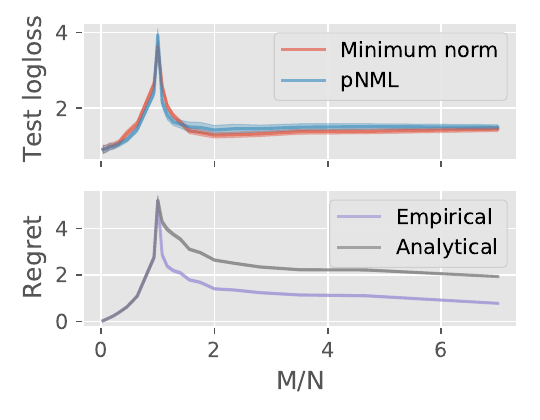} 
    \caption{Boston Housing \label{fig:norm_constrained_pnml:bostonHousing}}
\end{subfigure}
\begin{subfigure}[t]{0.32\linewidth}
    \includegraphics[width=\textwidth]{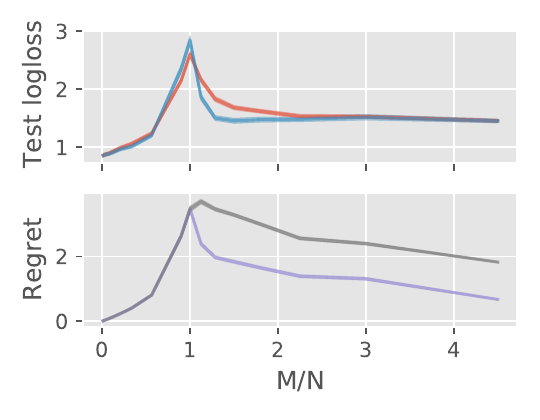}
    \caption{Concrete Strength \label{fig:norm_constrained_pnml:concrete}}
\end{subfigure}
\begin{subfigure}[t]{0.32\linewidth}
    \includegraphics[width=\textwidth]{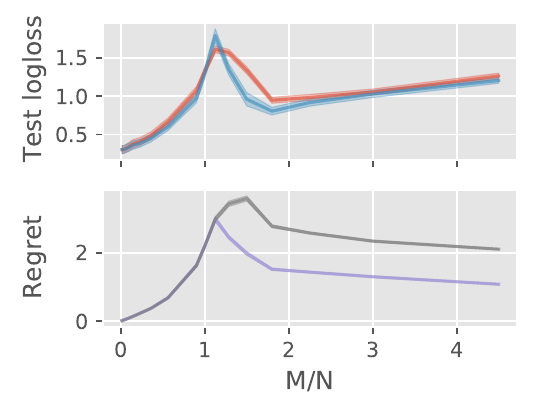}
    \caption{Energy Efficiency \label{fig:norm_constrained_pnml:energy}}
\end{subfigure}
\begin{subfigure}[t]{0.32\linewidth}
    \includegraphics[width=\textwidth]{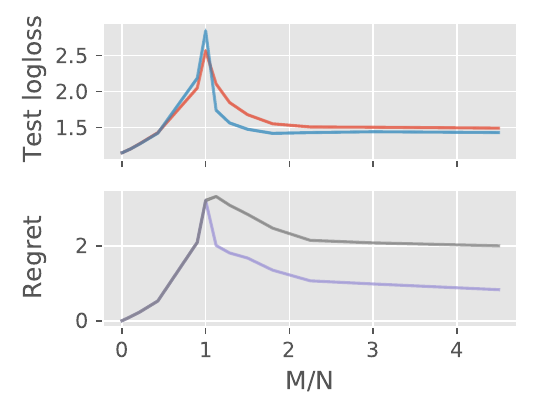}
    \caption{Kin8nm \label{fig:norm_constrained_pnml:kin8nm}}
\end{subfigure}
\begin{subfigure}[t]{0.32\linewidth}
    \includegraphics[width=\textwidth]{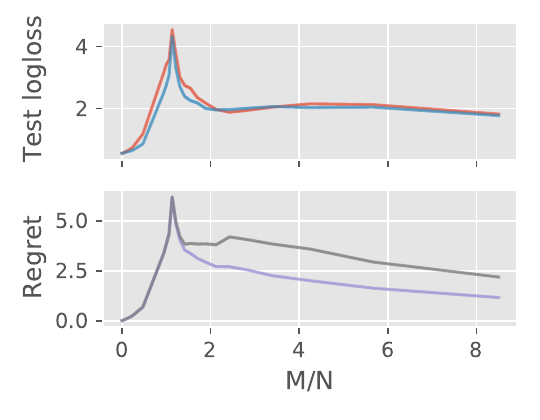}
    \caption{Naval Propulsion \label{fig:norm_constrained_pnml:naval-propulsion-plant}}
\end{subfigure}
\begin{subfigure}[t]{0.32\linewidth}
    \includegraphics[width=\textwidth]{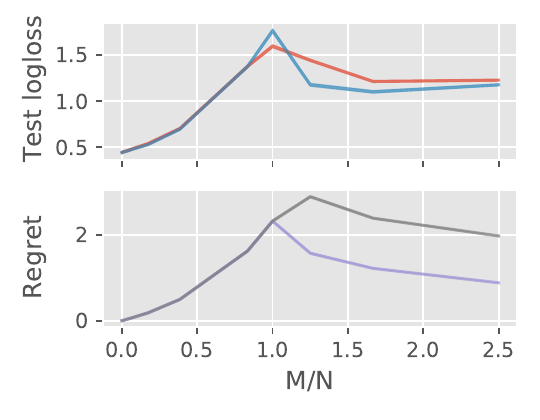}
    \caption{Combined Cycle Power Plant \label{fig:norm_constrained_pnml:power-plant}}
\end{subfigure}
\begin{subfigure}[t]{0.32\linewidth}
    \includegraphics[width=\textwidth]{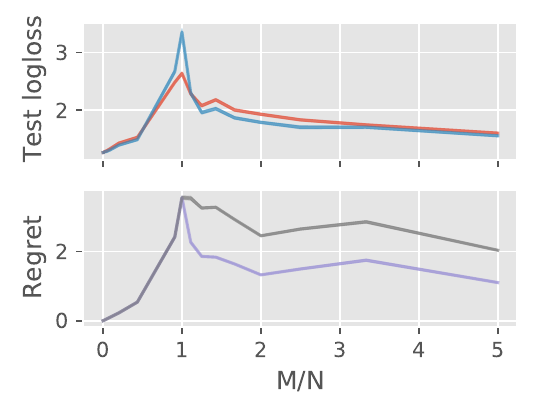}
    \caption{Protein Structure \label{fig:norm_constrained_pnml:protein-structure}}
\end{subfigure}
\begin{subfigure}[t]{0.32\linewidth}
    \includegraphics[width=\textwidth]{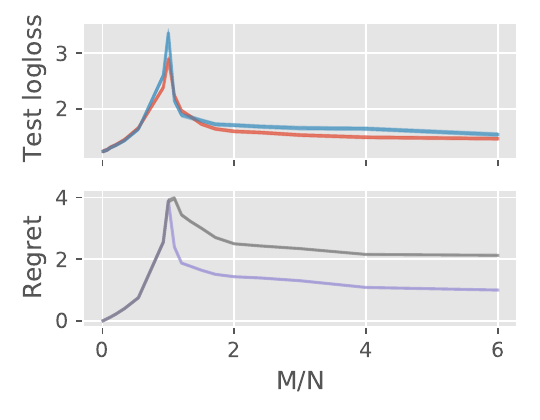}
    \caption{Wine Quality Red \label{fig:norm_constrained_pnml:wine}}
\end{subfigure}
\begin{subfigure}[t]{0.32\linewidth}
    \includegraphics[width=\textwidth]{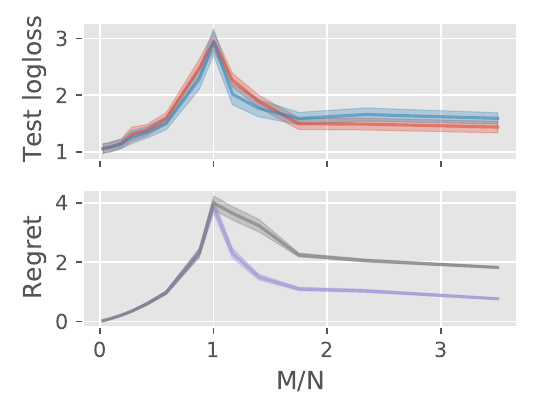}
    \caption{Yacht Hydrodynamics \label{fig:norm_constrained_pnml:yacht}}
\end{subfigure}
\caption{Double-descent of the test log-loss and regret for UCI datasets}
\label{fig:norm_constrained_pnml:real_data_double_descent}
\end{figure}

\section{Concluding remarks} 
\label{sec:norm_constrained_pnml:conclusion}
We derived an analytical upper bound of the pNML regret which is associated with the prediction uncertainty for over-parameterized linear regression.
The pNML prediction equals the MN solution thus the derived regret can be used to quantify the prediction uncertainty of the MN solution.
The derived result holds for a wide range of scenarios as we considered the individual setting where there is no assumption of a probabilistic relationship between the training and test.

Analyzing the pNML regret we can observe that if a test sample lies in the subspace spanned by the eigenvectors associated with large eigenvalues of the training data correlation matrix then over-parameterized linear regression generalizes well.
Finally, we provided simulations of the pNML for real trigonometric polynomial interpolation. We showed that the pNML regret can be used as a confidence measure and can is correlated with the test error double-descent phenomenon for 9 sets from the UCI repository.

\label{sec:norm_constrained_pnml:limitations}
For future work, we would like to derive the explicit expression of the pNML regret rather than an upper bound. In addition, the pNML regret can be used for additional tasks such as active learning, probability calibration, and adversarial attack detection.

\chapter{The Luckiness Perspective}  
\label{chap:luckiness}
\section{Introduction}
Ridge regression is a widely used method for linear regression when the data dimension is large compared to the training set size.
It has been applied in a large variety of domains such as econometrics~\cite{sengupta2020asymptotic}, bioinformatics~\cite{xu2020blood}, and social science~\cite{grimmer2021machine}.
From a Bayesian perspective, it coincides with the mean of the predictive distribution where the parameter prior and noise are Gaussian~\cite{DBLP:conf/iclr/LiuD20}.

The most popular variant is the \textit{Ridge empirical risk minimizer} (Ridge ERM): the model is chosen to minimize the training set loss and the ridge parameter is selected either to minimize a validation set or with the leave-one-out protocol, which leads to the same asymptotic performance as the optimally-tuned ridge estimator~\cite{hastie2019surprises}.

When using Ridge ERM, the underlying assumption is that there is a probabilistic relationship between the data and labels and between the training and test. 
In the \textit{stochastic setting}, see \citet{merhav1998universal}, it is assumed that the probabilistic relation between the test feature $x$ and its label $y$ is given by an (unknown) model from a given hypothesis class $P_\Theta$. 
For \textit{the probably approximately correct} (PAC) setting~\cite{valiant1984theory}, $x$ and $y$ are assumed to be generated by some source $P(x,y)=P(x)P(y|x)$ which is not necessarily a member of the hypothesis class.
These assumptions, however, may not hold in a real-world scenario.

Recall the pNML learner learner minimizes the regret for the worst-case test label 
\begin{equation} \label{eq:linear_regression_with_luckiness:minmax_prob}
\Gamma = \min_q \max_y R(\DN,x,y,q).
\end{equation}
and its predictive distribution is
\begin{equation} \label{eq:linear_regression_with_luckiness:pNML}
q_{\mbox{\tiny{pNML}}}(y|x)=\frac{\probthetagenie}{\int \probthetagenietag dy'} .
\end{equation}
However, the pNML may not be defined for an \textit{over-parameterized} hypothesis class, where the number of parameters exceeds the training set size. The reason is that in the denominator of~\eqref{eq:linear_regression_with_luckiness:pNML}, every possible value of the test label $y'$ can be perfectly fitted such that the integral diverges.

The pNML root lies in the \textit{normalized maximum likelihood} (NML) approach for online prediction~\cite{shtar1987universal}.
Since the NML may also be improper, a leading solution is \textit{NML with luckiness} (LNML)~\cite{roos2004mdl}: A luckiness function $w(\theta)$ is designed such that on sequences with small $w(\theta)$, we are prepared to incur large regret~\cite{grunwald2007minimum}. 
This is the equivalent to a Bayesian model prior.

In this chapter, we apply the luckiness concept to the pNML and call it LpNML.
For linear regression, we design the luckiness function to be proportional to the $\ell_2$ model norm.
This leads to a genie that equals ridge regression.
We derive the corresponding LpNML and show its \textit{prediction}, i.e., the mean of the predictive distribution, differs from the Ridge ERM's:
When the test sample lies within the subspace associated with the small eigenvalues of the empirical correlation matrix of the training data, the prediction is shifted toward 0. This behavior is shown in~\figref{fig:linear_regression_with_luckiness:lpnml_prediction}.

We demonstrate the LpNML attains a better MSE than Ridge ERM for 50 real-world PMLB sets~\cite{Olson2017PMLB}, reducing the error by up to 20\%. Furthermore, we show the LpNML outperforms leading methods for the distribution-shift benchmark~\cite{tripuraneni2020single}, when the test set differs from training.

\begin{figure}[t]
\centering
\begin{subfigure}{0.49\textwidth}
\centering
    \includegraphics[width=\textwidth]{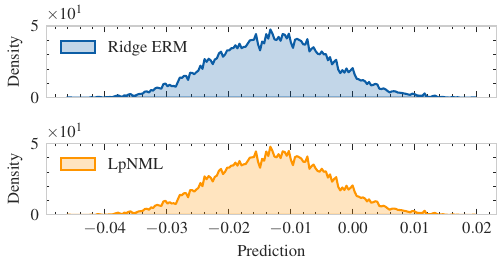}
    \caption{Greater than 0 eigenvalue subspace}
    \label{fig:linear_regression_with_luckiness:lpnml_parallel}
\end{subfigure}
\begin{subfigure}{0.49\textwidth}
\centering
    \includegraphics[width=\textwidth]{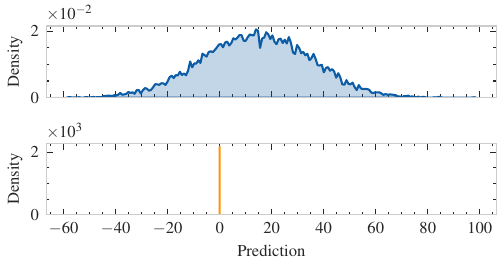}
    \caption{Equal to 0 eigenvalue subspace}
    \label{fig:linear_regression_with_luckiness:lpnml_orthogonal}
\end{subfigure}
\caption{Ridge ERM and LpNML prediction histograms}
\label{fig:linear_regression_with_luckiness:lpnml_prediction}
\end{figure}

\section{Notation and preliminaries} \label{sec:linear_regression_with_luckiness:preliminaries}
\paragraph{Linear regression setting.}
Given a training set
\begin{equation}
\DN = \{(x_n, y_n)\}_{n=1}^{N} , \quad x_n \in R^{M \times 1}, \quad y_n \in R,
\end{equation}
the goal is to predict the test label $y$ based on a new data sample $x$.
A common assumption is a linear relationship between the data and labels with an additive white noise
\begin{equation} \label{eq:linear_regression_with_luckiness:linear}
    y_n = \theta^\top x_n + e_n, \qquad e_n \sim \mathcal{N}(0,\sigma^2).
\end{equation}
Denote the design matrix and the training set label vector
\begin{equation} \label{eq:linear_regression_with_luckiness:trainset_matrix}
X_N = \begin{bmatrix} x_1 & \dots & x_N \end{bmatrix}^\top,  \quad 
Y_N = \begin{bmatrix} y_1 & \dots & y_N  \end{bmatrix}^\top
\end{equation}
where $X_N \in R^{N \times M} $ and $Y_N  \in R^{N \times 1}$,
the ERM solution which minimizes the log-loss of the training set (and is also the maximum likelihood estimator) is
\begin{equation} \label{eq:linear_regression_with_luckiness:linear_regression_erm}
\hat{\theta} = (X_N^\top X_N)^{-1} X_N^\top Y_N.
\end{equation}

\paragraph{Ridge ERM.}
In regularized linear regression, leading to the Ridge ERM learner, the hypothesis class $\Theta$ is a sphere $\norm{\theta}^2_2\leq A$~\cite{ridgeregression}. 
Expressing this constraint with the Lagrangian:
\begin{equation}
\mathcal{L}(\theta, \lambda)= \norm{Y_N - X_N \theta}^2_2 + \lambda \left( \norm{\theta}^2_2 - A \right),
\end{equation}
the Ridge ERM learnable vector is
\begin{equation} \label{eq:linear_regression_with_luckiness:ridge_erm}
\thetalamb = \left(X_N^\top X_N + \lambda I \right)^{-1} X_N^\top Y_N.
\end{equation}

\paragraph{Bayesian linear regression.} %
Defining the Gaussian prior of the learnable parameters 
\begin{equation}
p(\theta) = \mathcal{N}\left(m_0,S_0 \right),
\end{equation}
the posterior distribution is (see \citet{deisenroth2020mathematics})
\begin{equation}
p(\theta|D_N) = \mathcal{N} \left(\theta^\top m_N, S_N\right), 
S_N = \big(S_0^{-1} + \sigma^{-2} X_N^\top X_N \big)^{-1},
m_N = S_N \big(S_0^{-1} m_0 + \sigma^{-2} X_N^\top Y_N\big).
\end{equation}
To compute the predictive distribution with the posterior distribution, the parameters are integrated out based on their posterior probabilities. With $m_0=0$ and $S_0 = \sigma^{-2} \lambda I$, the predictive distribution is given by:
\begin{equation}
q_\textit{\tiny{Bayesian}}(y|x) = \int p(y|x,\theta) p(\theta|\DN) d\theta
=
\mathcal{N}\left(\thetalamb^\top x,\sigma^2\left[1+ x^\top \left(X_N^\top X_N + \lambda I \right)^{-1}  x \right] \right).
\end{equation}
The mean of the Bayesian learner's predictive distribution matches the Ridge ERM solution of~\eqref{eq:linear_regression_with_luckiness:ridge_erm}. The prediction uncertainty is captured in the variance, which decreases as the test sample aligns more strongly with the design matrix, reflecting reduced uncertainty in well-supported regions of the feature space.
\paragraph{The LNML.}
For online prediction with individual sequences, the min-max optimal regret is given by the NML~\cite{shtar1987universal,grunwald2007minimum}.
The prediction is performed for the entire sequence: Denote the sequence $y^N = \left\{y_n \right\}_{n=1}^N$, the NML probability assignment is 
\begin{equation} \label{eq:linear_regression_with_luckiness:nml}
q_\textit{\tiny{NML}}(y^N) = \frac{\max_\theta p_\theta(y^N)}{\int \max_\theta p_\theta (y'^N) dy'^N}.
\end{equation}

Since NML may be improper, several treatments have been proposed. Among these treatments is setting a restriction on the range of data or the range of parameters~\cite{hirai2011efficient}.
The drawback of this method is that samples can fall outside of any valid restrictions.
A different approach is the LNML that is also named ``generalized NML''~\cite{roos2004mdl}:
A luckiness function $w(\theta)$ is set such that the sequence distribution becomes
\begin{equation} \label{eq:linear_regression_with_luckiness:lnml}
q_\textit{\tiny{LNML}}(y^N) = \frac{ \max_\theta  p_\theta(y^N) w(\theta)}{\int \max_\theta p_\theta (y^N) w(\theta) dy^N}.
\end{equation}
The advantage of LNML is that there is large freedom in choosing the luckiness function: We can choose a function that has particularly pleasant properties~\cite{grunwald2007minimum}.
Choosing the luckiness function to equal a constant for example, reduces the LNML back to the NML.

\citet{miyaguchi2017normalized} derived the LNML for multivariate normal distributions with the conjugate prior luckiness function.
\citet{dwivedi2021revisiting} incorporated a luckiness to the LNML that is proportional to the model norm to find the best ridge regularization factor $\lambda$. 
As opposed to our approach, once $\lambda$ was found, the prediction equals the Ridge ERM.

\paragraph{Transductive prediction.}
A transductive inference uses unlabeled data to improve predictions of unlabeled examples. \citet{chapelle2000transductive} chose the test label values to minimize the leave-one-out error of ridge regression with both training and test data. \citet{cortes2007transductive} estimated the label of the unlabeled test data $x$ by using only the labeled neighbors of $x$. They also presented error bounds for the VC-dimension. \citet{alquier2012transductive} established risk bounds for a transductive version of the Lasso learner. \citet{lei2021near} developed a min-max linear estimator in the presence of a covariate shift, where the maximum is for the target domain learnable parameters.

Although all of the above have demonstrated empirical and theoretical benefits, many unlabeled test points need to be simultaneously available.
\citet{tripuraneni2020single} presented a single point transductive procedure for linear regression that improves the prediction root mean square deviation (RMSE), especially under a distribution shift.

\section{pNML with luckiness}
Inspired by the luckiness concept for NML, we define the genie, a learner that knows the true test label, with a luckiness function $w(\theta)$ as follows
\begin{equation}
\hat{\theta}_y = \argmin_{\theta \in \Theta} \bigg[ \sum_{n=1}^N \ell(p_\theta,x_n,y_n) + \ell(p_\theta,x,y) - \log w\left(\theta\right)  \bigg].
\end{equation}
The luckiness function is used as a model prior: The genie is more likely to select $\theta$ that yields a larger $w(\theta)$.
The related regret in this setting is
\begin{equation}
R(q,D_N,x,y,w) =  \frac{\log p_{\hat{\theta}_y} (y|x) w(\hat{\theta}_y)}{\log q(y|x)}.
\end{equation}

\begin{theorem}
The LpNML is the learner that minimizes the worst-case regret objective
\begin{equation}
   q_\textit{\tiny{LpNML}}(y|x) = \argmin_q \max_y R(q,D_N,x,y,w).
\end{equation}
The LpNML predictive distribution is
\begin{equation} \label{eq:linear_regression_with_luckiness:lpnml}
q_\textit{\tiny{LpNML}}(y|x) = \frac{\probthetay w(\thetay)}{\int \probthetaytag w(\thetaytag) dy'}
\end{equation}
and its min-max regret is
\begin{equation}
\Gamma = \log \int \probthetaytag w(\thetaytag) dy'.
\end{equation}
\end{theorem}
\begin{proof}
This proof essentially follows that of~\citet{fogel2019universal} with the additional luckiness function:
The LpNML has a valid predictive distribution $\int q_\textit{LpNML}(y|x) dy = 1$.
The min-max regret of the LpNML is equal for all choices of $y$. If we consider a different predictive distribution, it should assign a smaller probability to at least one of the outcomes. If the true label is one of those outcomes, it will result in a greater regret. 
\end{proof}

Intuitively, the LpNML~\eqref{eq:linear_regression_with_luckiness:lpnml} assigns a probability for a potential test label as follows: (i) Add the test sample to the training set with an arbitrary label $y'$, (ii) find the ERM solution with the defined luckiness function of this new set $\hat{\theta}_{y'}$, and (iii) take the probability it gives to the assumed label weighted by the luckiness function $p_{\hat{\theta}_{y'}}(y'|x) w(\thetaytag)$. Follow (i)-(iii)  for every possible test label value and normalize to get a valid predictive distribution. In the next section, we analytically derive the LpNML for ridge regression.

\subsection{LpNML for ridge regression}
We formulate the luckiness function as follows:
\begin{equation} \label{eq:linear_regression_with_luckiness:luckiness_ridge}
w(\theta) = \exp\bigg\{-\frac{\lambda}{2\sigma^2}\norm{\theta}^2 \bigg\}.
\end{equation}
The genie learnable parameters using this luckiness function is the solution of the following minimization objective
\begin{equation}
\thetay  =
\argmin_\theta \bigg[ \sum_{n=1}^N \left(y_n - \theta^\top x_n \right)^2 + \left(y - \theta^\top x \right)^2 + \lambda \norm{\theta}^2\bigg].
\end{equation}
The solution is ridge regression which we express with the recursive least squares formulation~\cite{hayes19969}
\begin{equation} \label{eq:linear_regression_with_luckiness:genie_rls}
\thetay  = \thetalamb + \frac{\Plamb x}{\Klamb}  \left(y - \thetalamb^\top x\right),
\end{equation}
where 
\begin{equation}
\Plamb \triangleq \left(X_N^\top X_N + \lambda I \right)^{-1}, \quad  \Klamb \triangleq 1+ x^\top \Plamb x.
\end{equation}

The goal is to analytically derive the LpNML predictive distribution \eqref{eq:linear_regression_with_luckiness:lpnml} with the luckiness function of \eqref{eq:linear_regression_with_luckiness:luckiness_ridge}. For this, we first have to determine the genie predictive distribution of the true test label.

\begin{lemma}
The genie predictive distribution weighted by the luckiness function of \eqref{eq:linear_regression_with_luckiness:luckiness_ridge} is 
\begin{equation} \label{eq:linear_regression_with_luckiness:lemma1}
\probthetay w(\thetay) = 
\frac{c}{\sqrt{2\pi\sigma^2}}
\exp\left\{- \frac{1}{2 \sigmalpnml} \left(y-\thetalamb^\top x + \mulpnml \right)^2 
\right\}
\end{equation}
where 
\begin{equation} \label{eq:linear_regression_with_luckiness:lpnml_mu_sigma}
\begin{split}
& \mulpnml \triangleq \frac{\lambda \Klamb \thetalamb^\top \Plamb x}{1+\lambda x^\top \Plamb^2 x}
, \quad
\sigmalpnml \triangleq \frac{\sigma^2 \Klamb^2}{1+\lambda x^\top \Plamb^2 x},
\\ & \qquad
c \triangleq 
\exp\bigg\{\frac{1}{2\sigma^2}\bigg[
\frac{\big(\lambda \thetalamb^\top \Plamb x\big)^2}{1+\lambda x^\top \Plamb^2 x}
- 
\lambda \norm{\thetalamb}^2 
\bigg]
\bigg\}.
\end{split}
\end{equation}
\end{lemma}
\begin{proof}
The genie probability assignment weighted by the luckiness function is
\begin{equation} \label{eq:linear_regression_with_luckiness:genie_prob_derivation}
\begin{split}
&\probthetay w(\thetay) 
=
\frac{1}{\sqrt{2\pi\sigma^2}}
\exp\left\{
-\frac{1}{2\sigma^2}
\left(y - \thetay^\top x \right)^2  
\right\}
\exp\left\{
-\frac{\lambda}{2\sigma^2} \norm{\thetay}^2
\right\}.
\end{split}
\end{equation}
The first exponent argument using the recursive least squares formulation~\eqref{eq:linear_regression_with_luckiness:genie_rls}:
\begin{equation}
\begin{split}
\big(y - \thetay^\top x \big)^2  
&= 
\left[y
- \thetalamb^\top x
-
\frac{1}{\Klamb} x^\top \Plamb \big(y - \thetalamb^\top x\big)x \right]^2 
\\
&=
\left[ \bigg(1 -\frac{x^\top \Plamb x}{\Klamb}\bigg) y 
-
\left(
1
-
\frac{x^\top \Plamb x}{\Klamb}
\right)
\thetalamb^\top x \right]^2
\\ &=
\frac{1}{\Klamb^2}\big(y - \thetalamb^\top x \big)^2.
\end{split}
\end{equation}
Substituting it back to~\eqref{eq:linear_regression_with_luckiness:genie_prob_derivation} and using again the recursive least squares
\begin{equation}
\begin{split}
& \probthetay  w(\thetay) = 
\frac{1}{\sqrt{2\pi\sigma^2}}  \exp\bigg\{ 
-\frac{\big(y - \thetalamb^\top x\big)^2}{2 \sigma^2 \Klamb^2} 
-\frac{\lambda}{2\sigma^2} \bnorm{\thetalamb + \frac{\Plamb x}{\Klamb}  \big(y - \thetalamb^\top x\big)}^2
\bigg\}.
\end{split}
\end{equation}
Deriving the exponential argument:
\begin{equation}
\begin{split}
& \frac{1}{\Klamb^2}
\big(y - \thetalamb^\top x\big)^2
+\lambda \bnorm{ \thetalamb + \frac{\Plamb x}{\Klamb}  \big(y - \thetalamb^\top x\big)}^2
\\ & \qquad = 
\frac{1 + \lambda x^\top \Plamb^2 x}{\Klamb^2} \big(y - \thetalamb^\top x\big)^2
+ 
2\frac{\lambda \thetalamb^\top \Plamb x}{\Klamb} \big(y - \thetalamb^\top x\big)
+ 
\lambda \norm{\thetalamb}^2  
\\ & \qquad =
\frac{1+\lambda x^\top \Plamb^2 x}{\Klamb^2}
\big(y - \thetalamb^\top x +  \frac{\lambda \Klamb \thetalamb^\top \Plamb x}{1+\lambda x^\top \Plamb^2 x} \big)^2 
-
\frac{\big(\lambda \thetalamb^\top \Plamb x\big)^2}{1+\lambda x^\top \Plamb^2 x}
+ \lambda \norm{\thetalamb}^2 .
\end{split}
\end{equation}
Substitute in \eqref{eq:linear_regression_with_luckiness:genie_prob_derivation} proves the lemma.
\end{proof}

$\mulpnml$ is the deviation of the genie prediction from the Ridge ERM prediction. 
If $\lambda=0$, the deviation is 0 and the genie prediction is equal to the ERM solution of~\eqref{eq:linear_regression_with_luckiness:linear_regression_erm}.
The genie predictive distribution is not valid since $\int \probthetay dy>1$.
Next, we derive the LpNML by normalizing the genie predictive distribution weighted by the luckiness function.
\begin{theorem}
With the luckiness function of~\eqref{eq:linear_regression_with_luckiness:luckiness_ridge}, the LpNML distribution is
\begin{equation}
q_{\textit{\tiny{LpNML}}}(y|x) = \mathcal{N}\left(\thetalamb^\top x - \mulpnml, \sigmalpnml \right).
\end{equation}
\end{theorem}
\begin{proof}
Following \eqref{eq:linear_regression_with_luckiness:lpnml}, to get the normalization factor we integrate the genie predictive distribution over all possible test label values. Utilizing lemma 1:
\begin{equation} \label{eq:linear_regression_with_luckiness:nf_derivation}
\begin{split}
& K_\textit{\tiny{LpNML}} = \int_{- \infty}^{\infty} \probthetaytag w(\thetaytag) dy' 
\\ &=
\int_{-\infty}^{\infty} \frac{c}{\sqrt{2\pi\sigma^2}}
\exp\bigg\{-\frac{\big(y'-\thetalamb^\top x + \mulpnml \big)^2}{2 \sigmalpnml}  
\bigg\} dy'
\\ &=
\frac{c}{\sqrt{2\pi \sigma^2}} \sqrt{2\pi \sigmalpnml} = c \sqrt{\frac{\sigmalpnml}{\sigma^2}}.
\end{split}
\end{equation}
Dividing \eqref{eq:linear_regression_with_luckiness:lemma1} by the normalization factor \eqref{eq:linear_regression_with_luckiness:nf_derivation}, the LpNML solution is obtained.
\end{proof}
The LpNML prediction deviates from the Ridge ERM prediction by $\mulpnml$. Both $\mulpnml$ and the LpNML variance $\sigmalpnml$ depend on the test data $x$. 
When $\lambda=0$, the LpNML reduces to the pNML solution for under-parameterized linear regression that was derived in chapter~\ref{chap:ordinary_least_squares_regression} (and published in \cite{pNML_linear_regression})
\begin{equation}
    q_\textit{\tiny{pNML}} = \mathcal{N}\left(\thetahat^\top x, \sigma^2\left[1 + x\left(X_N^\top X_N \right)^{-1} x\right]^2\right).
\end{equation}
In the next section, we show that if the test data fall within the subspace corresponding to the small eigenvalues of the design matrix $X_N$, the deviation from Ridge ERM $\mulpnml$ is large which shifts the LpNML prediction toward 0.


\subsection{The learnable subspace} \label{sec:linear_regression_with_luckiness:learnable_subpsace}
Focusing on over-parameterized linear regression, we analyze the dissimilarity between the LpNML and the Ridge ERM predictions:
\begin{equation}
\mulpnml 
=
\frac{\lambda \Klamb  \thetalamb^\top \Plamb x}{1+\lambda x^\top \Plamb^2 x}   
=
\frac{\lambda \left(1+x^\top \Plamb x\right) \thetalamb^\top \Plamb x}{1+\lambda x^\top \Plamb^2 x}.
\end{equation}
Let $u_m$ and $h_m$ be the $m$-th eigenvector and eigenvalue of the design matrix such that
\begin{equation}
X_N^\top X_N = \sum_{m=1}^M  h_m^2 u_m u_m^\top,
\end{equation}
for over-parameterized linear regression ($M > N$)
\begin{equation}
\begin{split}
x^\top P_\lambda x &=
\sum_{m=1}^N \frac{\left(u_m^\top x\right)^2}{h_m^2 + \lambda}
+
\sum_{m=N+1}^M \frac{\left(u_m^\top x\right)^2}{\lambda},
\\ 
x^\top P_\lambda^2 x &=
\sum_{m=1}^N \frac{\left(u_m^\top x\right)^2}{\left(h_m^2 + \lambda\right)^2}
+
\sum_{m=N+1}^M \frac{\left(u_m^\top x\right)^2}{\lambda^2}.
\end{split}
\end{equation}
We analyze two cases: A case where $x$ falls in the \textit{largest eigenvalue subspace} and a case where $x$ lies in the \textit{smallest eigenvalue subspace}.

\paragraph{Largest eigenvalue subspace.}
For a test feature $x$ that lies in the subspace that is spanned by the eigenvectors that are associated with the large eigenvalues of the design matrix $x^\top \Plamb x \ll 1$. The deviation from the Ridge ERM is
\begin{equation} \label{eq:linear_regression_with_luckiness:lpnml_mu_parallel}
\begin{split}
& \mulpnml
=
\frac{\lambda \left(1+x^\top \Plamb x\right) \thetalamb^\top \Plamb x}{1+\lambda x^\top \Plamb^2 x} 
\approx
\lambda \thetalamb^\top \Plamb x 
\\ & \qquad \qquad =
\lambda Y_N^\top  X_N \Plamb^2 x 
=
\lambda Y_N^\top X_N \sum_{m=1}^N \frac{u_m u_m^\top x}{\left(h_m^2 + \lambda \right)^2}
\ll
1.
\end{split}
\end{equation}
The LpNML prediction is similar to the Ridge ERM.
The LpNML variance in this case is
\begin{equation} \label{eq:linear_regression_with_luckiness:lpnml_var_parallel}
\sigmalpnml = \frac{\sigma^2 \Klamb^2}{1 + \lambda x^\top \Plamb^2 x}
=
\frac{\sigma^2 \left(1 + x^\top \Plamb x \right)^2}{1 + \lambda x^\top \Plamb^2 x } \approx \sigma^2.
\end{equation}
There is a high confidence in the prediction since this is the smallest possible variance.

\paragraph{Smallest eigenvalue subspace.}
For a test vector that lies in the subspace that is spanned by the eigenvectors corresponding to the smallest eigenvalue of the regularized empirical correlation matrix of the training data
\begin{equation}
x^\top \Plamb x =  \sum_{m=N+1}^M \frac{\left(u_m^\top x\right)^2}{\lambda^2} = \frac{1}{\lambda} ||x||^2.   
\end{equation}
For a small regularization term $\frac{1}{\lambda} ||x||^2 \gg 1$, the deviation from the Ridge ERM is
\begin{equation} \label{eq:linear_regression_with_luckiness:lpnml_mu_orthogonal}
\begin{split}
\mulpnml
&=
\frac{\lambda \left(1+ \frac{||x||^2}{\lambda} \right) \thetalamb^\top \frac{x}{\lambda}}{1+\lambda \frac{||x||^2}{\lambda^2}}  
\approx
\frac{\lambda \frac{||x||^2}{\lambda} \thetalamb^\top \frac{x}{\lambda}}{\lambda \frac{||x||^2}{\lambda^2}}   
= 
\thetalamb^\top  x.
\end{split}
\end{equation}
The correction term equals the Ridge ERM prediction thus the LpNML prediction is shifted to 0. 
The LpNML variance in this situation is
\begin{equation} \label{eq:linear_regression_with_luckiness:lpnml_var_orthogonal}
\sigmalpnml =
\sigma^2
\frac{\left(1 + \frac{||x||^2}{\lambda}\right)^2}{1 + \lambda \frac{||x||^2}{\lambda^2} } 
=
\sigma^2 \left(1 + \frac{||x||^2}{\lambda}\right).
\end{equation}
Compared to \eqref{eq:linear_regression_with_luckiness:lpnml_var_parallel}, the variance is large which reflects high uncertainty in the prediction.
In~\secref{sec:linear_regression_with_luckiness:experiments}, we empirically show the behavior of the LpNML variance on a synthetic set and demonstrate that the LpNML deviation from the Ridge ERM improves its performance for real-world datasets.

\section{Experiments} \label{sec:linear_regression_with_luckiness:experiments}
We demonstrate the LpNML prediction behavior for fitting a polynomial function to synthetic data and for fitting over-parameterized linear regression to a high-dimensional synthetic dataset. In addition, we show that the LpNML outperforms the Ridge ERM for real-world PMLB datasets and attains state-of-the-art performance for the distribution-shift benchmark

\subsection{Polynomial fitting to synthetic data} \label{sec:linear_regression_with_luckiness:exp_polynomial}
We sampled 6 training points $\left\{(t_n,y_n\right\}_{n=1}^6$ uniformly in the interval $[-1, 1]$.
and converted the data to features with a polynomial of degree 10 such that the design matrix is
\begin{equation} \label{eq:linear_regression_with_luckiness:experiment_data_matrix}
X_N = 
\begin{bmatrix}
1 & t_1 & \dots & t_1^{10} \\
1 & t_2 & \dots & t_2^{10} \\
\vdots & \vdots & \ddots & \vdots \\
1 & t_6 & \dots & t_{6}^{10}
\end{bmatrix}.
\end{equation}
Based on this training set, we performed Bayesian and LpNML prediction for all $t$ values in the interval $[-1,1]$. We set the regularization $\lambda$ to $10^{-3}$.

\Figref{fig:linear_regression_with_luckiness:pnml_syntetic_data} shows the Ridge ERM, Bayesian and LpNML predictions.
The Bayesian learner has the same prediction as to the Ridge ERM which is different from the LpNML: The LpNML prediction is closer to 0 in intervals that lack training data, e.g., $t \leq 0.75$ and $0.9 \leq t$.

The confidence intervals are shown in \figref{fig:linear_regression_with_luckiness:pnml_syntetic_data} with lighter colors. Both learners have large confidence intervals at the figure's edges and for $0.1 \leq t \leq 0.8$, where the training points are scarce. The LpNML has much larger confidence intervals at the interval edges than the Bayesian learner.

\begin{figure}[bt]
\centering
\includegraphics[width=0.6\columnwidth]{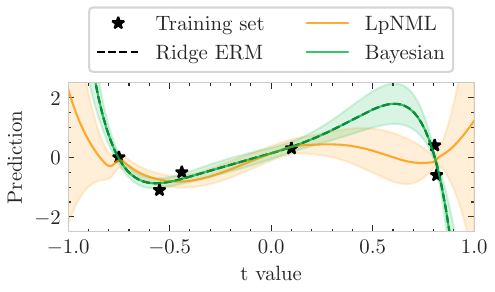}
\caption{Polynomial fitting to synthetic data} 
\label{fig:linear_regression_with_luckiness:pnml_syntetic_data}
\end{figure}

\subsection{Parallel and orthogonal subspace} \label{sec:linear_regression_with_luckiness:exp_lpnml_prediction}
We illustrate the LpNML behavior for test samples that lie in the largest eigenvalue subspace and smallest eigenvalue subspace as analyzed in~\secref{sec:linear_regression_with_luckiness:learnable_subpsace}.

We created a synthetic dataset with $N=40$ training samples and $M=100$ features. We set the regularization term $\lambda$ to $10^{-9}$.
We randomly sampled $10,000$ test samples for each of the following two scenarios: The test sample data reside in the subspace that is spanned by the eigenvectors of the design matrix corresponding to the eigenvalues that are (a) greater than 0, or (b) equal 0.

\Figref{fig:linear_regression_with_luckiness:lpnml_parallel} shows the histogram of Ridge ERM and LpNML predictions for test samples that reside in scenario (a).
The LpNML prediction for these test samples equals the Ridge ERM prediction which verifies equation~\eqref{eq:linear_regression_with_luckiness:lpnml_mu_parallel} result.

\Figref{fig:linear_regression_with_luckiness:lpnml_orthogonal} presents Ridge ERM and LpNML prediction histograms of samples from scenario (b).
The pNML predicts 0 while the Ridge ERM prediction varies between $-60$ and $100$, which is 2 order of magnitude larger than the prediction of scenario (a). The LpNML prediction of 0 is aligned with equation~\eqref{eq:linear_regression_with_luckiness:lpnml_mu_orthogonal}: 
To avoid a large log-loss, the LpNML shifts the Ridge prediction of test samples that differ from the training data to 0.

\begin{figure}[bth]
\centering
\begin{subfigure}{0.49\textwidth}
\centering
    \includegraphics[width=\textwidth]{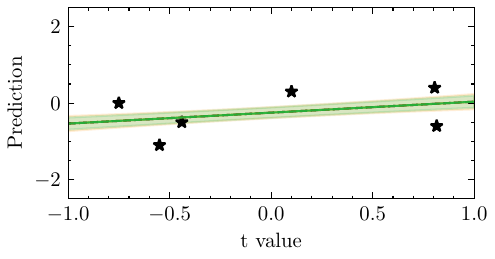}
    \caption{$M=2$}
    \label{fig:linear_regression_with_luckiness:m_2}
\end{subfigure}
\begin{subfigure}{0.49\textwidth}
\centering
    \includegraphics[width=\textwidth]{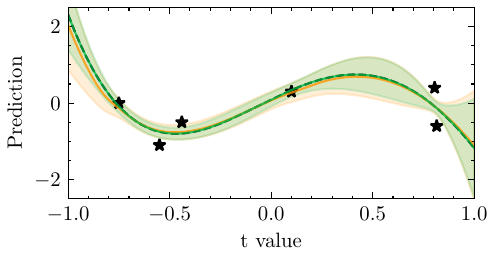}
    \caption{$M=5$}
    \label{fig:linear_regression_with_luckiness:m_5}
\end{subfigure}
\begin{subfigure}{0.49\textwidth}
\centering
    \includegraphics[width=\textwidth]{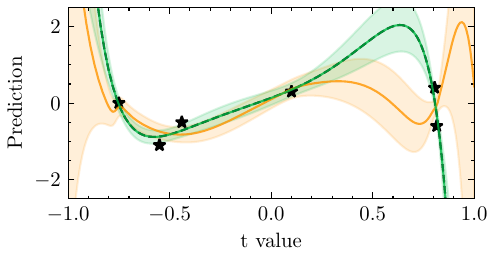}
    \caption{$M=20$}
    \label{fig:linear_regression_with_luckiness:m_20}
\end{subfigure}
\begin{subfigure}{0.49\textwidth}
\centering
    \includegraphics[width=\textwidth]{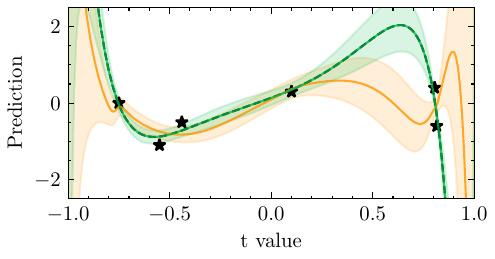}
    \caption{$M=100$}
    \label{fig:linear_regression_with_luckiness:m_100}
\end{subfigure}
\caption{LpNML and Bayesian learners for different model degrees}
\label{fig:linear_regression_with_luckiness:more_m}
\end{figure}

\Figref{fig:linear_regression_with_luckiness:more_m} shows additional model degrees $M$:  \Figref{fig:linear_regression_with_luckiness:m_2} and \figref{fig:linear_regression_with_luckiness:m_5} show the prediction for the under-parameterized case where $M < N$: The LpNML and Bayesian have a similar prediction and both have small variance.
For $M=20$ and $M=100$ as demonstrated in \figref{fig:linear_regression_with_luckiness:m_20} and \figref{fig:linear_regression_with_luckiness:m_100} respectively, in the absence of training data the LpNML prediction is closer to 0 than the Bayesian prediction.

\begin{table}[tb]
\centering
\caption{Distribution shift dataset characteristics}
\label{tab:linear_regression_with_luckiness:dist_shift_metadata}
\small
\begin{tabular}{cccc}
\toprule
Dataset & M & Training set size & Test set size \\
\toprule
Wine        &   8   & 69    & 31            \\
Parkinson   & 10    & 320   & 197          \\
Fire        & 17    & 1877  & 3998         \\
Fertility   & 11    & 4898  & 1599         \\
Triazines   & 60    & 139   & 47          \\
\bottomrule
\end{tabular}
\end{table}

\subsection{Leave-one-out real data performance} \label{sec:linear_regression_with_luckiness:exp_real_sets}
We evaluated the LpNML for 50 real-world datasets from the PMLB repository~\cite{Olson2017PMLB}.

\paragraph{Train-test split.}
The prediction task for the Wine dataset is to predict acidity levels: The training data comprised only red wines with a test set contains only white wines.
In the Parkinsons dataset, the task is to predict a jitter index. This set was split into train and test based on the age feature of patients: Age less than 60 for the train set and greater than 60 for the test set.
For the Fertility dataset, the task is to predict the fertility of a sample. The train contains subjects who are younger than 36 and the test set contains subjects older than 36.
Finally, for the Fires dataset, where the task is to predict the burned area of forest fires that occurred in Portugal during a roughly year-long period, the split was done into train/test based on the month feature of the fire: Those occurring before September for a train set and those after September for the test set.
The Triazines dataset does not include a distribution shift and was randomly split.
In~\Tableref{tab:linear_regression_with_luckiness:dist_shift_metadata}, we include further information on these datasets

\paragraph{$\lambda$ tuning.} To tune the ridge parameter $\lambda$ and the variance $\sigma^2$, we executed the leave-one-out procedure: We constructed with $N$ samples $N$ sets, each set was divided into $N-1$ training samples and a single validation sample for which we optimized $\lambda$ and $\sigma^2$ of Ridge ERM, Bayesian, and LpNML learners.
The average of the $N$ values of $\lambda$ and $\sigma^2$ were used to predict the test labels.
We repeated this procedure for different train-test splits to compute the 95\% confidence intervals.
The MSE reduction is measured in percentage
\begin{equation}
100 \times \left( 1 - \frac{MSE_\textit{LpNML}(x,y)}{MSE_\textit{Ridge-ERM}(x,y)}\right)
\end{equation}
and the log-loss reduction is measured with subtraction
\begin{equation}
\ell(q_\textit{\tiny{Bayesian}},x,y) - \ell(q_\textit{\tiny{LpNML}},x,y).
\end{equation} 

\paragraph{Results.} 
\Tableref{tab:linear_regression_with_luckiness:real_data} shows the test MSE and test log-loss. The LpNML outperforms the Ridge ERM for 48 of 50 sets.
The mean and median MSE reductions are 2.03\%  and  0.96\%, respectively.
The largest MSE reduction is 20.0\% for the 1199\_BNG\_echoMonths set.
For the log-loss, the LpNML outperforms the other learners in 37 out of 50 sets by a mean value of 2.17 and median value of 0.13, while degrading the log-loss of only 3 sets.

Overall, the LpNML has a smaller regularization term: For the 1199\_BNG\_echoMonths dataset, the LpNML has 1.45 lower $\lambda$. This may explain the better performance: For the interpolation region, the small $\lambda$ of the LpNML means a better fit to the test sample.
For the extrapolation region: Although the LpMNL has a smaller $\lambda$, the LpNML prediction is shifted toward 0, and therefore the LpNML test MSE is smaller than the Ridg MSE.

\begin{figure}[bth]
\centering
\begin{subfigure}{1.0\textwidth}
\includegraphics[width=\textwidth]{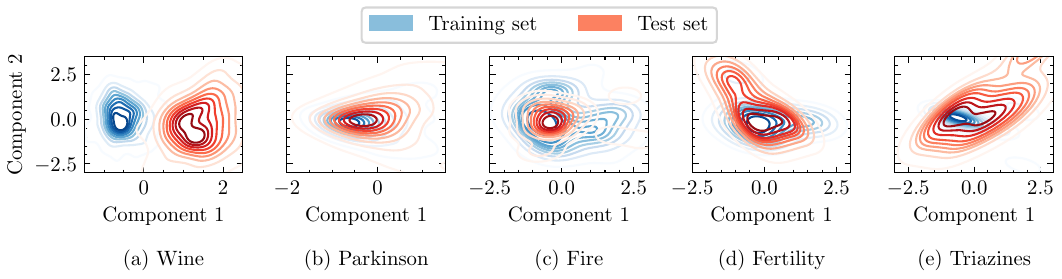}
\caption{PCA components 1 and 2}
\label{fig:linear_regression_with_luckiness:dist_shift}
\end{subfigure}
\begin{subfigure}{1.0\textwidth}
\centering
\includegraphics[width=\textwidth]{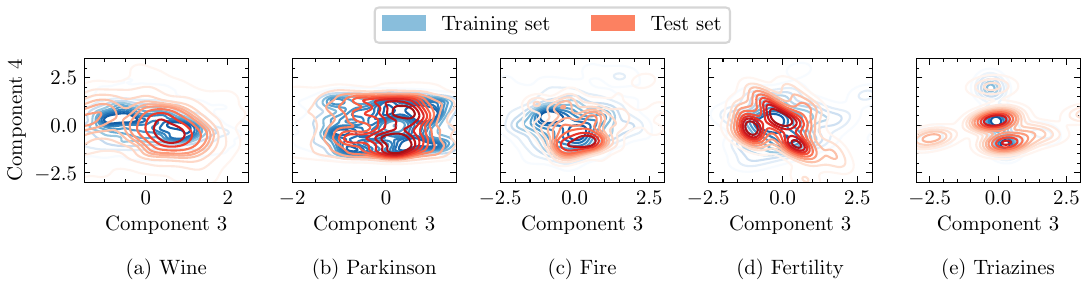}
\caption{PCA components 3 and 4}
\label{fig:linear_regression_with_luckiness:pca_3_4}
\end{subfigure}
\caption{PCA for the distribution-shift benchmark}
\end{figure}

\begin{table}[tb]
\centering
\caption{Distribution-shift benchmark} 
\label{table:dist_shift}
\resizebox{\textwidth}{!}{%
\begin{tabular}{lccccc}
\toprule
TriazinesMethod                    & Wine         & Parkinson                & Fire             & Fertility                & Triazines \\ 
\toprule
OLS                                 & 1.012$\pm$0.016       & 12.792$\pm$0.149                  & 82.715$\pm$35.514         & 0.399$\pm$0.066                   & 0.172$\pm$0.037   \\
Ridge ERM                           & 0.994$\pm$0.015       & 12.527$\pm$0.145                  & 82.346$\pm$35.595         & 0.399$\pm$0.066                   & 0.147$\pm$0.028  \\
\citet{chapelle2000transductive}    & 0.841$\pm$0.001       & 12.253$\pm$0.002                  & 82.066$\pm$ 2.567          & 0.409$\pm$0.013                  & 0.173$\pm$0.001   \\
\citet{cortes2007transductive}      & 0.834$\pm$0.015       & 12.333$\pm$0.145                  & 81.947$\pm$35.834         & \textbf{0.385$\pm$0.076}          & 0.151$\pm$0.024  \\ 
\citet{alquier2012transductive}     & 0.981$\pm$0.015       & 12.253$\pm$0.136                  & 82.066$\pm$36.032         & 0.409$\pm$0.072                   & 0.148$\pm$0.024  \\
\citet{tripuraneni2020single}       & 0.770$\pm$0.014       & 12.089$\pm$0.137                  & 81.979$\pm$35.787         & 0.398$\pm$0.065                   & 0.151$\pm$0.024   \\
\citet{dwivedi2021revisiting}       & 0.929$\pm$0.015       & 12.693$\pm$0.147                  & 82.634$\pm$35.533         & 0.407$\pm$0.071                   & 0.166$\pm$0.021  \\
LpNML (ours)                        & \textbf{0.732$\pm$0.014}       & \textbf{12.027$\pm$0.142}    & \textbf{81.918$\pm$35.923}        & 0.398$\pm$0.067       & \textbf{0.147$\pm$0.024}   \\ 
\bottomrule
\end{tabular}
}
\end{table}

\subsection{Distribution-shift benchmark} \label{sec:linear_regression_with_luckiness:dist_shift}
We followed the benchmark that was proposed by~\citet{tripuraneni2020single}: Four datasets from the UCI repository~\cite{Dua:2019} were chosen and split such that the test data contain a distribution shift from the training. 
The fifth dataset (Triazines) does not include a distribution shift. 
The train-test split was performed randomly. A detailed explanation of the train-test split is provided in the appendix. This benchmark was evaluated using the RMSE metric and the hyperparameters were optimized using the leave-one-out procedure.

The LpNML attains a smaller RMSE for four sets. The largest improvement is for the Wine set for which the LpNML reduces the RMSE of~\citet{tripuraneni2020single} by 4.93\%.
For the Fertility set, the LpNML is the second-best following the method of~\citet{cortes2007transductive}, which has 3.27\% smaller RMSE.
For the Triazines set that does not contain a distribution shift, the LpNML performs the same as the Ridge ERM and outperforms the other methods.

\Figref{fig:linear_regression_with_luckiness:dist_shift} shows the principal component analysis (PCA) with 2 components of the benchmark sets and the third and forth components are shown in  \figref{fig:linear_regression_with_luckiness:pca_3_4}.
For the Wine set in \figref{fig:linear_regression_with_luckiness:dist_shift}a, the difference between the train and test data is the most visually seen. For this set, the LpNML has the largest RMSE reduction over the Ridge ERM: an RMSE reduction of 26.36\%.
For the Parkinson, Fire, Fertility, and Triazines datasets, the LpNML reduces the Ridge ERM RMSE by 3.99\%, 0.52\%, 0.25\%, 0.0\% respectively, which is correlated to the degree to which the train-test splits are visually separated.

\section{Concluding remarks} \label{sec:linear_regression_with_luckiness:conclusion}
In this section, we introduced the LpNML by incorporating a luckiness function to the min-max regret objective.
For ridge regression, where we defined the luckiness function as the Gaussian prior, we have shown that the LpNML prediction is shifted toward 0 for a test vector that resides in the subspace that is associated with the small eigenvalues of the design matrix.
For real-world datasets, the LpNML attains up to 20\% better test MSE than Ridge ERM and for the distribution-shift benchmark, the LpNML reduces the error of recent leading methods by up to 4.93\%.

We believe that our approach can be valuable to fields that use linear regression and require high-precision prediction.
For future work, our LpNML framework can be extended with more luckiness functions such as $\ell_1$ by defining the luckiness function to be the Laplace prior.

\begin{table}[tb]
\centering
\caption{Leave-one-out test performance for PMLB sets}
\resizebox{\textwidth}{!}{%
\begin{tabular}{lc|cc|ccc}
\toprule
               \hspace{0.7cm} Set name &  M & \thead{Ridge ERM \\ MSE} &  \thead{ LpNML \\ MSE} &       \thead{Ridge ERM \\ log-loss} &\thead{Bayesian \\ log-loss} & \thead{LpNML \\ log-loss} \\
\midrule
1199\_BNG\_echoMonths & 9 & 1.76 $\pm$ 0.55 & 1.41 $\pm$ 0.03$\textcolor{green}{\blacktriangledown20.0\%}$ & 173 $\pm$ 276 & 8.86 $\pm$ 9.83 & 2.64 $\pm$ 0.57$\textcolor{green}{\blacktriangledown6.22}$ \\
1089\_USCrime & 13 & 0.93 $\pm$ 0.02 & 0.86 $\pm$ 0.02$\textcolor{green}{\blacktriangledown8.30\%}$ & 1.76 $\pm$ 0.09 & 1.53 $\pm$ 0.04 & 1.40 $\pm$ 0.02$\textcolor{green}{\blacktriangledown0.14}$ \\
294\_satellite\_image & 36 & 0.90 $\pm$ 0.02 & 0.83 $\pm$ 0.02$\textcolor{green}{\blacktriangledown8.25\%}$ & 1.34 $\pm$ 0.01 & 1.34 $\pm$ 0.01 & 1.29 $\pm$ 0.01$\textcolor{green}{\blacktriangledown0.05}$ \\
banana & 2 & 2.84 $\pm$ 0.51 & 2.61 $\pm$ 0.35$\textcolor{green}{\blacktriangledown8.11\%}$ & 418 $\pm$ 101 & 296 $\pm$ 40.5 & 290 $\pm$ 40.5$\textcolor{green}{\blacktriangledown5.51}$ \\
195\_auto\_price & 15 & 1.28 $\pm$ 0.08 & 1.20 $\pm$ 0.08$\textcolor{green}{\blacktriangledown6.42\%}$ & 2.32 $\pm$ 0.28 & 1.71 $\pm$ 0.08 & 1.59 $\pm$ 0.15$\textcolor{green}{\blacktriangledown0.12}$ \\
695\_chatfield\_4 & 12 & 1.20 $\pm$ 0.07 & 1.13 $\pm$ 0.06$\textcolor{green}{\blacktriangledown5.67\%}$ & 1.73 $\pm$ 0.11 & 1.58 $\pm$ 0.08 & 1.44 $\pm$ 0.05$\textcolor{green}{\blacktriangledown0.14}$ \\
503\_wind & 14 & 1.13 $\pm$ 0.03 & 1.07 $\pm$ 0.03$\textcolor{green}{\blacktriangledown4.78\%}$ & 1.55 $\pm$ 0.03 & 1.51 $\pm$ 0.02 & 1.53 $\pm$ 0.12$\textcolor{pink}{\blacktriangle0.02}$ \\
560\_bodyfat & 14 & 0.59 $\pm$ 0.02 & 0.56 $\pm$ 0.02$\textcolor{green}{\blacktriangledown4.07\%}$ & 2.15 $\pm$ 0.35 & 1.32 $\pm$ 0.14 & 1.45 $\pm$ 0.23$\textcolor{pink}{\blacktriangle0.13}$ \\
659\_sleuth\_ex1714 & 7 & 1.69 $\pm$ 0.09 & 1.64 $\pm$ 0.09$\textcolor{green}{\blacktriangledown3.29\%}$ & 6.25 $\pm$ 1.89 & 3.00 $\pm$ 0.34 & 2.44 $\pm$ 0.19$\textcolor{green}{\blacktriangledown0.56}$ \\
344\_mv & 10 & 1.42 $\pm$ 0.09 & 1.38 $\pm$ 0.08$\textcolor{green}{\blacktriangledown2.89\%}$ & 7.71 $\pm$ 1.54 & 5.09 $\pm$ 1.36 & 4.64 $\pm$ 1.42$\textcolor{green}{\blacktriangledown0.45}$ \\
229\_pwLinear & 10 & 1.38 $\pm$ 0.04 & 1.35 $\pm$ 0.04$\textcolor{green}{\blacktriangledown2.25\%}$ & 3.32 $\pm$ 0.38 & 2.36 $\pm$ 0.18 & 2.23 $\pm$ 0.19$\textcolor{green}{\blacktriangledown0.13}$ \\
1027\_ESL & 4 & 1.72 $\pm$ 0.09 & 1.69 $\pm$ 0.08$\textcolor{green}{\blacktriangledown1.92\%}$ & 31.4 $\pm$ 5.82 & 24.7 $\pm$ 4.53 & 23.1 $\pm$ 4.49$\textcolor{green}{\blacktriangledown1.57}$ \\
653\_fri\_c0\_250\_25 & 25 & 1.16 $\pm$ 0.03 & 1.14 $\pm$ 0.03$\textcolor{green}{\blacktriangledown1.81\%}$ & 1.57 $\pm$ 0.05 & 1.53 $\pm$ 0.02 & 1.51 $\pm$ 0.02$\textcolor{green}{\blacktriangledown0.01}$ \\
230\_machine\_cpu & 6 & 13.1 $\pm$ 2.77 & 12.9 $\pm$ 2.72$\textcolor{green}{\blacktriangledown1.76\%}$ & 77.7 $\pm$ 79.8 & 8.31 $\pm$ 1.88 & 7.36 $\pm$ 1.97$\textcolor{green}{\blacktriangledown0.94}$ \\
1203\_BNG\_pwLinear & 10 & 1.41 $\pm$ 0.04 & 1.39 $\pm$ 0.04$\textcolor{green}{\blacktriangledown1.61\%}$ & 3.18 $\pm$ 0.49 & 2.33 $\pm$ 0.22 & 2.21 $\pm$ 0.23$\textcolor{green}{\blacktriangledown0.12}$ \\
561\_cpu & 7 & 13.2 $\pm$ 5.72 & 13.0 $\pm$ 5.69$\textcolor{green}{\blacktriangledown1.39\%}$ & 318 $\pm$ 266 & 12.4 $\pm$ 4.95 & 7.52 $\pm$ 3.06$\textcolor{green}{\blacktriangledown4.89}$ \\
564\_fried & 10 & 1.35 $\pm$ 0.03 & 1.34 $\pm$ 0.03$\textcolor{green}{\blacktriangledown1.33\%}$ & 4.07 $\pm$ 0.61 & 2.60 $\pm$ 0.28 & 2.31 $\pm$ 0.23$\textcolor{green}{\blacktriangledown0.29}$ \\
633\_fri\_c0\_500\_25 & 25 & 1.13 $\pm$ 0.02 & 1.12 $\pm$ 0.02$\textcolor{green}{\blacktriangledown1.24\%}$ & 1.54 $\pm$ 0.02 & 1.51 $\pm$ 0.01 & 1.51 $\pm$ 0.02\hphantom{00000} \\
598\_fri\_c0\_1000\_25 & 25 & 1.14 $\pm$ 0.02 & 1.12 $\pm$ 0.02$\textcolor{green}{\blacktriangledown1.20\%}$ & 1.57 $\pm$ 0.03 & 1.52 $\pm$ 0.02 & 1.51 $\pm$ 0.02$\textcolor{green}{\blacktriangledown0.01}$ \\
706\_sleuth\_case1202 & 6 & 2.50 $\pm$ 0.16 & 2.47 $\pm$ 0.16$\textcolor{green}{\blacktriangledown1.19\%}$ & 26.7 $\pm$ 9.12 & 11.1 $\pm$ 3.64 & 9.52 $\pm$ 3.61$\textcolor{green}{\blacktriangledown1.57}$ \\
651\_fri\_c0\_100\_25 & 25 & 1.10 $\pm$ 0.02 & 1.08 $\pm$ 0.02$\textcolor{green}{\blacktriangledown1.18\%}$ & 1.58 $\pm$ 0.04 & 1.52 $\pm$ 0.02 & 1.50 $\pm$ 0.02$\textcolor{green}{\blacktriangledown0.02}$ \\
635\_fri\_c0\_250\_10 & 10 & 1.41 $\pm$ 0.04 & 1.39 $\pm$ 0.04$\textcolor{green}{\blacktriangledown1.11\%}$ & 3.66 $\pm$ 0.86 & 2.42 $\pm$ 0.28 & 2.17 $\pm$ 0.20$\textcolor{green}{\blacktriangledown0.26}$ \\
656\_fri\_c1\_100\_5 & 5 & 2.59 $\pm$ 0.10 & 2.56 $\pm$ 0.10$\textcolor{green}{\blacktriangledown1.09\%}$ & 31.8 $\pm$ 7.88 & 19.1 $\pm$ 5.60 & 15.2 $\pm$ 4.72$\textcolor{green}{\blacktriangledown3.94}$ \\
1096\_FacultySalaries & 4 & 3.17 $\pm$ 0.35 & 3.14 $\pm$ 0.31$\textcolor{green}{\blacktriangledown1.06\%}$ & 107 $\pm$ 72.8 & 48.6 $\pm$ 20.9 & 13.6 $\pm$ 4.23$\textcolor{green}{\blacktriangledown35.0}$ \\
595\_fri\_c0\_1000\_10 & 10 & 1.34 $\pm$ 0.03 & 1.33 $\pm$ 0.03$\textcolor{green}{\blacktriangledown1.01\%}$ & 3.38 $\pm$ 0.42 & 2.41 $\pm$ 0.21 & 2.13 $\pm$ 0.16$\textcolor{green}{\blacktriangledown0.28}$ \\
1193\_BNG\_lowbwt & 9 & 1.48 $\pm$ 0.04 & 1.46 $\pm$ 0.04$\textcolor{green}{\blacktriangledown0.91\%}$ & 4.12 $\pm$ 0.94 & 3.09 $\pm$ 0.78 & 3.04 $\pm$ 0.80$\textcolor{green}{\blacktriangledown0.05}$ \\
650\_fri\_c0\_500\_50 & 50 & 1.06 $\pm$ 0.02 & 1.05 $\pm$ 0.02$\textcolor{green}{\blacktriangledown0.88\%}$ & 1.46 $\pm$ 0.01 & 1.45 $\pm$ 0.01 & 1.45 $\pm$ 0.01\hphantom{00000} \\
666\_rmftsa\_ladata & 10 & 2.73 $\pm$ 0.24 & 2.71 $\pm$ 0.24$\textcolor{green}{\blacktriangledown0.87\%}$ & 4.63 $\pm$ 2.25 & 2.87 $\pm$ 0.75 & 2.48 $\pm$ 0.32$\textcolor{green}{\blacktriangledown0.39}$ \\
1028\_SWD & 10 & 1.40 $\pm$ 0.03 & 1.38 $\pm$ 0.03$\textcolor{green}{\blacktriangledown0.86\%}$ & 5.56 $\pm$ 2.52 & 3.85 $\pm$ 1.42 & 2.89 $\pm$ 0.66$\textcolor{green}{\blacktriangledown0.96}$ \\
192\_vineyard & 2 & 11.5 $\pm$ 1.56 & 11.4 $\pm$ 1.45$\textcolor{green}{\blacktriangledown0.85\%}$ & 722 $\pm$ 212 & 698 $\pm$ 212 & 669 $\pm$ 211$\textcolor{green}{\blacktriangledown29.2}$ \\
225\_puma8NH & 8 & 1.48 $\pm$ 0.04 & 1.47 $\pm$ 0.03$\textcolor{green}{\blacktriangledown0.82\%}$ & 5.32 $\pm$ 0.77 & 3.17 $\pm$ 0.33 & 2.82 $\pm$ 0.39$\textcolor{green}{\blacktriangledown0.34}$ \\
634\_fri\_c2\_100\_10 & 10 & 1.47 $\pm$ 0.04 & 1.46 $\pm$ 0.04$\textcolor{green}{\blacktriangledown0.78\%}$ & 3.68 $\pm$ 1.54 & 2.39 $\pm$ 0.35 & 1.93 $\pm$ 0.10$\textcolor{green}{\blacktriangledown0.46}$ \\
657\_fri\_c2\_250\_10 & 10 & 1.47 $\pm$ 0.05 & 1.46 $\pm$ 0.05$\textcolor{green}{\blacktriangledown0.67\%}$ & 4.59 $\pm$ 2.42 & 2.25 $\pm$ 0.21 & 2.20 $\pm$ 0.25$\textcolor{green}{\blacktriangledown0.05}$ \\
603\_fri\_c0\_250\_50 & 50 & 1.06 $\pm$ 0.02 & 1.06 $\pm$ 0.02$\textcolor{green}{\blacktriangledown0.49\%}$ & 1.45 $\pm$ 0.01 & 1.45 $\pm$ 0.01 & 1.45 $\pm$ 0.01\hphantom{00000} \\
624\_fri\_c0\_100\_5 & 5 & 2.11 $\pm$ 0.11 & 2.10 $\pm$ 0.10$\textcolor{green}{\blacktriangledown0.39\%}$ & 31.4 $\pm$ 11.8 & 16.7 $\pm$ 10.5 & 8.71 $\pm$ 3.49$\textcolor{green}{\blacktriangledown8.03}$ \\
606\_fri\_c2\_1000\_10 & 10 & 1.52 $\pm$ 0.05 & 1.52 $\pm$ 0.05$\textcolor{green}{\blacktriangledown0.35\%}$ & 4.53 $\pm$ 2.00 & 3.95 $\pm$ 1.95 & 3.29 $\pm$ 1.88$\textcolor{green}{\blacktriangledown0.66}$ \\
579\_fri\_c0\_250\_5 & 5 & 2.00 $\pm$ 0.07 & 2.00 $\pm$ 0.07$\textcolor{green}{\blacktriangledown0.22\%}$ & 22.0 $\pm$ 4.73 & 11.7 $\pm$ 3.53 & 6.85 $\pm$ 1.61$\textcolor{green}{\blacktriangledown4.83}$ \\
648\_fri\_c1\_250\_50 & 50 & 1.08 $\pm$ 0.02 & 1.08 $\pm$ 0.02$\textcolor{green}{\blacktriangledown0.20\%}$ & 1.47 $\pm$ 0.02 & 1.46 $\pm$ 0.01 & 1.46 $\pm$ 0.01\hphantom{00000} \\
1191\_BNG\_pbc & 18 & 1.23 $\pm$ 0.03 & 1.23 $\pm$ 0.03$\textcolor{green}{\blacktriangledown0.19\%}$ & 2.10 $\pm$ 0.53 & 2.02 $\pm$ 0.52 & 1.86 $\pm$ 0.46$\textcolor{green}{\blacktriangledown0.16}$ \\
618\_fri\_c3\_1000\_50 & 50 & 1.10 $\pm$ 0.03 & 1.09 $\pm$ 0.03$\textcolor{green}{\blacktriangledown0.17\%}$ & 1.47 $\pm$ 0.01 & 1.47 $\pm$ 0.01 & 1.47 $\pm$ 0.01\hphantom{00000} \\
631\_fri\_c1\_500\_5 & 5 & 2.38 $\pm$ 0.09 & 2.38 $\pm$ 0.09$\textcolor{green}{\blacktriangledown0.14\%}$ & 38.2 $\pm$ 10.0 & 24.0 $\pm$ 6.32 & 24.7 $\pm$ 8.23$\textcolor{pink}{\blacktriangle0.67}$ \\
583\_fri\_c1\_1000\_50 & 50 & 1.06 $\pm$ 0.02 & 1.06 $\pm$ 0.02$\textcolor{green}{\blacktriangledown0.14\%}$ & 1.45 $\pm$ 0.01 & 1.45 $\pm$ 0.01 & 1.45 $\pm$ 0.01\hphantom{00000} \\
586\_fri\_c3\_1000\_25 & 25 & 1.18 $\pm$ 0.03 & 1.18 $\pm$ 0.03$\textcolor{green}{\blacktriangledown0.13\%}$ & 1.69 $\pm$ 0.23 & 1.64 $\pm$ 0.19 & 1.52 $\pm$ 0.02$\textcolor{green}{\blacktriangledown0.12}$ \\
542\_pollution & 15 & 1.32 $\pm$ 0.05 & 1.32 $\pm$ 0.05$\textcolor{green}{\blacktriangledown0.09\%}$ & 1.85 $\pm$ 0.15 & 1.60 $\pm$ 0.02 & 1.57 $\pm$ 0.02$\textcolor{green}{\blacktriangledown0.04}$ \\
687\_sleuth\_ex1605 & 5 & 2.17 $\pm$ 0.09 & 2.17 $\pm$ 0.09$\textcolor{green}{\blacktriangledown0.08\%}$ & 15.1 $\pm$ 3.06 & 7.66 $\pm$ 1.52 & 6.06 $\pm$ 1.31$\textcolor{green}{\blacktriangledown1.60}$ \\
645\_fri\_c3\_500\_50 & 50 & 1.09 $\pm$ 0.03 & 1.09 $\pm$ 0.03$\textcolor{green}{\blacktriangledown0.05\%}$ & 1.47 $\pm$ 0.02 & 1.46 $\pm$ 0.02 & 1.46 $\pm$ 0.02\hphantom{00000} \\
623\_fri\_c4\_1000\_10 & 10 & 1.55 $\pm$ 0.05 & 1.54 $\pm$ 0.05$\textcolor{green}{\blacktriangledown0.04\%}$ & 2.64 $\pm$ 0.52 & 2.01 $\pm$ 0.12 & 1.99 $\pm$ 0.13$\textcolor{green}{\blacktriangledown0.02}$ \\
622\_fri\_c2\_1000\_50 & 50 & 1.08 $\pm$ 0.02 & 1.08 $\pm$ 0.02$\textcolor{green}{\blacktriangledown0.04\%}$ & 1.46 $\pm$ 0.01 & 1.46 $\pm$ 0.01 & 1.46 $\pm$ 0.01\hphantom{00000} \\
658\_fri\_c3\_250\_25 & 25 & 1.18 $\pm$ 0.03 & 1.18 $\pm$ 0.03\hphantom{0000000} & 1.54 $\pm$ 0.02 & 1.52 $\pm$ 0.02 & 1.52 $\pm$ 0.02\hphantom{00000} \\
201\_pol & 48 & 1.07 $\pm$ 0.02 & 1.07 $\pm$ 0.02\hphantom{0000000} & 1.45 $\pm$ 0.01 & 1.45 $\pm$ 0.01 & 1.45 $\pm$ 0.01\hphantom{00000} \\
\bottomrule
\end{tabular}

}
\label{tab:linear_regression_with_luckiness:real_data}
\end{table}

\chapter{Confidence Estimation for Neural Networks} \label{chap:neural_networks}
\section{Introduction} \label{sec:neural_networks:introduction}

An important concern that limits the adoption of DNN in critical safety systems is how to assess our confidence in their predictions, i.e, quantifying their \textit{generalization capability}~\citep{DBLP:conf/bmvc/KaufmanBBCH19,willers2020safety}.
Take, for instance, a machine learning model for medical diagnosis~\citep{bibas2021learning}. 
It may produce (wrong) diagnoses in the presence of test inputs that are different from the training set rather than flagging them for human intervention~\citep{singh2021uncertainty}. 
Detecting such unexpected inputs had been formulated as the out-of-distribution (OOD) detection task~\citep{hendrycks17baseline}, as flagging test inputs that lie outside the training classes, i.e., are not \textit{in-distribution} (IND).

Previous learning methods that designed to offer such generalization measures, include VC-dimension~\citep{vapnik2015uniform,zhong2017recovery} and norm based bounds~\citep{DBLP:conf/nips/BartlettFT17,DBLP:conf/iclr/NeyshaburBS18}.
As a whole, these methods characterized the generalization ability based on the properties of the parameters. However, they do not consider the test sample that is presented to the model~\citep{DBLP:conf/iclr/JiangNMKB20}, which makes them useless for OOD detection.
Other approaches build heuristics over the ERM learner, by post-processing the model output~\citep{gram} or modifying the training process~\citep{PAPADOPOULOS2021138,vyas2018out}. Regardless of the approach, these methods choose the learner that minimizes the loss over the \textit{training set}. This may lead to a large generalization error because the ERM estimate may be wrong on unexpected inputs; especially with large models such as DNN~\citep{belkin2019reconciling}.

To produce a useful generalization measure, we exploit the individual setting framework~\citep{merhav1998universal} along with the pNML learner: We derive an analytical solution of the pNML learner and its generalization error (the regret) for a single layer Neural Network (NN).
We analyze the derived regret and show it obtains low values when the test input either (i) lies in a subspace spanned by the eigenvectors associated with the large eigenvalues of the training data empirical correlation matrix or (ii) is located far from the decision boundary.
Crucially, although our analysis focuses on a single layer NN, our results are applicable to the last layer of DNNs without changing the network architecture or the training process:
We treat the pretrained DNN as a feature extractor with the last layer as a single layer NN classifier.
We can therefore show the usage of the pNML regret as a confidence score for the OOD detection task.

In addition, we explore an alternative method of computing the pNML learning for the DNN hypothesis class by fully training the last layer of DNN and show how it benefits the \textit{open-set classification} task, where the task is to identify input samples as known or unknown and simultaneously correctly classifying all the known classes is referred to as the open-set recognition task~\cite{scheirer2012toward}. 

To summarize, we make the following contributions.
\begin{enumerate}
    \item We derive an analytical expression of the pNML regret, which is associated with the generalization error, for a single layer NN. 
    \item
     We explore the pNML regret characteristics as a function of the test sample data, training data, and the corresponding ERM prediction. We provide a visualization on low dimensional data and demonstrate the situations in which the pNML regret is low and the prediction can be trusted.
    \item We propose an adaptation of the derived pNML regret to {\em any} pretrained DNN that uses the softmax function with neither additional parameters nor extra data.
\end{enumerate}

Applying the derived pNML regret to a pretrained DNN does not require additional data, it is efficient and can be easily implemented.
The derived regret is theoretically justified for OOD detection since it is the individual setting solution for which we do not require any knowledge on the test input distribution.
Our evaluation includes 74 IND-OOD detection benchmarks using DenseNet-BC-100~\cite{huang2017densely}, ResNet-34~\cite{he2016deep}, and WideResNet-40~\cite{zagoruyko2016wide} trained with CIFAR-100~\cite{krizhevsky2014cifar}, CIFAR-10, SVHN~\cite{netzer2011reading}, and ImageNet-30~\cite{DBLP:conf/nips/HendrycksMKS19}.
Our approach outperforms leading methods in nearly all 74 OOD detection benchmarks up to $+15.2$\%

\section{Analytical solution for a single layer neural-network}
In order to derive the pNML for a single layer NN, we first present a framework of online update of a neural network:
Let $X_N$ and $Y_N$ be the data and label matrices of $N$ training points respectively
\begin{equation} \label{eq:single_layer_nn:training set_matrix}
X_N = 
\begin{bmatrix}
x_1 & x_2 & \hdots & x_N
\end{bmatrix}^\top
\in \mathcal{R}^{N \times M}, \quad
Y_N = 
\begin{bmatrix}
y_1 & y_2 & \dots & y_N
\end{bmatrix}^\top
\in \mathcal{R}^{N \times C},
\end{equation} 
such that the number of input features and model outputs are $M$ and $C$ respectively.
Denote $X_N^+$ as the Moor-Penrose inverse of the data matrix
\begin{equation} \label{eq:single_layer_nn:pseudo-inverse}
X_N^+ = 
\begin{cases} 
(X_N^\top X_N)^{-1} X_N^\top & \textit{Rank}(X_N^\top X_N) = M\\ 
X_N^\top (X_N X_N^\top )^{-1} & \textit{otherwise},
\end{cases}
\end{equation}
$f(\cdot)$ and $f^{-1}(\cdot)$ as the activation and inverse activation functions, and $\theta \in \mathcal{R}^{M \times C}$ as the learnable parameters.

The ERM that minimizes the training set MSE is given by  
\begin{equation} \label{eq:single_layer_nn:nn_solution}
\hat{\theta}_N = \argmin_{\theta} \norm{Y_N - f(X_N \theta)}_F^2 =  X_N^{+} f^{-1}(Y_N)  
\end{equation}  
where \(\norm{\cdot}_F\) denotes the Frobenius norm, as \(Y_N - f(X_N \theta)\) represents an \(N \times C\) matrix.

Recently, \citet{zhuang2020training} proposed a recursive formulation for updating the weights of a DNN in an online manner. Using their scheme, only one training sample is processed at a time, with updates made iteratively.  
Denote the projection of a sample \(x\) onto the orthogonal subspace of the training set correlation matrix as  
\begin{equation}
x_\bot = \left(I - X_N^+ X_N \right) x,
\end{equation}  
the update rule for receiving a new training sample with data \(x\) and label \(y\) is given by  
\begin{equation} \label{eq:single_layer_nn:update_nn}
\thetagenie = \hat{\theta}_N + g \left(f^{-1}(y) - x^\top \hat{\theta}_N \right), \quad
g \triangleq
\begin{cases}
\frac{1}{\norm{x_\bot}^2} x_\bot &  x_\bot  \neq 0 \\
\frac{1}{1 + x^\top X_N^+ X_N^{+ \top} x} X_N^+ X_N^{+ \top} x &  x_\bot = 0
\end{cases}.
\end{equation}  
Here, \(\hat{\theta}_N\) represents the (ERM solution based on \(N\) training samples. It is important to note that this algorithm does not compute the exact solution for the updated dataset with the new training sample. Instead, it provides an iterative step towards the optimal solution.  
\citet{zhuang2020training} applied this formulation to train a DNN in a layer-by-layer fashion.

\subsubsection{The pNML for a single layer neural-network}
\label{sec:single_layer_nn:pnml_regret}

Intuitively, the pNML as stated in~\eqref{eq:pnml} can be described as follows: To assign a probability for a potential outcome, (i) add it to the training set with an arbitrary label, (ii) find the best-suited model, and (iii) take the probability it gives to the assumed label. 
Follow this procedure for every label and normalize to get a valid probability assignment. Use the log normalization factor as the confidence measure.
This method can be extended to any general learning procedure that generates a prediction based on a training set. 
One such method is a single layer NN.

A single layer NN maps an input $x\in \mathcal{R}^{M \times 1}$ using the softmax function to a probability vector which represents the probability assignment to one of $C$ classes
\begin{equation}
p_\theta(y_c|x) 
= f(x^\top \theta)_c
= \frac{e^{\theta_c^\top x}}{\sum_{c'=1}^C e^{\theta_{c'}^\top x}},
\quad c \in \{1,\dots, C\}.
\end{equation}
To align with the recursive formulation of~\eqref{eq:single_layer_nn:update_nn}, the label $y$ is a one-hot row vector with $C$ elements, $y_c$ is the $c$ element of $y$, and the learnable parameters $\left\{\theta_{c'}\right\}_{c'=1}^C$ are the columns of the parameter matrix of~\eqref{eq:single_layer_nn:nn_solution}. In addition, the inverse of the softmax activation is
\begin{equation}
z \triangleq f^{-1}\left(p_\theta \left(i|x\right) \right) = \ln p_\theta \left(i|x\right) + \ln \sum_{j=1}^C e^{\theta_j^\top x}.
\end{equation}

To compute the genie prediction of the test label we add the test sample to the training set. Then we optimize the learnable parameters to minimize the loss of this new dataset.  
\begin{lemma} \label{lemma:genie} 
Let $c$ be the true test label, $p_c$ the probability assignment of the ERM model of the label $c$, $g$ as defined in \eqref{eq:single_layer_nn:update_nn}.
Given test data $x$ with a one-hot row vector $y$, the genie prediction is 
\begin{equation}
p_{\thetagenie} (c|x) 
=
\frac{p_c}
{p_c + p_c^{x^\top g}\left(1 - p_c\right)},
\end{equation}
\end{lemma}
\begin{proof}
With~\eqref{eq:single_layer_nn:update_nn}, the probability assignment of the genie can be written as follows.
\begin{equation} \label{eq:single_layer_nn:lemma_genie}
p_{\thetagenie}(c|x) 
 =
\frac{e^{\thetagenie^\top x}}{\sum_{\substack{j=1 \\ j \neq c}}^C e^{\theta_j^\top x} + e^{\thetagenie^\top x} }
 =
\frac
{e^{x^\top \left[ \theta_{c} + g \left(z - \theta_c^\top x \right)\right]}}
{\sum_{j=1}^C e^{\theta_j^\top x} - e^{\theta_c^\top x} + e^{x^\top \left[ \theta_c + g \left(z - \theta_c^\top x \right)\right]}}.
\end{equation}
The genie knows the true test label $c$ thus the inverse activation function can be written as 
$z = \ln \sum_{j=1}^C e^{\theta_j^\top x}$.
The simplified numerator is
\begin{equation}
e^{\theta_{c}^\top x}e^{x^\top g \left(z - \theta_c^\top x \right)}
=
e^{\theta_{c}^\top x}\left[S e^{-\theta_c^\top x}\right]^{x^\top g }
=
\left( \sum_{j=1}^C e^{\theta_j^\top x} \right) p_c^{-x^\top g} p_c.
\end{equation}
Substituting to \eqref{eq:single_layer_nn:lemma_genie} and dividing the numerator and denominator by $\left( \sum_{j=1}^C e^{\theta_j^\top x} \right)$ provides the result.
\end{proof}

The true test label is not available to a legit learner. Therefore in the pNML process every possible label is taken into account. The pNML regret is the logarithm of the sum of models' prediction, each one trained with a different test label value.

\begin{theorem}
Denote $p_i$ as the ERM prediction of label $i$, the pNML regret of a single layer NN is
\begin{equation} \label{eq:single_layer_nn:pnml}
\Gamma =\log \sum_{i=1}^{C} \frac{p_i} {p_i + p_i^{x^\top g}\left(1 - p_i\right)}.
\end{equation}
\end{theorem}
\begin{proof}
The normalization factor is the sum of the probabilities assignment of models that were trained with a specific value of the test sample
$K = \sum_{i=1}^{C} p_{\thetageniein{e_i}}(i|x)$.
As shown in~\eqref{eq:pnml}, the log normalization factor is the pNML regret. 
With \lemmaref{lemma:genie}, we get the explicit expression.
\end{proof}

The pNML probability assignment of label $i \in \{1,\dots,C\}$ is the probability assignment of a model that was trained with that label divided by the normalization factor  $q_\textit{pNML}(i|x) = \frac{1}{K} \probthetagenie$.

Let $u_m$ and $h_m$ be the $m$-th eigenvector and eigenvalue of the training set data matrix $X_N$ such that for $x_\bot =0$, the quantity $x^\top g$ is
\begin{equation}
x^\top g = \frac{x^\top X_N^+ X_N^{+ \top} x}{1 + x^\top X_N^+ X_N^{+ \top} x} 
=
\frac{\frac{1}{N} \sum_{m=1}^{M} \frac{1}{h_m^2} \left(x^\top u_m\right)^2 }{1 + \frac{1}{N} \sum_{i=1}^{M} \frac{1}{h_m^2} \left(x^\top u_m\right)^2}.
\end{equation}
We make the following remarks.
\begin{enumerate}
\item 
If the test sample $x$ lies in the subspace spanned by the eigenvectors with large eigenvalues, $x^\top g$ is small and the corresponding regret is low
$\lim_{x^\top g \xrightarrow{} 0}\Gamma = \log \sum_{i=1}^C p_i =  0$.
In this case, the pNML prediction is similar to the genie and can be trusted.
\item
Test input that resides is in the subspace that corresponds to the small eigenvalues produces $x^\top g=1$ and a large regret is obtained 
$\lim_{x^\top g \xrightarrow{} 1} \Gamma =\log \sum_{i=1}^C \frac{1}{2 - p_i^2}$.
The prediction for this test sample cannot be trusted. In~\secref{sec:single_layer_nn:experiments} we show that in this situation the test sample can be classified as an OOD sample.
\item
As the training set size ($N$) increases $x^\top g$ becomes smaller and the regret decreases.
\item If the test sample is far from the decision boundary, the ERM assigns to one of the labels probability 1. In this case, the regret is 0 no matter in which subspace the test vector lies.
\end{enumerate}

\subsubsection{pNML min-max regret simulation}
\label{sec:single_layer_nn:regret_simulation}

We simulate the response of the pNML regret for two classes (C=2) and divide it by $\log C$ to have the regret bounded between 0 and 1.
\Figref{fig:regret_simulation} shows the regret behaviour for different $p_1$ (the ERM probability assignment of class 1) as a function of $x^\top g$.

For an ERM model that is certain on the prediction ($p_1=0.99$ that is represented by the purple curve), a slight variation of $x^\top g$ causes a large response of the regret comparing to $p_1$ that equals 0.55 and 0.85.

\begin{figure}[tb]
    \centering
    \includegraphics[width=0.65\linewidth]{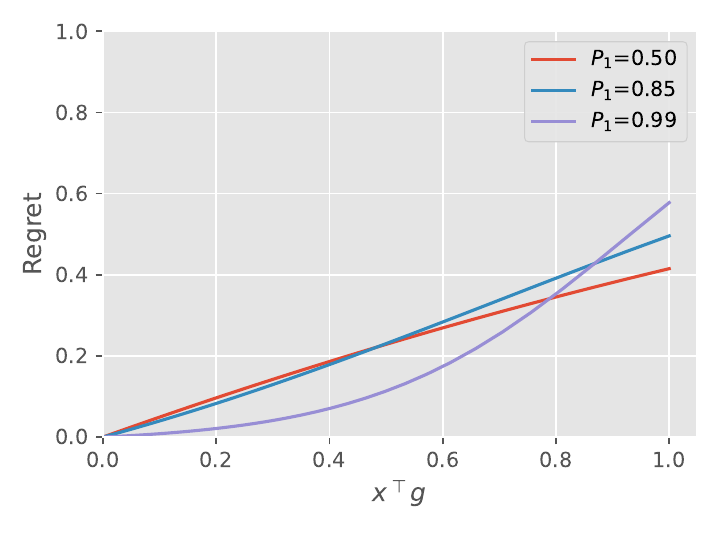}
    \caption{The pNML regret simulation for a two class predictor}
    \label{fig:regret_simulation}
\end{figure}

\begin{figure}[tb]
\centering
\begin{subfigure}[t]{0.49\textwidth}
\includegraphics[width=1.0\textwidth]{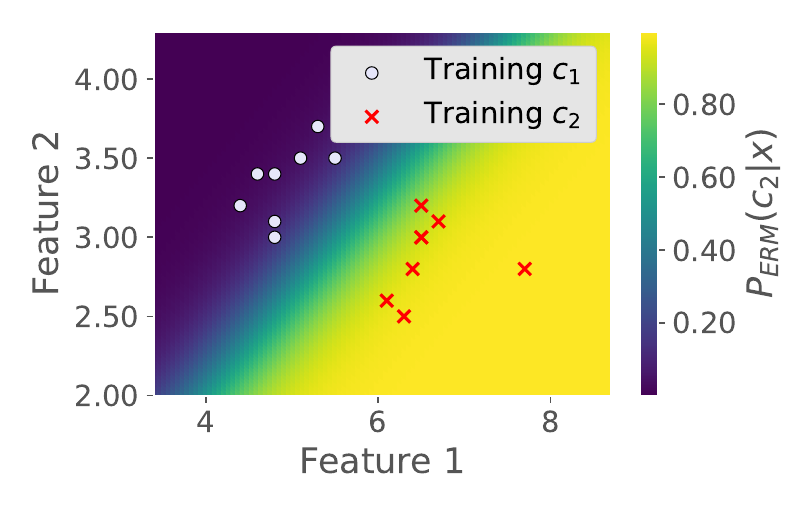}
\caption{ERM for a separable split \label{fig:single_layer_nn:syntetic_erm_prob}}
\end{subfigure}
\begin{subfigure}[t]{0.49\textwidth}
\includegraphics[width=1.0\textwidth]{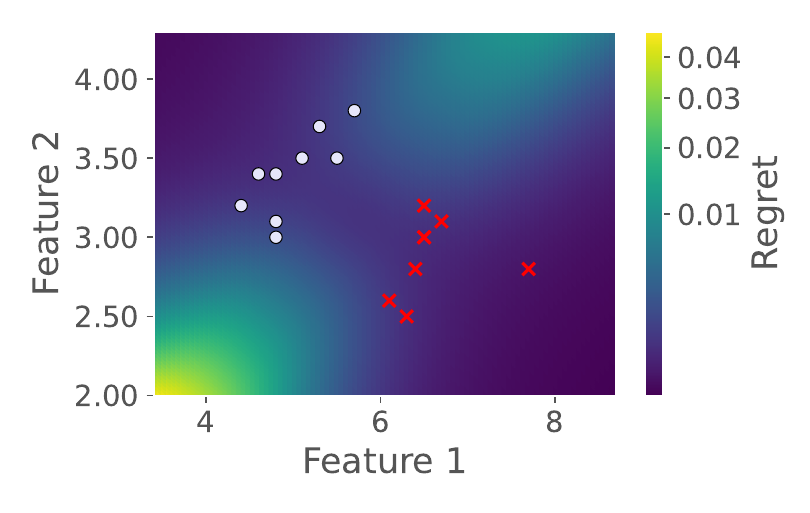}
\caption{pNML regret for a separable split \label{fig:single_layer_nn:syntetic_regret}}
\end{subfigure}
\begin{subfigure}[t]{0.49\textwidth}
\includegraphics[width=1.0\textwidth]{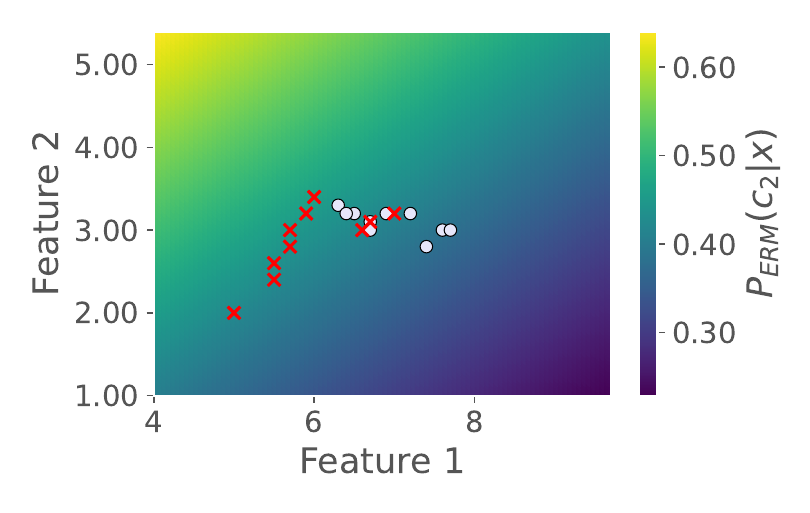}
\caption{ERM for a inseparable split \label{fig:single_layer_nn:syntetic_erm_prob_inseparable}}
\end{subfigure}
\begin{subfigure}[t]{0.49\textwidth}
\includegraphics[width=1.0\textwidth]{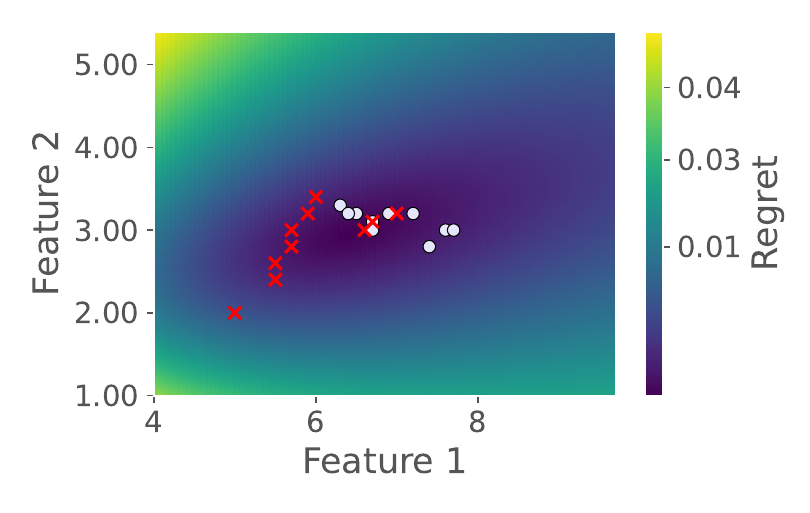}
\caption{pNML regret for a inseparable split \label{fig:single_layer_nn:syntetic_regret_inseparable}}
\end{subfigure}
\caption{The pNML regret for a Iris dataset}
\end{figure}


\subsubsection{The pNML regret characteristics using a low-dimensional dataset} \label{sec:single_layer_nn:low_dim_set}
We demonstrate the characteristics of the derived regret and show in what situations the prediction of the test sample can be trusted.
To visualize the pNML regret on a low-dimensional dataset, we use the Iris flower data set~\cite{fisher1936use}. We utilize two classes and two features and name them $c_1$, $c_2$, and feature 1, feature 2 respectively.

\Figref{fig:single_layer_nn:syntetic_erm_prob} shows the ERM probability assignment of class $c_2$ for a single layer NN that was fitted to the training data, which are marked in red.
At the top left and bottom right, the model predicts with high probability that a sample from these areas belongs to class $c_1$ and $c_2$ respectively.

\Figref{fig:single_layer_nn:syntetic_regret} presents the analytical pNML regret. 
At the upper left and lower right, the regret is low: Although there are no training samples there, these regions are far from the decision boundary, adding one sample would not alter the probability assignment significantly, thus the pNML prediction is close to the genie.
At the top right and bottom left, there are no training points therefore the regret is relatively high and the confidence in the prediction is low.
In \secref{sec:single_layer_nn:experiments}, we show that these test samples, which are associated with high regret, can be classified as OOD samples.

In addition, we visualize the regret for overlapping classes.
In \figref{fig:single_layer_nn:syntetic_erm_prob_inseparable}, the ERM probability assignment for inseparable class split is shown.
The ERM probability is lower than 0.7 for all test feature values. \Figref{fig:single_layer_nn:syntetic_regret_inseparable} presents the corresponding pNML regret. The pNML regret is small in the training data surroundings (including the mixed label area). The regret is large in areas where the training data is absent, as in the figure edges.

\subsubsection{Deep neural network adaptation}
\label{sec:single_layer_nn:dnn_adaptation}
In previous sections, we derived the pNML for a single layer NN. 
We next show that our derivations can, in fact, be applied to {\em any pretrained NN}, without requiring additional parameters or extra data. 

First extract the embeddings of the training set:
Denote $\phi(\cdot)$ as the embedding creation (feature extraction) using a pretrained ERM model, propagate the training samples through the DNN up to the last layer and compute the inverse of the data matrix $\phi(X_N)^+ \phi(X_N)^{+\top}$.
Then, given a specific test example $x$, extract its embedding $\phi(x)$ and its ERM probability assignment $\{p_i\}_{i=1}^C$.
Finally calculate the regret as described in~\eqref{eq:single_layer_nn:pnml} using the training set and test embedding vectors.


We empirically found that norms of OOD embeddings are lower than those of IND samples.
The regret depends on the norm of the test sample: For $0<a<b$, the regret of $a x$ is lower than the regret of $b x$.
Hence, we normalize all embeddings (training, IND, and OOD) to have $\normltwo$ norms equal to 1.0.

Samples with a high regret value are considered samples with a large distance from the genie, the learner that knows the true label, and therefore the prediction cannot be trusted. Our proposed method utilizes this regret value to determine whether a test data item represents a known or an unknown.

\subsubsection{Application to out-of-distribution detection} 
\label{sec:single_layer_nn:experiments}
We rigorously test the effectiveness of the pNML regret for OOD detection.
The motivation for using the individual setting and the pNML as its solution for OOD detection is that in the individual setting there is no assumption on the way the data is generated. The absence of assumption means that the result holds for a wide range of scenarios (PAC, stochastic, and even adversary) and specifically to OOD detection, where the OOD samples are drawn from an unknown distribution.

\paragraph{Experimental setup.} 
\label{sec:single_layer_nn:setup}
We follow the standard experimental setup~\cite{lee2018simple,liu2020energy,gram}.
All the assets we used are open-sourced with either Apache-2.0 License or Attribution-NonCommercial 4.0 International licenses.
We ran all experiments on NVIDIA K80 GPU.

IND sets:
For datasets that represent known classes, we use CIFAR-100, CIFAR-10~\cite{krizhevsky2014cifar} and  SVHN~\cite{netzer2011reading}. These sets contain RGB images with 32x32 pixels.
In addition, to evaluate higher resolution images, we use ImageNet-30 set~\cite{DBLP:conf/nips/HendrycksMKS19}.

\begin{table}[tb]
\centering
\small
\caption{AUROC of OOD detection for DenseNet-BC-100 model}
\label{tab:single_layer_nn:auroc_densenet}
\resizebox{\textwidth}{!}{%
\begin{tabular}{clcccc}
\toprule
IND & OOD &        Baseline/+pNML &            ODIN/+pNML &            Gram/+pNML &            OECC/+pNML \\
\midrule
\multirow{8}{*}{CIFAR-100} & iSUN &  69.7 / \textbf{96.4} &  84.5 / \textbf{96.7} &  99.0 / \textbf{99.5} &  99.2 / \textbf{99.5} \\
     & LSUN (R) &  70.8 / \textbf{96.6} &  86.0 / \textbf{96.9} &  99.3 / \textbf{99.7} &  99.4 / \textbf{99.6} \\
     & LSUN (C) &  80.1 / \textbf{93.1} &  91.5 / \textbf{93.1} &  91.4 / \textbf{94.5} &  93.9 / \textbf{96.1} \\
     & Imagenet (R) &  71.6 / \textbf{97.4} &  85.5 / \textbf{97.6} &  99.0 / \textbf{99.5} &  99.0 / \textbf{99.5} \\
     & Imagenet (C) &  76.2 / \textbf{95.7} &  88.8 / \textbf{96.0} &  97.7 / \textbf{98.7} &  98.2 / \textbf{99.0} \\
     & Uniform &   43.3 / \textbf{100} &   83.7 / \textbf{100} &    100 / \textbf{100} &   99.9 / \textbf{100} \\
     & Gaussian &   30.6 / \textbf{100} &   50.6 / \textbf{100} &    100 / \textbf{100} &    100 / \textbf{100} \\
     & SVHN &  82.6 / \textbf{96.2} &  92.5 / \textbf{96.2} &  97.3 / \textbf{98.4} &  97.0 / \textbf{97.5} \\
\midrule
\multirow{8}{*}{CIFAR-10} & iSUN &  94.8 / \textbf{98.7} &  98.9 / \textbf{98.9} &   99.8 / \textbf{100} &   99.9 / \textbf{100} \\
     & LSUN (R) &  95.5 / \textbf{98.9} &  99.2 / \textbf{99.2} &   99.9 / \textbf{100} &   99.9 / \textbf{100} \\
     & LSUN (C) &  93.0 / \textbf{96.4} &  95.8 / \textbf{96.4} &  97.5 / \textbf{98.7} &  98.9 / \textbf{99.9} \\
     & Imagenet (R) &  94.1 / \textbf{98.8} &  98.5 / \textbf{99.0} &  99.7 / \textbf{99.9} &  99.8 / \textbf{99.9} \\
     & Imagenet (C) &  93.8 / \textbf{97.7} &  97.6 / \textbf{97.9} &  99.3 / \textbf{99.7} &  99.5 / \textbf{99.9} \\
     & Uniform &   96.6 / \textbf{100} &    100 / \textbf{100} &    100 / \textbf{100} &    100 / \textbf{100} \\
     & Gaussian &   97.6 / \textbf{100} &    100 / \textbf{100} &    100 / \textbf{100} &    100 / \textbf{100} \\
     & SVHN &  89.9 / \textbf{98.4} &  94.6 / \textbf{98.7} &  99.1 / \textbf{99.6} &   99.6 / \textbf{100} \\
\midrule
\multirow{9}{*}{SVHN} & iSUN &  94.4 / \textbf{98.7} &  92.8 / \textbf{99.1} &  99.8 / \textbf{99.9} &    100 / \textbf{100} \\
     & LSUN (R) &  94.1 / \textbf{98.4} &  92.5 / \textbf{98.9} &   99.8 / \textbf{100} &    100 / \textbf{100} \\
     & LSUN (C) &  92.9 / \textbf{98.0} &  88.6 / \textbf{98.1} &  98.6 / \textbf{99.4} &   99.8 / \textbf{100} \\
     & Imagenet (R) &  94.8 / \textbf{98.6} &  93.3 / \textbf{99.0} &  99.7 / \textbf{99.9} &    100 / \textbf{100} \\
     & Imagenet (C) &  94.6 / \textbf{98.6} &  92.8 / \textbf{98.8} &  99.4 / \textbf{99.8} &    100 / \textbf{100} \\
     & Uniform &  93.2 / \textbf{99.8} &   91.6 / \textbf{100} &   99.9 / \textbf{100} &    100 / \textbf{100} \\
     & Gaussian &  97.4 / \textbf{99.8} &  98.9 / \textbf{99.9} &    100 / \textbf{100} &    100 / \textbf{100} \\
     & CIFAR-10 &  91.8 / \textbf{96.7} &  88.9 / \textbf{97.8} &  95.4 / \textbf{97.3} &   99.5 / \textbf{100} \\
     & CIFAR-100 &  91.4 / \textbf{96.7} &  88.2 / \textbf{97.8} &  96.4 / \textbf{98.0} &   99.6 / \textbf{100} \\
\bottomrule
\end{tabular}

}
\end{table}

OOD sets:
The OOD sets are represented by TinyImageNet~\cite{liang2017enhancing}, LSUN~\cite{yu15lsun}, iSUN~\cite{xu2015turkergaze}, Uniform noise images, and Gaussian noise images. 
We use two variants of TinyImageNet and LSUN sets: a 32x32 image crop that is represented by ``(C)'' and a resizing of the images to 32x32 pixels that termed by ``(R)''.
We also used CIFAR-100, CIFAR-10, and SVHN as OOD for models that were not trained with them.

Evaluation methodology:
We benchmark our approach by adopting the following metrics~\cite{gram,lee2018simple}: 
(i) AUROC: The area under the receiver operating characteristic curve of a threshold-based detector. A perfect detector corresponds to an AUROC score of 100\%.
(ii) TNR at 95\% TPR: The probability that an OOD sample is correctly identified (classified as negative) when the true positive rate equals 95\%.
(iii) Detection accuracy: Measures the maximum possible classification accuracy over all possible thresholds.

\begin{table}[tb]
\centering
\small
\caption{AUROC comparison of OOD detection for ResNet-34 model}
\label{tab:single_layer_nn:auroc_resnet}
\resizebox{\textwidth}{!}{
\begin{tabular}{clcccc}
\toprule
IND & OOD &        Baseline/+pNML &            ODIN/+pNML &            Gram/+pNML &            OECC/+pNML \\

\midrule
\multirow{8}{*}{CIFAR-100} & iSUN &  75.7 / \textbf{83.0} &  85.6 / \textbf{87.6} &  98.8 / \textbf{99.1} &  99.0 / \textbf{99.3} \\
     & LSUN (R) &  75.6 / \textbf{83.8} &  85.4 / \textbf{88.0} &  99.2 / \textbf{99.4} &  99.3 / \textbf{99.6} \\
     & LSUN (C) &  75.5 / \textbf{83.1} &  82.6 / \textbf{88.1} &  92.2 / \textbf{94.6} &  95.7 / \textbf{97.8} \\
     & Imagenet (R) &  77.1 / \textbf{84.4} &  87.7 / \textbf{88.5} &  98.9 / \textbf{99.2} &  98.7 / \textbf{98.9} \\
     & Imagenet (C) &  79.6 / \textbf{85.8} &  85.6 / \textbf{88.6} &  97.7 / \textbf{98.4} &  97.9 / \textbf{98.1} \\
     & Uniform &  85.2 / \textbf{98.1} &  99.0 / \textbf{99.4} &    100 / \textbf{100} &    100 / \textbf{100} \\
     & Gaussian &  45.0 / \textbf{86.5} &  83.8 / \textbf{95.7} &    100 / \textbf{100} &    100 / \textbf{100} \\
     & SVHN &  79.3 / \textbf{90.9} &  94.0 / \textbf{95.4} &  96.0 / \textbf{97.9} &  97.0 / \textbf{97.6} \\
\midrule
\multirow{8}{*}{CIFAR-10} & iSUN &  91.0 / \textbf{96.4} &  94.0 / \textbf{97.5} &   99.8 / \textbf{100} &  99.9 / \textbf{99.9} \\
     & LSUN (R) &  91.1 / \textbf{96.6} &  94.1 / \textbf{97.7} &   99.9 / \textbf{100} &   \textbf{100} / 99.9 \\
     & LSUN (C) &  91.8 / \textbf{95.4} &  93.6 / \textbf{95.6} &  97.9 / \textbf{99.1} &  99.1 / \textbf{99.5} \\
     & Imagenet (R) &  91.0 / \textbf{95.4} &  93.9 / \textbf{96.6} &  99.7 / \textbf{99.9} &  99.9 / \textbf{99.9} \\
     & Imagenet (C) &  91.4 / \textbf{95.4} &  93.3 / \textbf{96.2} &  99.3 / \textbf{99.7} &  99.7 / \textbf{99.8} \\
     & Uniform &  96.1 / \textbf{99.8} &   99.9 / \textbf{100} &    100 / \textbf{100} &    100 / \textbf{100} \\
     & Gaussian &   97.5 / \textbf{100} &    100 / \textbf{100} &    100 / \textbf{100} &    100 / \textbf{100} \\
     & SVHN &  89.9 / \textbf{95.1} &  95.8 / \textbf{97.9} &  99.5 / \textbf{99.8} &  99.8 / \textbf{99.8} \\
\midrule
\multirow{9}{*}{SVHN} & iSUN &  92.2 / \textbf{97.1} &  91.4 / \textbf{98.0} &  99.8 / \textbf{99.9} &    100 / \textbf{100} \\
     & LSUN (R) &  91.5 / \textbf{96.7} &  90.6 / \textbf{97.7} &   99.8 / \textbf{100} &    100 / \textbf{100} \\
     & LSUN (C) &  92.8 / \textbf{97.0} &  92.3 / \textbf{97.1} &  98.8 / \textbf{99.6} &  99.7 / \textbf{99.9} \\
     & Imagenet (R) &  93.5 / \textbf{97.5} &  92.8 / \textbf{98.3} &  99.8 / \textbf{99.9} &    100 / \textbf{100} \\
     & Imagenet (C) &  94.2 / \textbf{97.5} &  93.7 / \textbf{98.2} &  99.5 / \textbf{99.9} &   99.9 / \textbf{100} \\
     & Uniform &  96.0 / \textbf{98.5} &  95.5 / \textbf{99.5} &    100 / \textbf{100} &    100 / \textbf{100} \\
     & Gaussian &  96.1 / \textbf{98.4} &  96.1 / \textbf{99.6} &    100 / \textbf{100} &    100 / \textbf{100} \\
     & CIFAR-10 &  93.0 / \textbf{97.4} &  92.0 / \textbf{98.0} &  97.4 / \textbf{99.3} &  99.4 / \textbf{99.8} \\
     & CIFAR-100 &  92.5 / \textbf{97.1} &  91.7 / \textbf{97.8} &  97.5 / \textbf{99.2} &  99.4 / \textbf{99.8} \\
\bottomrule
\end{tabular}

}
\end{table}

\paragraph{Results.} 
\label{sec:single_layer_nn:results}
We build upon existing leading methods: Baseline~\cite{hendrycks2016baseline}, ODIN~\cite{liang2017enhancing}, Gram~\cite{gram},  OECC~\cite{PAPADOPOULOS2021138}, and Energy~\cite{liu2020energy}.
We use the following pretrained models: ResNet-34~\cite{he2016deep}, DenseNet-BC-100~\cite{huang2017densely} and WideResNet-40~\cite{zagoruyko2016wide}. Training was performed using CIFAR-100, CIFAR-10 and SVHN, each training set used separately to provide a complete picture of our proposed method's capabilities. 
Notice that ODIN, OECC, and Energy methods use OOD sets during training and the Gram method requires IND validation samples.

\Tableref{tab:single_layer_nn:auroc_densenet} and \Tableref{tab:single_layer_nn:auroc_resnet} show the AUROC of different OOD sets for DenseNet and ResNet models respectively.
Our approach improves all the compared methods in nearly all combinations of IND-OOD sets. 
The largest AUROC gain over the current state-of-the-art is of CIFAR-100 as IND and LSUN (C) as OOD: For the DenseNet model, we improve Gram and OECC method by 3.1\% and 2.2\% respectively. For the ResNet model, we improve this combination by 2.4\% and 2.1\% respectively. 
The additional metrics (TNR at 95\% FPR and detection accuracy) are shown in \secref{sec:single_layer_nn:ood_results}.

The Baseline method uses a pretrained ERM model with no extra data. 
Combining the pNML regret with the standard ERM model as shown in the Baseline+pNML column surpasses Baseline by up to 69.4\% and 41.5\% for DensNet and ResNet, respectively.
Also, Baseline+pNML is comparable to the more sophisticated methods:
Although it lacks tunable parameters and does not use extra data, Baseline+pNML outperforms ODIN in most DenseNet IND-OOD set combinations.

Evidently, our method improves the AUROC of the OOD detection task in 14 out of 16 IND-OOD combinations.
The most significant improvement is in CIFAR-100 as IND and ImageNet (R) and iSUN as the OOD sets. In these sets, we improve the AUROC by 15.6\% and 15.2\% respectively. 
For TNR at TPR 95\%, the pNML regret enhances the CIFAR-100 and Gaussian combination by 90.4\% and achieves a perfect separation of IND-OOD samples. 

For high resolution images, we use Resnet-18 and ResNet-101 models trained on ImageNet. We utilize the ImageNet-30 training set for computing $\phi(X_N)^+ \phi(X_N)^{+ \top}$. All images were resized to $254 \times 254$ pixels. We compare the result to the Baseline method in~\Tableref{tab:single_layer_nn:aruoc_imagnet30}.
The table shows that the pNML outperforms Baseline by up to 9.8\% and 8.23\% for ResNet-18 and ResNet-101 respectively.

\begin{table}[t]
\centering
\small
\caption{AUROC of OOD detection for WideResNet-40 model}
\label{tab:single_layer_nn:auroc_wrn_results}
\begin{tabular}{clccc}
\toprule
          IND & OOD &                 AUROC &        TNR at TPR 95\% &        Detection Acc. \\
 \midrule & & \multicolumn{3}{c}{Energy/+pNML} \\ \cmidrule{3-5} 

\multirow{8}{*}{CIFAR-100} & iSUN &  78.4 / \textbf{93.6} &  30.7 / \textbf{62.5} &  71.1 / \textbf{87.0} \\
          & LSUN (R) &  80.3 / \textbf{94.1} &  31.2 / \textbf{65.5} &  73.1 / \textbf{87.5} \\
          & LSUN (C) &  \textbf{95.9} / 95.5 &  \textbf{80.0} / 79.3 &  \textbf{89.3} / 89.1 \\
          & Imagenet (R) &  71.4 / \textbf{87.0} &  22.1 / \textbf{44.8} &  66.1 / \textbf{79.9} \\
          & Imagenet (C) &  79.7 / \textbf{87.3} &  36.9 / \textbf{49.4} &  72.8 / \textbf{79.7} \\
          & Uniform &  97.9 / \textbf{99.8} &   95.2 / \textbf{100} &  95.8 / \textbf{99.6} \\
          & Gaussian &  92.0 / \textbf{99.8} &    9.6 / \textbf{100} &  92.3 / \textbf{99.8} \\
          & SVHN &  \textbf{96.5} / 96.4 &  79.2 / \textbf{82.8} &  90.5 / \textbf{91.3} \\
\midrule
\multirow{8}{*}{CIFAR-10} & iSUN &  99.3 / \textbf{99.4} &  98.3 / \textbf{98.7} &  96.7 / \textbf{97.0} \\
          & LSUN (R) &  99.3 / \textbf{99.5} &  98.6 / \textbf{99.0} &  97.0 / \textbf{97.3} \\
          & LSUN (C) &  99.4 / \textbf{99.5} &  98.6 / \textbf{98.6} &  97.0 / \textbf{97.1} \\
          & Imagenet (R) &  98.1 / \textbf{98.1} &  92.0 / \textbf{92.4} &  94.0 / \textbf{94.0} \\
          & Imagenet (C) &  98.6 / \textbf{98.6} &  94.4 / \textbf{94.6} &  94.9 / \textbf{94.9} \\
          & Uniform &  99.0 / \textbf{99.9} &    100 / \textbf{100} &  98.7 / \textbf{99.8} \\
          & Gaussian &  99.1 / \textbf{99.9} &    100 / \textbf{100} &  98.7 / \textbf{99.8} \\
          & SVHN &  99.3 / \textbf{99.6} &  98.3 / \textbf{98.9} &  96.9 / \textbf{97.6} \\
\bottomrule
\end{tabular}

\end{table}

\begin{table}[t]
\centering
\small
\caption{AUROC of OOD detection with ImageNet-30 as the IND set}
\label{tab:single_layer_nn:aruoc_imagnet30}
\begin{tabular}{clcc}
\toprule
IND & OOD &        ResNet-18 &   ReseNet-101  \\
\midrule 
 & & Baseline/+pNML & Baseline/+pNML \\ \cmidrule{3-4} 
\multirow{8}{*}{ImageNet-30} 
    & iSUN      &  95.58 / \textbf{99.74}   &  96.26 / \textbf{99.54}   \\
    & LSUN (R)  &  95.51 / \textbf{99.72}   &  95.77 / \textbf{99.43}   \\
    & LSUN (C)  &  96.89 / \textbf{99.77}   &  98.00 / \textbf{99.86}   \\
    & Uniform   &  99.35 / \textbf{99.99}   &   98.70 / \textbf{100}    \\
    & Gaussian  &  98.78 / \textbf{100}     &   98.61 / \textbf{100}    \\
    & SVHN      &  99.18 / \textbf{99.99}   &  98.94 / \textbf{99.98}   \\
    & CIFAR-10  &  89.99 / \textbf{99.79}   &   91.24 / \textbf{99.47}  \\
    & CIFAR-100 &  92.15 / \textbf{92.15}   &  93.39 / \textbf{99.58}   \\
\bottomrule
\end{tabular}

\end{table}

\section{Fully training the last model layer}
In the previous section, we derived an analytical solution for applying the pNML on the last later of a DNN. In the derivation, we assumed the MSE loss with a specific hypothesis class. An alternative to this approach, which we present in thus section, is the explicitly train the last layer of the model.

\subsection{Efficient evaluation of the pNML regret}\label{sec:fully_train_last_layer:heart}
Intuitively, the pNML can be described as follows: To assign a probability for a potential outcome, add it to the trainset, find the best-suited model, and take the probability it gives to that label. 
Follow this procedure for every label and normalize to get a valid probability assignment. Use the log normalization factor as the confidence measure.
This method can be extended to any general learning procedure that generates a prediction based on a trainset. 
One such method can be the stochastic gradient descent used in training DNN.

One of the main problems in using the pNML learner in DNN is the richness of the hypothesis class. 
State of the art convolutional networks for image classification trained with stochastic gradient methods easily fit a random labeling of the training data~\cite{DBLP:conf/iclr/ZhangBHRV17}.
Having a family $P_{\theta}$ that is too large, as in DNN models, produces a perfect fit for every test label option and, therefore, high regret.
We propose an efficient pNML scheme which both reduces the size of the hypothesis class size and accelerates the compute time.

Our proposed approach is as follows: train a DNN model using the standard ERM training procedure. 
Next, extract the features of the trainset: propagate the trainset through the DNN up to the last layer and create the features trainset $\{(\phi(x_i),y_i)\}_{i=1}^N$, (i.e., the trainset embeddings), where $\phi(\cdot)$ represents the feature extraction.
Then, given a specific test example $x$, extract its features $\phi(x)$. 
We define the hypothesis class as one fully connected layer, $w$, and execute the \emph{fine-tuning} phase: 
We add the pair $(\phi(x),y)$ to the trainset embeddings with arbitrary choice of the label $y$. We train the single fully connected layer using the ERM procedure
\begin{equation}
w_{y} = \argmin_{w} \left[\sum_{i=1}^N  - \log p_{w}\left(y_i|\phi(x_i)\right) 
                                        - \log p_{w}\left(y|\phi(x)\right) \right].
\end{equation}
Given the trained layer, we perform a prediction of the class we trained with:
\begin{equation}
    p_{y} = p_{w_{y}}(y|\phi(x)).
\end{equation}
We repeat this process for every potential test label value.
The probability assignment of the pNML learner for a specific label is the predicted distribution for that label after the fine-tuning step, normalized by the summation of all the corresponding probabilities.
Notice that the normalization makes the pNML prediction a valid probability assignment.
The associated regret is then
\begin{equation}
    \Gamma = \log \sum_{y \in \mathcal{Y}} p_{y} 
    = \log \sum_{y \in \mathcal{Y}}  p_{w_{y}}(y|\phi(x)). 
\end{equation}

Samples with a high regret value, $\Gamma$, are considered samples with a large distance from the reference learner -- a learner that knows the true label -- and therefore the prediction cannot be trusted. Our proposed method, therefore, determines whether a test data item represents a known or unknown class based on this regret value.

\begin{figure}[t]
\label{dnn:Efficient_pNML_procedure}
\begin{algorithm}[H]
    \caption{Efficient pNML procedure}
    \label{alg:deep_pNML}
    \begin{algorithmic}
        \State {\bfseries Input:} Trainset $\{(x_i,y_i)\}_{i=1}^N$, test data $x$, features extractor $\phi(\cdot)$.
        \ForEach {$y \in \mathcal{Y}$}
        \State $w_y = \argmax_{w} \left[ p_{w}\left(y|\phi(x)\right) \cdot \prod_{i=1}^N p_{w}\left(y_i|\phi(x_i)\right) \right]$
        \State $p_y = p_{w_y}(y|\phi(x))$ \Comment{Predict the class we trained with}
        \EndFor
        \State Normalization factor: $C=\sum_{y\in \mathcal{Y}}p_y$ 
        \State Regret: $\Gamma=\log{C}$
        \ForEach {$y \in \mathcal{Y}$} \Comment{Normalize to get a valid probability assignment}
        \State $q_{\textit{pNML}}(y)= \frac{1}{C} p_y$
        \EndFor
        \State {\bfseries Return} $q_{\textit{pNML}}$,  $\Gamma$
    \end{algorithmic}
    \end{algorithm}
\end{figure}

\subsection{Experiments} \label{sec:fully_train_last_layer:experiments}
We next show the effectiveness of our method for open-set classification on several standard benchmark datasets.

\paragraph{Evaluation methodology.}
We follow the proposed Open-Set Classification Rate (OSCR) evaluation metric~\cite{dhamija2018reducing,phillips2011evaluation}.
In the evaluation phase, two datasets are used: $\mathcal{D}_c$ which contains test data with known classes and $\mathcal{D}_u$ that has test data with unknowns.
Let $T$ be a score threshold.
For samples from $\Dc$, the Correct Classification Rate (CCR) is defined as samples fraction for which the correct class $c$ has the maximum probability and has a confidence score $S(x)$ greater than $T$
\begin{equation}  \label{eq:fully_train_last_layer:ccr}
\mathrm{CCR}(T) = \frac{1}{|\Dc|}\bigl|\{x \mid  x \in \Dc \  \wedge 
    \argmax_{y \in \mathcal{Y}} p(y|x) = c \ \wedge 
    S(x) > T \}\bigr|.
\end{equation}
The False Positive Rate (FPR) is the fraction of samples from $\Du$ that are classified as any known class with a confidence score greater than $T$
\begin{equation}  \label{eq:fully_train_last_layer:fpr}
    \mathrm{FPR}(T) =  \frac{1}{|\Du|} \left|\{x \mid x \in \Du \wedge S(x) \geq T \}\right|.
\end{equation}

We vary the threshold $T$ and define the CCR as a function of the FPR. 
For the smallest $T$, we get the largest FPR, a value of 1.0. 
The CCR at FPR 1.0 is identical to the closed-set performance, i.e., the classification accuracy on $\Dc$.
We use the Area Under the Curve (AUC) of the CCR versus the FPR in our evaluations along with CCR values at specific FPRs: 0.1, 0.2, 0.6, 0.8 and 1.0.

\paragraph{Datasets.} \label{sec:fully_train_last_layer:datasets}
We test our method rigorously on the following standard benchmarks. 
\begin{itemize}
\item 
CIFAR10 and CIFAR100~\cite{krizhevsky2014cifar}: These datasets include $32 \times 32$ natural color images. In CIFAR10, 10 classes are defined with 50,000 training images and 10,000 test images to a total of  6,000 images per class, while CIFAR100 has 100 classes with a total of 600 images per class.
\item 
CelebA~\cite{liu2015faceattributes}: The Celeb Faces Attributes (CelebA) set consists of 202,000 face images. Each image is labeled with 40 binary attributes along with the subject identity.
We arbitrarily select the ``Arched Eyebrows" attribute with training and test set sizes of 50,000 and 10,000, respectively.
\item 
LSUN~\cite{yu15lsun}: The Large-scale Scene Understanding set (LSUN) consists of 10,000 test images from 10 different scene categories such as ``bridge'' and ``bedroom''.
\item 
SVHN~\cite{netzer2011reading}: The Street View House Numbers (SVHN) dataset contains $32 \times 32$ color images of house numbers. There are 10 classes consisting of the digits $0$--$9$. The training set has 604,388 images, and the test set contains 26,032 images.
\item 
Noise: A synthetic Gaussian noise dataset consisting of 10,000 random 2D Gaussian noise images.
Each RGB value of every pixel is sampled from a normal distribution with zero mean and unit variance.
\end{itemize}

For the open-set classification task, we use CIFAR10, CIFAR100 and CelebA as closed-sets separately.
LSUN, SVHN and Noise are used to simulate OOD images. 
We suppressed overlap classes from LSUN, for instance, the class ``bedroom'' in LSUN is overlapped with ``bed'' in CIFAR100 and therefore omitted.

\paragraph{Compared methods.} \label{sec:fully_train_last_layer:compared_methods}
We compare our proposed confidence score with several recent leading methods.
\begin{itemize}
\item 
Max-prob~\cite{hendrycks2016baseline}:
A naive method that assigns the confidence score based on the maximum probability of a trained DNN.
Input with a low maximum probability is considered to have an unknown class.
\item 
ODIN~\cite{liang2017enhancing}:
To increase the margin between the maximum softmax score of data with known classes and data with unknowns, this method pre-processes the input by perturbing it with the loss gradient.
The score is the DNN's output maximum probability when feeding it with the perturbed input.
\item 
Entropic Open-Set~\cite{dhamija2018reducing}: 
During training, a dataset with unknown classes $\Du'$ is used. 
The loss function is modified such that data from $\Du'$ produce a uniform probability vector and therefore will have a low maximum probability score.
The value of the maximum probability of the prediction is used to distinguish between images with known classes and inputs with unknowns.
Different unknowns sets ($\Du'$ and $\Du$) are used in the training and evaluation phases.
\item 
Objectosphere~\cite{dhamija2018reducing}:
Similarly to the Entropic Open-Set approach, a model is trained to have a low response to unknown classes.
An additional loss term is added that directly penalizes high norm features of OOD samples and encourages closed-set samples to have high features norm.
In the evaluation, the maximum probability value of the prediction is used as the confidence score.
\item 
Cosine Similarity~\cite{techapanurak2019hyperparameterfree,wang2018cosface}: 
This method proposes a change in the softmax layer.
The main idea is to use the cosine of the angle between the weights of the last fully connected layer that are associated with a specific class and the features as the logits that feed the softmax layer. 
Same as other methods, this approach uses the maximum probability for confidence.
\item 
Maximum Discrepancy~\cite{yu2019unsupervised}:
Two classifiers are used along with an additional OOD training set.
During training, a new loss term is adopted that increases the discrepancy between the classifiers on the OOD training set.
The confidence score is the $\normlone$ distance between the classifiers' probability assignments.
\end{itemize}

\paragraph{Open-set performance.}
\begin{figure}[tb]
\begin{center}
    \begin{subfigure}[t]{0.325\textwidth}
    \centering
    \centerline{\includegraphics[width=\columnwidth]{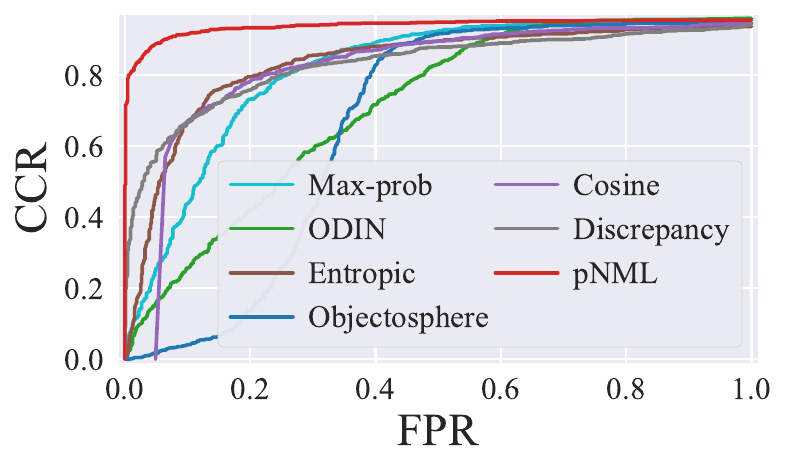}}
    \caption{CIFAR10 \label{fig:fully_train_last_layer:cifar10_lsun}}
    \end{subfigure}
    \begin{subfigure}[t]{0.325\textwidth}
    \centering
    \centerline{\includegraphics[width=\columnwidth]{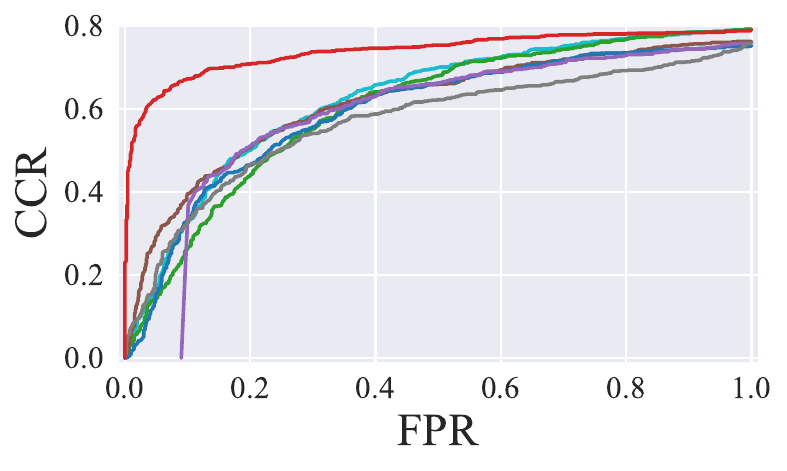}}
    \caption{CIFAR100 \label{fig:fully_train_last_layer:cifar100_lsun}}
    \end{subfigure}
    \begin{subfigure}[t]{0.325\textwidth}
    \centering
    \centerline{\includegraphics[width=\columnwidth]{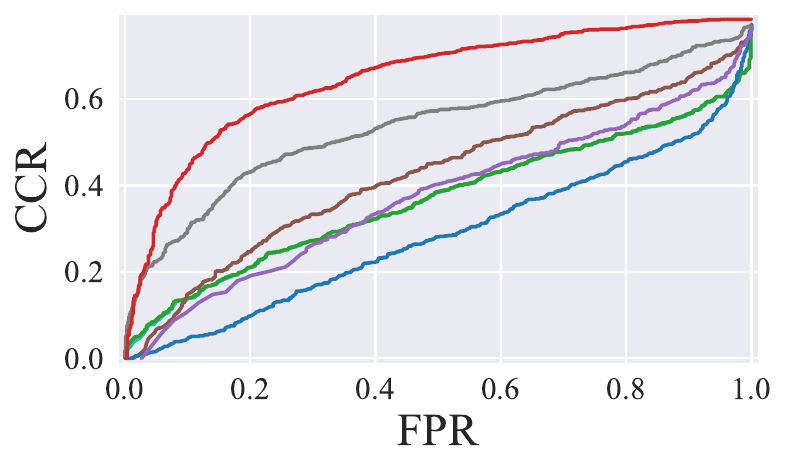}}
    \caption{CelebA \label{fig:fully_train_last_layer:celeba_lsun}}
    \end{subfigure}
\caption{CCR vs FPR curves of the compared methods}
\label{fig:fully_train_last_layer:performance_lsun}
\end{center}
\end{figure}

\begin{table}[tb]
    \centering
    \caption{DenseNet100 with CIFAR10 trainset experimental results \label{table:cifar10}}
    \smallskip
    \small
    \begin{tabular}{cccccccc}
\toprule
\multirow{2}{*}{$\begin{array}{c}\hbox{\bf Unknowns} \\\end{array}$}&
\multirow{2}{*}{\bf Algorithm} & 
\multirow{2}{*}{ $\begin{array}{c}\hbox{\bf AUC}\end{array}$} &
\multicolumn{5}{c}{\bf CCR at FPR of} \\
& & & {$0.1$} & {$0.2$} &  {$0.6$} & {$0.8$} & {$1.0$}
\\ 
\hline
\multirow{6}{*}{LSUN}  &          
    Max-prob           &  0.82 &  0.44 &  0.73 &  0.94 &  0.95 &  0.96 \\
&   ODIN               &  0.71 &  0.26 &  0.42 &  0.91 &  0.95 &  0.96 \\
&   Entropic Open-Set  &  0.83 &  0.67 &  0.80 &  0.91 &  0.93 &  0.94 \\
&   Objectosphere      &  0.66 &  0.04 &  0.14 &  0.93 &  0.94 &  0.95 \\
&   Cosine Similarity  &  0.82 &  0.67 &  0.78 &  0.92 &  0.93 &  0.94 \\
&   Max. Discrepancy        &  0.83 &  0.67 &  0.76 &  0.89 &  0.92 &  0.95 \\
&   Ours               &  \textbf{0.94} &  \textbf{0.92} &  \textbf{0.93} &  \textbf{0.95} &  \textbf{0.95} &  \textbf{0.96} \\
\cline{1-8}
\multirow{6}{*}{SVHN} &         
    Max-prob           &  0.83 &  0.56 &  0.76 &  0.93 &  0.95 &  0.96 \\
&   ODIN               &  0.74 &  0.41 &  0.56 &  0.89 &  0.95 &  0.96 \\
&   Entropic Open-Set  &  0.88 &  0.79 &  0.86 &  0.92 &  0.93 &  0.94 \\
&   Objectosphere      &  0.65 &  0.06 &  0.19 &  0.92 &  0.94 &  0.95 \\
&   Cosine Similarity  &  0.92 &  0.88 &  0.91 &  0.94 &  0.94 &  0.95 \\
&   Max. Discrepancy        &  0.90 &  0.84 &  0.86 &  0.92 &  0.92 &  0.94 \\
&   Ours               &  \textbf{0.94} &  \textbf{0.93} &  \textbf{0.94} &  \textbf{0.95} &  \textbf{0.95} &  \textbf{0.96} \\
\cline{1-8}
\multirow{6}{*}{Noise}& 
    Max-prob           &  0.15 &  0.07 &  0.09 &  0.16 &  0.20 &  0.96 \\
&   ODIN               &  0.48 &  0.32 &  0.36 &  0.51 &  0.60 &  0.96 \\
&   Entropic Open-Set  &  0.93 &  0.92 &  0.92 &  0.93 &  0.94 &  0.94 \\
&   Objectosphere      &  0.27 &  0.02 &  0.03 &  0.16 &  0.64 &  0.95 \\
&   Cosine Similarity  &  0.93 &  0.93 &  0.93 &  0.94 &  0.94 &  0.95 \\
&   Max. Discrepancy        &  0.86 &  0.83 &  0.84 &  0.86 &  0.88 &  0.94 \\
&   Ours               &  \textbf{0.95} &  \textbf{0.94} & \textbf{0.95} &  \textbf{0.95} &  \textbf{0.96} &  \textbf{0.96} \\
\bottomrule
\end{tabular}

\end{table}

We adopt the DenseNet architecture~\cite{huang2017densely}.
We follow the originally proposed setup with a depth of 100, growth rate 12 (Dense-BC) and dropout rate 0.
We use CIFAR10 as the dataset with known classes and CelebA as OOD set to train and optimize the ODIN, Entropic Open-Set, Objectosphere and Maximum Discrepancy methods.
When applying our proposed method, we extract CIFAR10 trainset features $\{\phi(x_i),y_i\}_{i=1}^N$ and execute the fine-tuning phase of our method: We train a fully connected layer from scratch for the ten possible labels of the test image.

The AUC of the CCR as a function of the FPR is stated in~\tableref{table:cifar10}.
When comparing to other approaches our method provides the highest AUC values with a substantial improvement of 11\% over the Entropic Open-Set and the Maximum Discrepancy approaches on the LSUN benchmark.
The performance on the LSUN dataset can be also viewed at \figref{fig:fully_train_last_layer:cifar10_lsun}. Our method provides the highest CCR for all FPR rates.
\tableref{table:cifar10} also shows an improvement of 2\% over the Cosine Similarity method which is the next best one for the SVHN and Noise datasets as unexpected inputs.
The most significant improvement value on SVHN is at FPR of 0.1, in which our method improves the CCR by 5\%.

\begin{figure}[bht]
    \centering
    \includegraphics[width=\columnwidth]{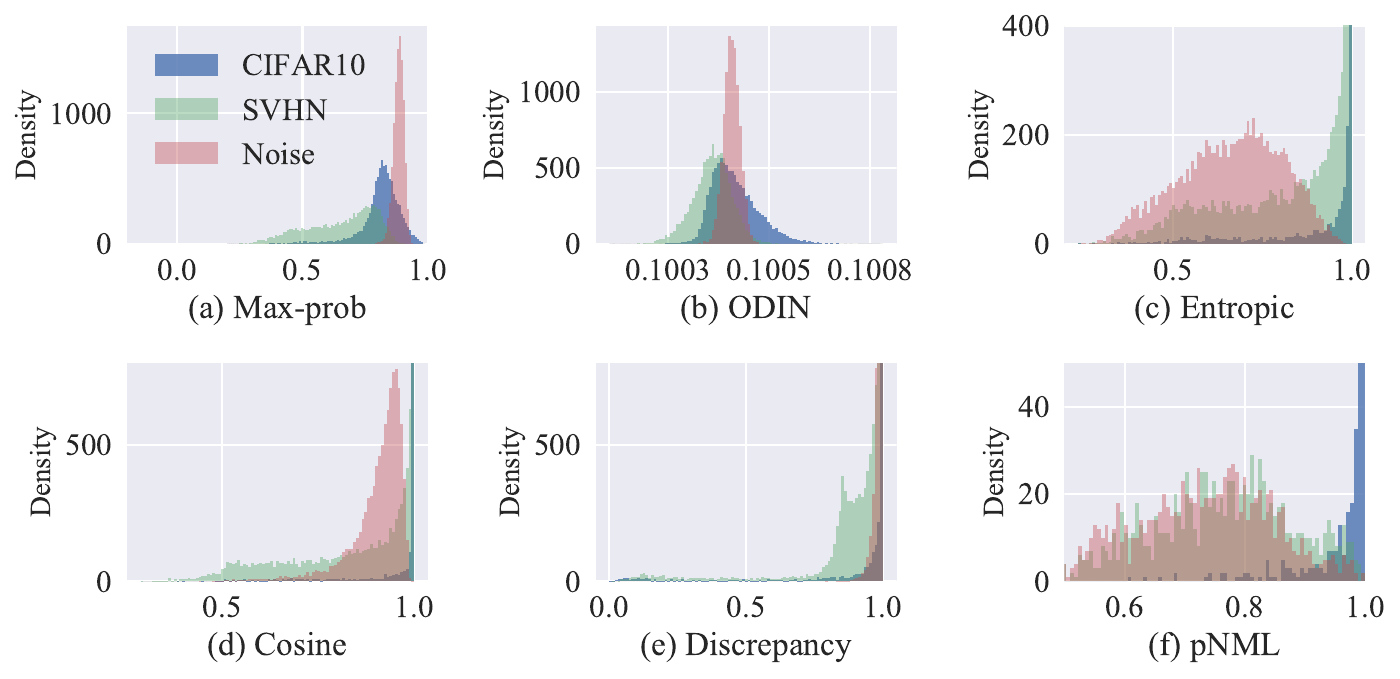}
    \caption{Confidence histograms for DenseNet100 CIFAR10 based models}
    \label{fig:fully_train_last_layer:cifar10_hist}
\end{figure}
We present in \figref{fig:fully_train_last_layer:cifar10_hist} the confidence score histograms of the compared methods for CIFAR10 as the closed-set.
For the Max-prob and ODIN methods, there is almost no separation and even higher confidence for noise images.
\Figref{fig:fully_train_last_layer:cifar10_hist}b also shows that ODIN's maximum probability score is relatively low, it is bounded between 0.1002 and 0.1008.
The reason lies in the adversarial attack that the method performs on the input which causes the maximum probability to drop. 
\Figref{fig:fully_train_last_layer:cifar10_hist}c presents the Entropic Open-Set method which manages to separate between the CIFAR10 dataset and Noise.
Most of Noise confidence scores are concentrated at a maximum probability value of 0.75, while CIFAR10 has a peak at a maximum probability value of 0.98.
For the SVHN dataset, however, we see an almost total overlap with CIFAR10 images confidence score.
Both Cosine Similarity and our methods have a clear distinction between CIFAR10 confidence scores and the Noise dataset.
In addition, our regret based score, as shown in \figref{fig:fully_train_last_layer:cifar10_hist}f, treats the SVHN and Noise dataset about the same and manages to separate both of them from CIFAR10 closed-set images.

\begin{table}[tb]
    \centering
    \caption{WideResNet16 with CIFAR100 trainset experimental results \label{table:cifar100}}
    \small
    \begin{tabular}{ccccccccc}
\toprule
\multirow{2}{*}{$\begin{array}{c}\hbox{\bf Unknowns} \\\end{array}$}&
\multirow{2}{*}{\bf Algorithm} & 
\multirow{2}{*}{ $\begin{array}{c}\hbox{\bf AUC}\end{array}$} &
\multicolumn{5}{c}{\bf CCR at FPR of} \\
& & & {$0.1$} & {$0.2$} &  {$0.6$} & {$0.8$} & {$1.0$}
\\ 
\hline
\multirow{6}{*}{LSUN} &         
    Max-prob           &  0.62 &  0.33 &  0.50 &  0.73 &  0.77 &  0.79 \\
&   ODIN               &  0.60 &  0.27 &  0.44 &  0.72 &  0.77 &  0.79 \\
&   Entropic Open-Set  &  0.61 &  0.40 &  0.51 &  0.69 &  0.74 &  0.76 \\
&   Objectosphere      &  0.59 &  0.33 &  0.46 &  0.69 &  0.73 &  0.75 \\
&   Cosine Similarity  &  0.59 &  0.37 &  0.51 &  0.69 &  0.73 &  0.76 \\
&   Max. Discrepancy        &  0.57 &  0.32 &  0.47 &  0.65 &  0.69 &  0.76 \\
&   Ours               &  \textbf{0.74} &  \textbf{0.67} &  \textbf{0.71} &  \textbf{0.77} &  \textbf{0.78} &  \textbf{0.79} \\
\cline{1-8}
\multirow{6}{*}{SVHN} &         
    Max-prob           &  0.73 &  0.63 &  0.69 &  0.77 &  0.78 &  0.79 \\
&   ODIN               &  0.73 &  0.66 &  0.71 &  0.76 &  0.78 &  0.79 \\
&   Entropic Open-Set  &  0.68 &  0.56 &  0.62 &  0.73 &  0.75 &  0.76 \\
&   Objectosphere      &  0.67 &  0.53 &  0.62 &  0.73 &  0.75 &  0.75 \\
&   Cosine Similarity  &  0.65 &  0.52 &  0.60 &  0.71 &  0.73 &  0.75 \\
&   Max. Discrepancy        &  0.66 &  0.51 &  0.60 &  0.72 &  0.74 &  0.76 \\
&   Ours               &  \textbf{0.75} &  \textbf{0.70} &  \textbf{0.73} &  \textbf{0.77} &  \textbf{0.78} &  \textbf{0.79} \\
\cline{1-8}
\multirow{6}{*}{Noise}& 
    Max-prob           &  0.61 &  0.52 &  0.55 &  0.63 &  0.66 &  0.79 \\
&   ODIN               &  0.66 &  0.58 &  0.62 &  0.68 &  0.72 &  0.79 \\
&   Entropic Open-Set  &  0.66 &  0.61 &  0.63 &  0.66 &  0.69 &  0.76 \\
&   Objectosphere      &  0.51 &  0.33 &  0.41 &  0.56 &  0.65 &  0.75 \\
&   Cosine Similarity  &  0.34 &  0.00 &  0.00 &  0.43 &  0.47 &  0.75 \\
&   Max. Discrepancy        &  0.36 &  0.28 &  0.30 &  0.38 &  0.41 &  0.76 \\
&   Ours               &  \textbf{0.76} &  \textbf{0.73} &  \textbf{0.74} &  \textbf{0.76} &  \textbf{0.78} &  \textbf{0.79} \\
\bottomrule
\end{tabular}
\end{table}

Next, we evaluate WideResNet~\cite{BMVC2016_87} and utilize CIFAR100 as the closed-set. 
We used the same training method as in CIFAR10 experiment.
We executed all labels in parallel, 100 possible labels, on Nvidia Tesla V100 GPU such that a single test runtime is 26 seconds.
In \tableref{table:cifar100} the various methods are compared.
We outperform the others in the AUC metric while maintaining the baseline performance on the closed-set task.
The Entropic Open-Set, Cosine Similarity and Maximum Discrepancy methods have a reduction of 3\% and the Objectosphere approach accuracy decreases in 4\% with respect to the Max-prob method.
\Figref{fig:fully_train_last_layer:cifar100_lsun} shows performance on LSUN, our score outperforms the others for all FPR rates with a substantial improvement of 27\% and 20\% at FPR of 0.1 and 0.2, respectively.

The Objectosphere method, which is considered the state of the art classifier-based method in open-set recognition, performs poorly with LSUN, SVHN and Noise images. 
The reason might be the use of data with unknown classes during the training phase.
The performance of this method depends on the OOD train data and its similarity to the unexpected test samples.
We train the Objectosphere model with CelebA as images with unknown class, which is, in general, a set of images with low variation and is very different from noise images.
This kind of training exposes the weakness of the method as opposed to approaches that do not use OOD images during training.

\begin{table}[tb]
    \centering
    \caption{ResNeXt with CelebA trainset experimental results \label{table:celeba}}
    \small
    \begin{tabular}{ccccccccc}
\toprule
\multirow{2}{*}{$\begin{array}{c}\hbox{\bf Unknowns} \\\end{array}$}&
\multirow{2}{*}{\bf Algorithm} & 
\multirow{2}{*}{ $\begin{array}{c}\hbox{\bf AUC}\end{array}$} &
\multicolumn{5}{c}{\bf CCR at FPR of} \\
& & & {$0.1$} & {$0.2$} &  {$0.6$} & {$0.8$} & {$1.0$}
\\ 
\hline
\multirow{6}{*}{LSUN}  &         
    Max-prob           &  0.37 &  0.14 &  0.21 &  0.43 &  0.52 &  0.77 \\
&   ODIN               &  0.37 &  0.14 &  0.21 &  0.43 &  0.52 &  0.77 \\
&   Entropic Open-Set  &  0.38 &  0.13 &  0.19 &  0.46 &  0.57 &  0.76 \\
&   Objectosphere      &  0.21 &  0.02 &  0.04 &  0.17 &  0.39 &  0.78 \\
&   Cosine Similarity  &  0.37 &  0.11 &  0.19 &  0.45 &  0.54 &  0.76 \\
&   Max. Discrepancy        &  0.56 &  0.35 &  0.46 &  0.63 &  0.67 &  0.78 \\
&   Ours               &  \textbf{0.65} &  \textbf{0.44} &  \textbf{0.57} &  \textbf{0.73} &  \textbf{0.76} &  \textbf{0.78} \\
\cline{1-8}
\multirow{6}{*}{SVHN}  &         
    Max-prob           &  0.47 &  0.27 &  0.38 &  0.54 &  0.58 &  0.77 \\
&   ODIN               &  0.48 &  0.27 &  0.38 &  0.54 &  0.58 &  0.77 \\
&   Entropic Open-Set  &  0.54 &  0.40 &  0.46 &  0.58 &  0.64 &  0.77 \\
&   Objectosphere      &  0.43 &  0.12 &  0.20 &  0.52 &  0.64 &  0.78 \\
&   Cosine Similarity  &  0.44 &  0.10 &  0.21 &  0.54 &  0.63 &  0.77 \\
&   Max. Discrepancy        &  0.64 &  0.49 &  0.57 &  0.69 &  0.72 &  0.78 \\
&   Ours               &  \textbf{0.68} &  \textbf{0.50} &  \textbf{0.61} &  \textbf{0.75} &  \textbf{0.78} &  \textbf{0.78} \\
\cline{1-8}
\multirow{6}{*}{Noise}& 
    Max-prob           &  0.44 &  0.29 &  0.34 &  0.48 &  0.52 &  0.77 \\
&   ODIN               &  0.44 &  0.29 &  0.34 &  0.48 &  0.52 &  0.77 \\
&   Entropic Open-Set  &  0.39 &  0.30 &  0.32 &  0.41 &  0.46 &  0.77 \\
&   Objectosphere      &  \textbf{0.77} &  \textbf{0.77} &  \textbf{0.77} &  \textbf{0.78} &  \textbf{0.78} &  \textbf{0.78} \\
&   Cosine Similarity  &  0.57 &  0.53 &  0.55 &  0.58 &  0.60 &  0.77 \\
&   Max. Discrepancy        &  0.64 &  0.47 &  0.59 &  0.70 &  0.73 &  0.78 \\
&   Ours               &  0.65 &  0.47 &  0.57 &  0.72 &  0.76 &  0.78 \\
\bottomrule
\end{tabular}
\end{table}

In the last experiment, we perform an evaluation using ResNeXt with depth factor 20 and CelebA as the dataset with known classes.
The Entropic Open-Set, Objectosphere and Maximum Discrepancy methods use also MNIST~\cite{lecun-mnisthandwrittendigit-2010} during training to simulate OOD inputs.
The ODIN's adversarial perturbation strength was optimized on MNIST too.
During our method's fine-tuning stage, we train one fully connected layer from a feature space size of 1,024 to two outputs, a total of 2,048 parameters, from random initialization.

The CCR versus FPR of using CelebA for closed-set and images from LSUN, SVHN and Noise as inputs with unknowns is presented in \tableref{table:celeba}. 
Our method provides the highest AUC value with respect to the compared approaches on the LSUN and SVHN sets, with an improvement of 9\% and 4\% respectively over the next best one.
\Figref{fig:fully_train_last_layer:celeba_lsun} shows the CCR vs FPR curve for LSUN dataset. We outperform all the competitors for all FPRs with a solid margin of about 10\% over the Maximum Discrepancy method.

On the Noise benchmark, the Objectosphere has the best AUC.
It may be related to the OOD training set that is used: The OOD trainset is MNIST, which might resemble noise images. In the previous experiments, with more challenging datasets and with higher capacity models this advantage disappeared.

\subsection{Ablation study} \label{sec:fully_train_last_layer:ablation_study}
In this section, we explore the effect of the hypothesis class size on the classification accuracy and the regret based confidence score.
In the previous sections, when executing our procedure, we use a hypothesis class that contains only one fully connected layer.
We now analyze the effect of varying the number of layers that are trained during the fine-tuning phase.

We perform the ablation study with the ResNet18 architecture~\cite{he2016deep} on the CIFAR10 dataset.
Initial training consisted of 200 epochs using stochastic gradient descent optimizer with a batch size of 128 and a learning rate of 0.1, with a decrease to 0.01 and 0.001 after 100 and 150 epochs, respectively. Fine-tuning consisted of 10 epochs with a learning rate of 0.001.

The first hypothesis class we consider is a class in which we do not train any layer during the fine-tuning phase.
The hypothesis class is simply the DNN obtained after the initial training, which is the ERM learner. We name this learner \emph{0 Layers}.
The second class is where the fine-tuning alters only the last two layers, i.e., \emph{2 Layers}
and \emph{7 Layers} is when the fine-tuning phase affects all.

The comparison between the hypothesis classes is summarized in \tableref{table:ablation_study}. 
We can see a 0.03 improvement in the average log-loss with a slightly better accuracy rate of the 2 Layers learner with respect to the basic DNN. 
More notable is the 0.57 improvement in the standard deviation of the log-loss, which suggests that the learner manages to avoid large loss values and it is indeed better in the worst case sense. 
This property of the pNML can also be observed by the log-loss histogram presented in \figref{fig:fully_train_last_layer:resnet18_loss_hist} in which 0 Layers and 2 Layers have a large number of samples with loss value higher than 0.9. 
When increasing the hypothesis class size, the worst-case loss is indeed reduced.
However, as stated in \tableref{table:ablation_study} and \figref{fig:fully_train_last_layer:resnet18_regret_hist}, the regret increases as well.

This coincides with the theoretical interpretation of the regret as ``insurance'' the pNML pays to guarantee that the log-loss will not explode no matter what the unknown test label is.
Increasing the hypothesis class too much, as shown in \tableref{table:ablation_study} as the difference between 0 Layers and 7 Layers, reduces the learner accuracy rate and increases the average log-loss of the test set.
This observation indicates the hypothesis class needs to be designed carefully.

\begin{table}[tb]
\centering
\caption{The performance as function of the hypothesis class size \label{table:ablation_study}}
\small
\begin{tabular}{cccccc}
\toprule
\# Tuned Layers \ \ \ & Acc.  & \ \ Loss avg \ &  \ \ Loss STD \ \ & Regret\\
\hline
0 layers    & 0.920 &  0.203 & 0.87 & 0.0   \\
2 layers    & 0.921 &  0.173 & 0.30 & 0.14  \\
7 layers    & 0.913 &  0.244 & 0.27 & 0.23  \\
\bottomrule
\end{tabular}
\end{table}

\begin{figure}[tb]
\begin{center}
    \begin{subfigure}[t]{0.49\textwidth}
    \centering
    \centerline{\includegraphics[width=\columnwidth]{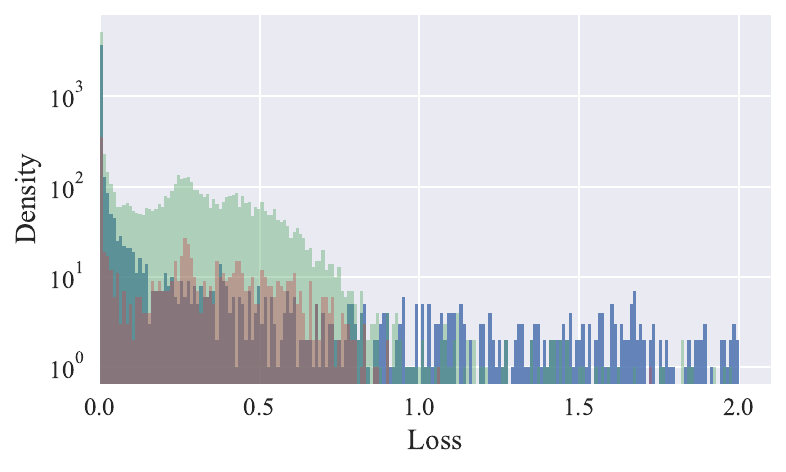}}
    \caption{Log-loss histogram \label{fig:fully_train_last_layer:resnet18_loss_hist}}
    \end{subfigure}
    \begin{subfigure}[t]{0.49\textwidth}
    \centering
    \centerline{\includegraphics[width=\columnwidth]{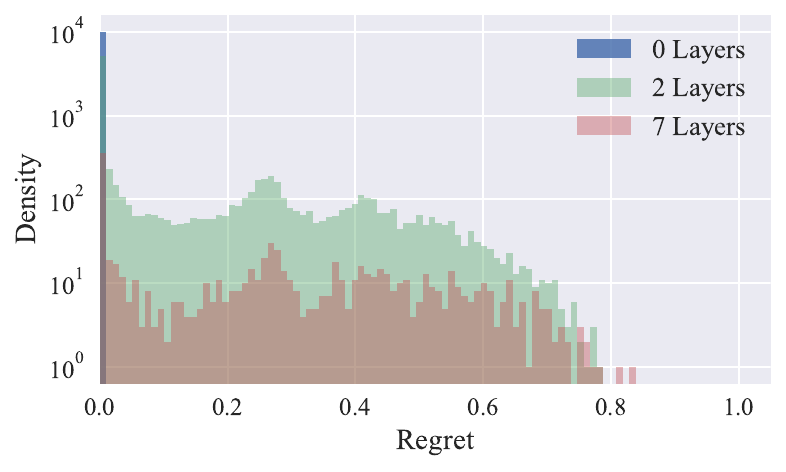}}
    \caption{Regret histogram \label{fig:fully_train_last_layer:resnet18_regret_hist}}
    \end{subfigure}
\caption{The pNML performance for the CIFAR10 set of different amount of trained layers}
\end{center}
\end{figure}

\section{Concluding remarks} 
\label{sec:dnn:conclusion}
We presented a novel method for detecting test instances from unknown classes while maintaining performance on closed-set data. The universal proposed scheme is inspired by the recently suggested predictive normalized maximum likelihood method for the individual setting, with respect to the log-loss function. 
We showed that the suggested pNML regret can be used successfully as a confidence measure for identifying unexpected inputs.

One could assume that when executing the procedure using the DNN hypothesis class, the model would fit exactly to each possible test label and, therefore, using the regret as a generalization measure might be useless.
We show that by effectively reducing the fitted model capacity, this is not the case and that the regret is an informative generalization measure where others fail.

\chapter{Adversarial Defense} \label{chap:adversarial_defense}


\section{Introduction}
DNNs have shown state-of-the-art performance on a wide range of complex problems~\citep{goodfellow2016deep}.
Despite the impressive performance, it has been found that DNNs are susceptible to adversarial attacks \citep{biggio2013evasion,szegedy2013intriguing}. 
These attacks cause the model to under-perform by adding specially crafted noise to the input, such that the original and modified inputs are almost indistinguishable. 

In the case of an image classification task, adversarial attacks can be categorized according to the threat model properties \citep{carlini2019evaluating}: its goal, capabilities and knowledge.
(i) The adversary goal may be to simply cause misclassification, i.e., untargeted attack. 
Alternatively, the goal can be to have the model misclassify certain examples from a source class into a target class of its choice. This is referred to as a targeted attack.
(ii) The adversary capabilities relate to the strength of the perturbation the adversary is allowed to cause the data, e.g., the distance between the original sample and the adversarial sample under a certain distance metric must be smaller than $\epsilon$.
(iii) The adversary knowledge represents what knowledge the adversary has about the model. 
This can range from a full white-box scenario, where the adversary has complete knowledge of the model architecture, its parameters, and the defense mechanism, to a black-box scenario where the adversary does not know the model and has only a limited number of queries on it. 

One of the simplest attacks is Fast Gradient Sign Method (FGSM) \citep{goodfellow2014explaining}. 
Let $w_0$ be the parameters of the trained model, $x$ be the test data, $y$ its corresponding label, $L$ the loss function of the model, $x_{adv}$ the adversary input and $\epsilon$ specifies the maximum $l_{\infty}$ distortion such that $||x-x_{\textit{adv}}||_{\infty}\leq \eps$. 
First, the signs of the loss function gradients are computed with respect to the image pixels,
Then, after multiplying the signs by $\epsilon$, they are added to the original image to create an adversary untargeted attack
\begin{equation} \label{eq:adversarial_defense:fgsm_untargeted}
x_{adv} = x + \epsilon \cdot \textit{sign}\nabla_x L(w_0, x, y_\textit{true}).
\end{equation}
It is also possible to improve classification chance for a certain label $y_\textit{target}$ (targeted attack):
\begin{equation} \label{eq:adversarial_defense:fgsm_targeted}
x_{adv} = x - \epsilon \cdot  \textit{sign}\nabla_x L(w_0, x,y_\textit{target}).
\end{equation}
A multi-step variant of FGSM that was used by \citet{madry2017towards} is called Projected Gradient Descent (PGD) and is considered to be one of the strongest attacks. 
Denote $\alpha$ as the size of the update, for each iteration an FGSM step is executed:
\begin{equation} \label{eq:adversarial_defense:pgd_untargeted}
x_{\textit{adv}}^{t + 1} = x_{\textit{adv}}^{t} + \alpha \cdot \textit{sign}(\nabla_x L(w_0, x,y_\textit{true}), \quad  0 \leq t \leq T.
\end{equation}
The number of iterations $T$ is predetermined. 
For each sample, the PGD attack creates multiple different adversarial samples by randomly choosing multiple different starting points $x_{\textit{adv}}^{0} = x + u$ where $u \sim U[-\epsilon,\epsilon]$.
Then the sample with the highest loss is chosen.

A different approach is taken by Hop-Skip-Jump-Attack (HSJA) \citep{chen2020hopskipjumpattack} black-box attack, where the adversary has only a limited number of queries to the model decision. 
The attack is iterative, each iteration involves three steps:
estimating the gradient direction, step-size search via geometric progression, and boundary search via a binary search.
 
The most prominent defense is adversarial training, which augments the training set to include adversarial examples~\citep{goodfellow2014explaining}. Many improvements in adversarial training were suggested. \citet{madry2017towards} showed that training with PGD adversaries offered robustness against a wide range of attacks. \citet{carmon2019unlabeled} suggested using semi-supervised learning with unlabeled data to further improve robustness. \citet{Wong2020Fast} offered a way to train a robust model with a much lower computational cost by using weak adversaries. 
An alternative to adversarial training is to encourage the model loss surface to become linear, so small changes at the input would not change the output greatly. \citet{qin2019adversarial} demonstrated how using a local linear regularizer during training creates a robust DNN model.
However, current methods are still unable to achieve a robust model for large scale datasets and high dimensional inputs.

The pNML scheme gives the optimal solution for a min-max game where the goal is to be as close as possible to a reference learner, who knows the true label but is restricted to use a learner from a given hypothesis class. 
In the pNML setting and derivation, there is no assumption on the way the data is generated. Both the training set and the test sample have some specific individual values. This matches the adversarial attack scenario, where the input is manipulated specifically to deceit the model.

In this chapter, we propose the \textit{Adversarial pNML} scheme as a new adversarial defense. 
Based on the pNML, we restrict the genie learner, a learner who knows the true test label, to perform minor refinements to the adversarial examples.
Our method uses an existing adversarial trained model.
We compose a hypothesis class: 
Each member in the class assumes a different label for the adversarial test sample and performs a simple targeted adversarial attack based on the assumed label.
Finally, by comparing the resulting hypotheses probabilities, predict the true label.

Contrary to existing methods that attempt to remove the adversary perturbation~\citep{guo2018countering,samangouei2018defensegan,song2018pixeldefend}, our method is unique in the sense it does not remove the perturbation but rather exploits the adversarial subspace properties, as will be elaborated below.

\section{pNML for adversarial defense} \label{sec:adversarial_defense:adversarial_pNML}
In the individual setting of universal learning, the relation between the data and labels is arbitrary and can be determined by an adversary.
The pNML is the min-max solution of universal learning in the individual setting and therefore using it as a defense against adversary attacks seems suitable.

Recall the genie solution
\begin{equation} \label{eq:adversarial_defenseLgenie_pnml_org} 
\hat{\theta}(z^N,x,y)  = \arg\max_\theta \left[ p_\theta(y|x) \cdot\Pi_{i=1}^N p_\theta(y_i|x_i) \right]
\end{equation}
This genie is the best learner one can attain when knowing the test label given the hypothesis class $\Theta$.
The main challenge is how to choose the hypothesis class. 
A good hypothesis class should be large enough such that the genie could achieve good performance on natural and adversarial data. However, it should not be too large to avoid overfitting of the other class members to ensure that the overall regret stays small.

We propose a novel hypothesis class by adding a refinement stage before a pretrained DNN model $w_0$. 
The refinement stage alters the test sample by performing a targeted attack based on the DNN model and a certain label $y$. Denote $\lambda$ as the refinement strength, the refined sample is:
\begin{equation} \label{eq:adversarial_defense:_x_refine}
x_\textit{refine}(x,y) = x - \lambda \cdot \textit{sign}(\nabla_xL(w_0,x,y)).
\end{equation}
The  hypothesis class we consider is therefore 
\begin{equation} \label{eq:adversarial_defense:hypo_class_adv}
\Theta = \left\{ p_{w_0}(\cdot|x_\textit{refine}(x,y)), \quad \forall y \in Y \right\}.
\end{equation}
Each member in the hypothesis class produces a probability assignment by adding a perturbation that strengthen one of the possible value of the test label. 


Our Adversarial pNML scheme consists of the following steps:
At first, we train a DNN model $w_0$ with adversarial training. 
Then, we produce the hypothesis class -  
given a test data $x$, we refine it using the trained DNN $w_0$ and an arbitrary test label $y'$ \eqref{eq:adversarial_defense:_x_refine}.
We save the label probability we refined with by feeding the refined test to the trained DNN
\begin{equation}
p_{w_0}(y'|x_\textit{refine}(x,y')).
\end{equation}
This action is repeated for every possible test label value.
At the end of the process we get a set of predictions, we normalize them and return the Adversarial pNML probability assignment (prediction)
\begin{equation}
q_\textit{pNML}(y'|x) = \frac{1}{K} p_{w_0}(y'|x_\textit{refine}(x,y')), \quad K=\sum_{y' \in Y} p_{w_0}(y'|x_\textit{refine}(x,y')).
\end{equation}
The corresponding regret is the log normalization factor and is $\Gamma = \log K$.

\section{Adversarial subspace interpretation}
\label{subsec:adversarial_defense:adversarial_subspace_interpretation}

Let $x_\textit{adv}$ be a strong adversarial example with respect to the label $y_\textit{target}$, i.e.,
the DNN model has a high probability of mistakenly classifying $x_\textit{adv}$ as $y_\textit{target}$.
For the task of binary classification there is two kinds of members in our suggested hypothesis class:
refinement towards the true label $y_\textit{true}$ and refinement towards the adversary target $y_\textit{target}$. 
There are two mechanisms that cause the refinement towards the true label to be stronger than the refinement towards the adversary target: the convergence to local maxima and refinement overshoot.

\paragraph{Convergence to local maxima.} 
\citet{szegedy2013intriguing} stated that adversarial examples represent low-probability pockets in the manifold which are hard to find by randomly sampling around the given sample.
\citet{madry2017towards} showed that FGSM often fails to find an adversarial example while PGD with a small step size succeeds. This implies that for some dimensions the loss local maxima is in the interval $[-\epsilon,\epsilon]$, which is also confirmed empirically for CIFAR10 by \citet{Wong2020Fast}.
Following these results, we conclude that for some dimensions the loss local maxima can be viewed as a ``hole'' in the probability manifold.
For those dimensions, refinement towards $y_{target}$ would not increase the probability of $y_{target}$ hypothesis since $x_{adv}$ already converged to the local maximum. On the other hand, refinement towards $y_{true}$ could cause the refined sample to escape the local maximum hole, thus increasing the probability of $y_{true}$ hypothesis.

\paragraph{Refinement overshoot.}
\label{refinement_overshoot}
PGD attack is able to converge to strong adversarial points by using multiple iterations with a small step size. This process avoids the main FGSM pitfall: As the perturbation size increase, the gradient direction change \citep{madry2017towards}, causing FGSM to move in the wrong direction and overshoot.
For the same reason, the FGSM refinement towards the $y_{target}$ might fail to create a strong adversarial. 
On the other hand, the refinement towards $y_{true}$ is more probable to succeed since the volume of the non-adversarial subspace is relatively large, thus a crude FGSM refinement is more likely to move in the right direction. To support that claim we note that the adversarial subspace has a low probability and is less stable compared to the true data subspace \citep{tabacof2016exploring}.
In other words, while the true hypothesis escapes the adversarial subspace, the target hypothesis can transform the strong PGD adversarial into weak FGSM adversarial. 

In the case of multi-label classification there is a third kind hypothesis: A refinement towards other label $y \not \in \{y_\textit{true},y_\textit{target}\}$.
This refinement effectively applies a weak targeted attack towards a specific label $y$. 
This hypothesis can be neglected for a strong adversarial input because a weak refinement towards other labels is unlikely to become more probable than refinement towards the target label.

\begin{figure}[bt]
\centering
\begin{subfigure}[b]{0.49\linewidth}
    \includegraphics[width=1.0\linewidth]{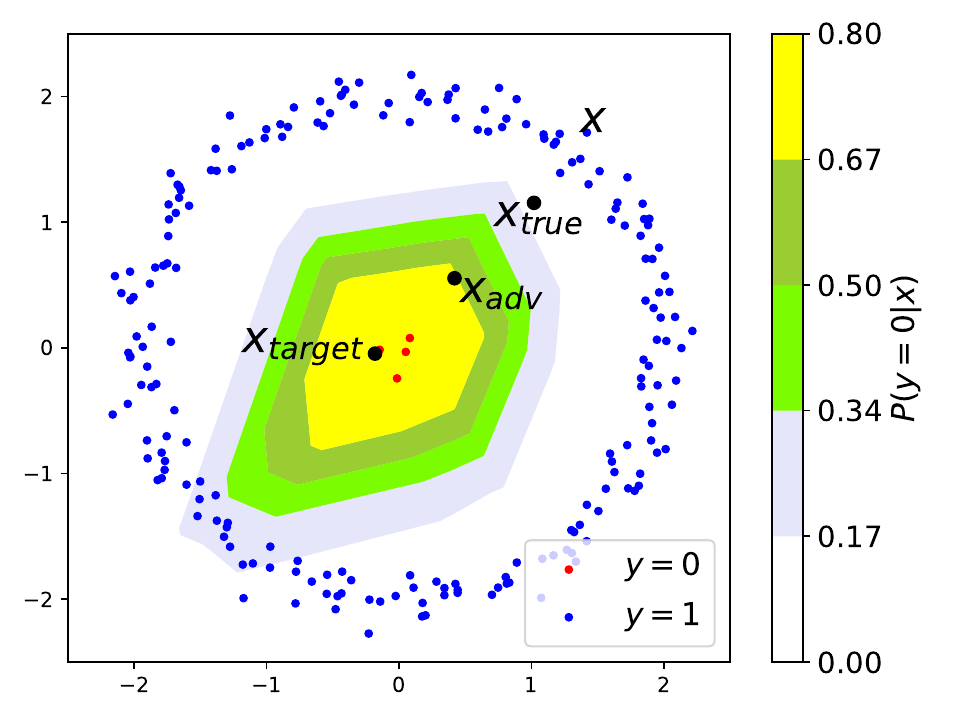}
    \caption{Convergence to local maxima}
    \label{fig:adversarial_defense:toy_ds_local_maxima}
\end{subfigure}
\hfill
\begin{subfigure}[b]{0.49\linewidth}
    \includegraphics[width=1.0\linewidth]{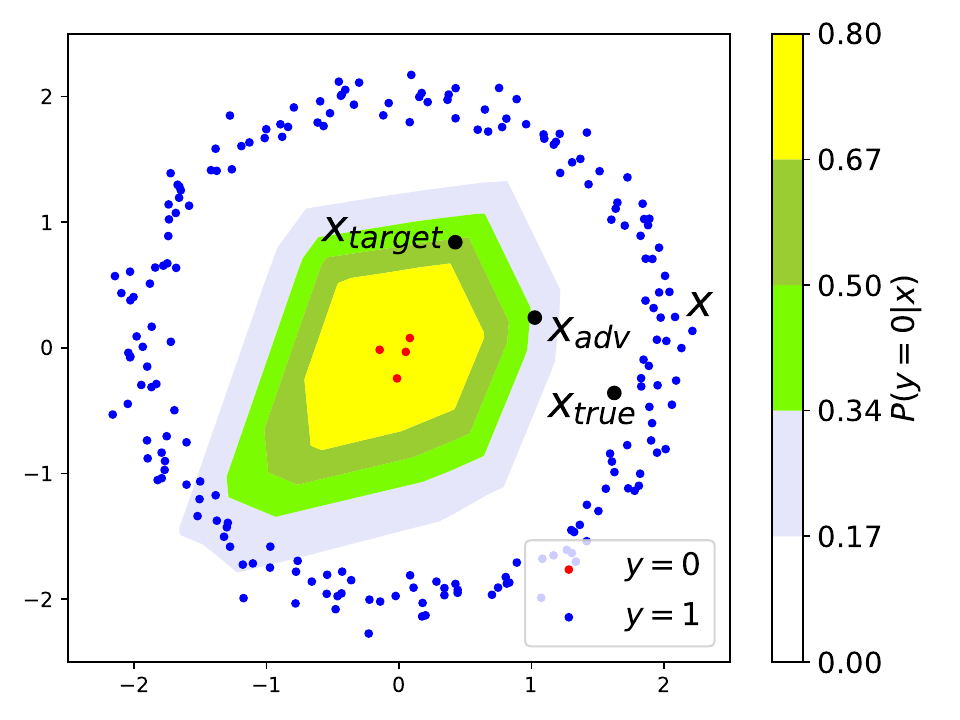}
    \caption{Refinement overshoot}
    \label{fig:adversarial_defense:toy_ds_crude_refine}
\end{subfigure}
 \caption{
 Label "$0$" probability manifold 
 }
 \label{fig:adversarial_defense:toy_ds_demo}
\end{figure}

To demonstrate the above mechanisms we present an experiment with two-dimensional synthetic data.
Let $\rho_0 \sim \mathcal{N}(0,0.01I)$ be the distribution with label $0$ and denote $\rho_1 \sim \mathcal{N}(M,0.01I)$ as the distribution of the data the corresponds to label 1, where $M$ is a random variable uniformly distributed on a circle of radius 2. 
We train a simple 4 fully connected layer classifier using adversarial training set generated by a PGD attack with $4$ iterations of size $0.25$ and $\eps=0.5$. 
For the adversarial testset, we increase $\epsilon$ to $0.95$ and for the refinement we choose $\lambda$ value of 0.6.

\Figref{fig:adversarial_defense:toy_ds_demo} shows the refinement process on top of the trained model label $0$ probability manifold. $x$ is the original sample with label $1$, $x_{adv}$ is the test adversarial sample, $x_{true}$ is the sample generated by refinement towards label $1$, and $x_{target}$ is generated by refinement towards label $0$. Using \eqref{eq:adversarial_defense:_x_refine}:
\begin{equation}
x_\textit{true} = x_{\textit{refined}}(x_\textit{adv},1), \quad 
x_\textit{target} = x_{\textit{refined}}(x_\textit{adv},0). 
\end{equation}

\Figref{fig:adversarial_defense:toy_ds_local_maxima} demonstrates the convergence to local maxima mechanism.
$x_{adv}$ converged to the maximum probability, therefore refinement towards the target label does not increase the probability while refinement towards the true label does. As a result, the true hypothesis probability is greater and the true label is predicted.
\Figref{fig:adversarial_defense:toy_ds_crude_refine} presents the refinement overshoot mechanism. The target hypothesis is refined in the wrong direction while the true hypothesis is refined in the correct direction.

We note that our method could degrade the accuracy for natural, non-adversarial, samples since the refinement is a weak targeted attack. To mitigate this drawback we incorporate adversarial training. 
Also, we show empirically in \secref{sec:adversarial_defense:ablation_study} that replacing FGSM refinement with PGD refinement decreases robustness.

\section{Experiments}  \label{sec:adversarial_defense:experiments}
We evaluate the natural and adversarial performance on MNIST~\citep{lecun-mnisthandwrittendigit-2010}, CIFAR10~\citep{krizhevsky2014cifar} and ImageNet~\citep{deng2009imagenet} datasets. We compare our scheme to other state-of-the-art algorithms and to a standard model trained on natural samples.

\subsection{Adaptive attack and gradient masking} 
\label{subsec:adversarial_defense:adaptive_adversary}

A main part of the defense evaluation is creating and testing against adaptive adversaries that are aware of the defense mechanism~\citep{carlini2019evaluating}. This is specifically important when the defense cause gradient masking~\citep{papernot2017practical}, in which gradients are manipulated, thus prevent a gradient-based attack from succeeding. 
Defense aware adversaries can overcome this problem by using a black-box attack or by approximating the true gradients~\citep{athalye2018obfuscated}. 

We designed a defense aware adversary for our scheme: We create an end-to-end model that calculates all possible hypotheses in the same computational graph.
We note that the end-to-end model causes gradient masking since the refinement $sign(\cdot)$ function sets some of the gradients to zero during the backpropagation phase. 
We, therefore, attack the end-to-end model with the black-box HSJA method and PGD with Backward Pass Differentiable Approximation (BPDA) technique~\citep{athalye2018obfuscated}, denoted as an adaptive attack.
In BPDA we perform the usual forward-pass but on the backward pass, we replace the non-differentiable part with some differentiable approximation. Assuming the refinement is small, one solution is simply to approximate the refinement stage by the identity operator, $x_\textit{refine}\approx x$, which leads to $\frac{\partial x_\textit{refine}}{\partial x} \approx 1$.
Further discussion on the adaptive attack can be found in the appendix.

\begin{table}[bt]
\caption{Accuracy comparison for the different defenses}
\label{table:adversarial_defense:imagenet_results}
\label{table:adversarial_defense:cifar_results}
\label{table:adversarial_defense:mnist_results}
\small
\centering  
\begin{tabular}{c|c|c|cccc|c}
\toprule
Dataset & Method & Natural & FGSM  & PGD & Adaptive & HSJA  & Best attack \\
\toprule
   & Standard & 99.3\%  & \ 0.6\%  & 0.0\% & - & - & 0.0\%\\
MNIST   & \citet{madry2017towards} & 97.8\% & 95.4\%  & 91.2\% & 89.9\% & 93.1\% & 89.9\%\\
    $\epsilon=0.3$ & pNML & 97.8\%  & 95.4\%  & 93.7\% & 90.7\% & 94.6\% & \textbf{90.7\%} \\
\midrule 
      &  Standard & 93.6\%  & 6.1\% &  0.0\% & - & - & 0.0\% \\
    &  WRN & 84.3\%  & 48.2\%  &  37.4\% & - & - & 37.4\% \\
      & WRN + pNML  & 84.3\%  &  48.9\% & 45.7\% & - & - &  45.7\% \\ \
    CIFAR10 &  \citet{madry2017towards} & 87.3\%  & 56.1\%  &  45.8\% & - & - & 45.8\% \\
     $\epsilon=0.031$ &  \citet{qin2019adversarial} & 86.8\%  & - &  54.2\% & - & - & 54.2\% \\
     &  \citet{carmon2019unlabeled} & 89.7\%  & 69.9\%  &  63.1\% & - & 78.8\%  & 63.1\% \\
     & pNML & 88.1\%  & 69.5\%  & 67.3\% & 68.6\% & 84.8\% & 
     \textbf{67.3\%} \\
\midrule
     &  Standard & 83.5\% & 7.0\% & 0.0\% & - & - & 0.0\% \\
ImageNet & \citet{Wong2020Fast} & 69.1\%  & 27.0\%  &  16.0\% & -  & 68.0\% & 16.0\% \\
     $\epsilon=8/255$ & pNML & 69.3\% & 28.0\% & 20.0\% & 19.0\% & 68.0\% & \textbf{19.0\%} \\
 \midrule
 ImageNet & \citet{Wong2020Fast} & 69.1\%  & 44.3\%  &  42.9\% & -  & - & 42.9\% \\
     $\epsilon=4/255$ &  pNML & 69.3\% & 49.0\% & 48.6\% & - & - & \textbf{48.6\%} \\
\bottomrule
\end{tabular}
\end{table}

\subsection{Experimental results} \label{sec:adversarial_defense:experiment_results}

\paragraph{MNIST.} \label{subsubsec:adversarial_defense:mnist_experiments}
We follow the model architecture as described in \citet{madry2017towards}. 
We use a network that consists of two convolutional layers with 32 and 64 filters respectively, each followed by $2\times2$ max-pooling, and a fully connected layer of size 1024. 
We trained the model for 106 epochs with adversarial trainset that was produced by PGD based attack on the natural training set with 40 steps of size 0.01 with a maximal $\eps$ value of 0.3. We used SGD with a learning rate of $0.01$, momentum value 0.9 and weight decay of 0.0001. For the last 6 epochs, we used adversarial training with the adaptive attack instead of PGD. We set the Adversarial pNML refinement strength to $\lambda=0.1$.

For evaluation, we set $\epsilon$ to 0.3 for all attacks. The PGD attack was configured with 50 steps of size 0.01 and 20 random starts. For the adaptive attack, we used 300 steps of size 0.01 and 20 random starts. 
For HSJA, we set the number of model queries to $26K$ per sample, which was shown to be enough queries for convergence \citep{chen2020hopskipjumpattack}.

In \tableref{table:adversarial_defense:mnist_results} we report the accuracy of our scheme in comparison to adversarial trained model (\citet{madry2017towards}). 
We observe that our scheme improves the robustness by 0.8\% with no loss to natural accuracy. The adaptive attack manages to decrease our approach robustness which indicates that this kind of attack is efficient. 
In addition, our scheme improves black-box accuracy by 1.5\%.

\paragraph{CIFAR10.}
For CIFAR10 dataset, we build our scheme upon a pre-trained Wide-ResNet 28-10 architecture \citep{zagoruyko2016wide} that was trained with both labeled and unlabeled data \citep{carmon2019unlabeled}.
We set the Adversarial pNML refinement strength to $\lambda=0.03$. 
For evaluation, we set $\epsilon=0.031$ for all attacks. PGD and adaptive attack were configured with 20 steps of size 0.007 and without random starts. 
For HSJA attack, we set the maximal number of model queries to $26K$ per sample and evaluate the accuracy for the first 2K samples (out of 10K). 

In \tableref{table:adversarial_defense:cifar_results} we report the accuracy of our scheme in comparison to other state-of-the-art algorithms. 
We observe that our scheme achieves state-of-the-art robustness, enhancing the robustness by 4.2\% and only degrading natural accuracy by 1.6\% compared to the base model. In addition, our scheme improves black-box accuracy by 6.0\% which supports the claim that the robustness improvement is not only the result of masked gradients. We note that the adaptive attack is unsuccessful and it might imply that the loss surface of the end-to-end model for CIFAR10 is non-linear and non-smooth degrading the gradient optimization. Against FGSM attack our scheme performs worse than the base model of \citet{carmon2019unlabeled} by 0.4\%. This result, together with the improvement our scheme achieves against PGD attack demonstrates, convergence to local maxima mechanism. 
Our scheme works best when the adversarial sample converged to a loss local maxima, as described in \secref{subsec:adversarial_defense:adversarial_subspace_interpretation}.  

We incorporate our suggested Adversarial pNML method with PGD based adversarial trained model as suggested by \citet{madry2017towards} and named it \emph{WRN + ours}.
In \tableref{table:adversarial_defense:cifar_results}, one can see that our scheme enhances the robustness of the WRN model by 8\% with no loss in natural accuracy. 

In \Figref{fig:adversarial_defense:cifar_acc_vs_attack_strength} we explore the robustness of our scheme against PGD attack. The results show that our scheme is more robust for all $\epsilon$ values greater than 0.01. 
Specifically, the maximal improvement is 8.8\% for $\epsilon$ that equals 0.05.
For $\epsilon=0.01$ our scheme is less robust by 0.6\%. To explain these results, recall that the refinement strength is 0.03. 
When $\eps<<\lambda$, one of the refinement hypothesis could generate adversarial examples stronger than the examples generated by the adversarial attack.

\begin{figure}[b]
\centering
\begin{subfigure}[b]{0.49\linewidth}
    \includegraphics[width=1.0\linewidth]{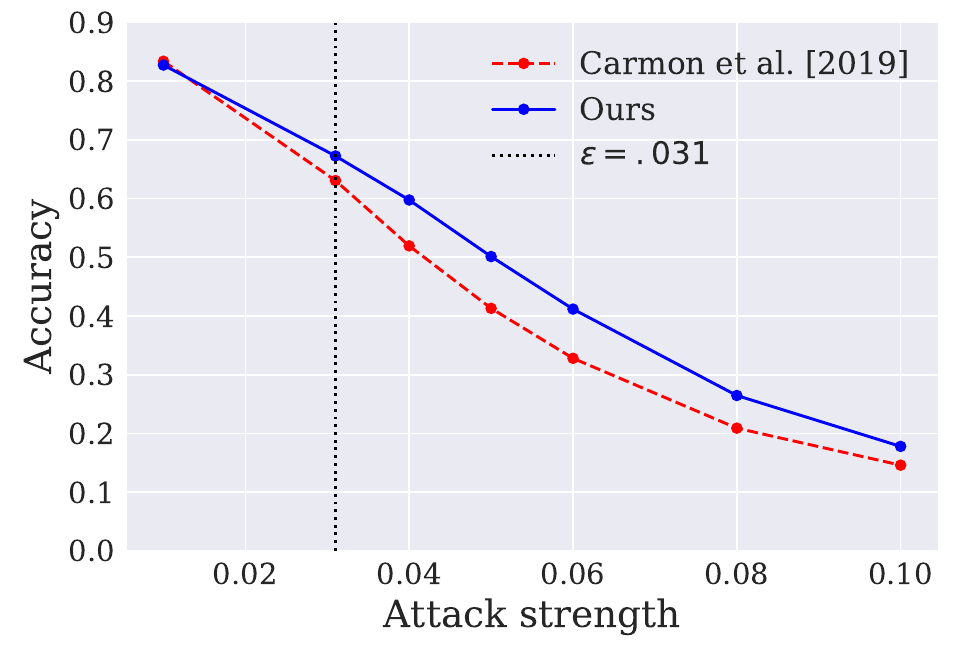}
    \caption{CIFAR10}
    \label{fig:adversarial_defense:cifar_acc_vs_attack_strength}
\end{subfigure}
\hfill
\begin{subfigure}[b]{0.49\linewidth}
    \includegraphics[width=1.0\linewidth]{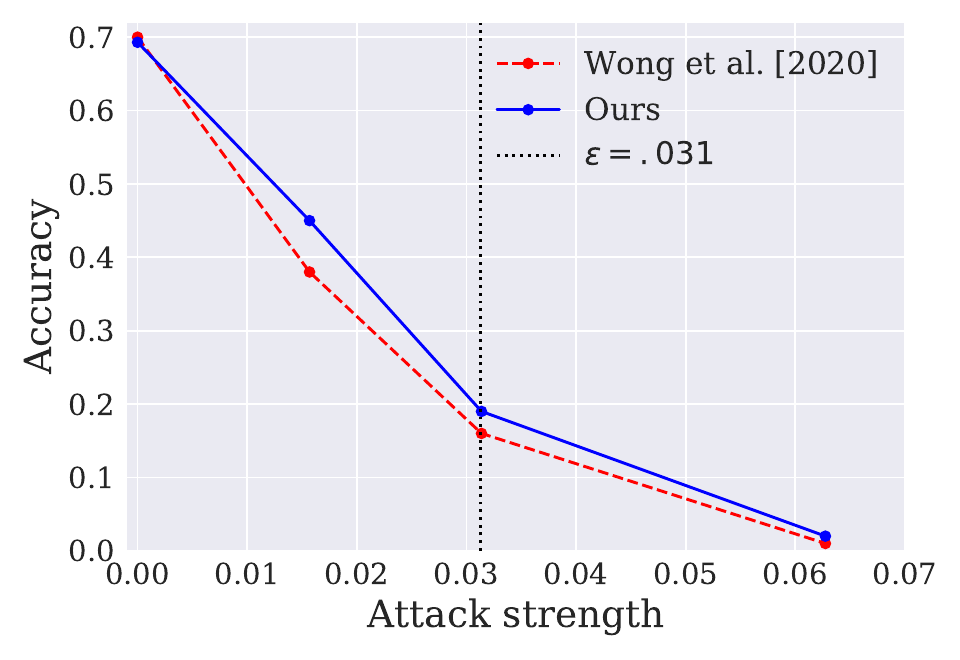}
    \caption{ImageNet}
    \label{fig:adversarial_defense:subset_imagenet_adaptive_acc}
\end{subfigure}
\caption{Robustness for different $\epsilon$ values}
\label{fig:adversarial_defense:acc_vs_attack_strength}
\end{figure}

\paragraph{ImageNet.}
We based our Adversarial pNML approach on a pre-trained ResNet50 architecture that was trained using fast adversarial training \citep{Wong2020Fast} with $\epsilon=4/255$. 
We set the Adversarial pNML refinement strength to $\lambda=3/255$. 
We used a subset of the evaluation set containing only the first 100 labels for a total of 5000 samples and we adjusted the model to only output the first 100 logits. 
PGD and adaptive attack were configured with 50 steps of size $1/255$ and with 10 random restarts. For HSJA, we set the number of model queries to $12K$ per sample. 

In \tableref{table:adversarial_defense:imagenet_results} we report the accuracy of our scheme in comparison to the base model (\citet{Wong2020Fast}) for $\epsilon=8/255$ and $\epsilon=4/255$.
For $\epsilon=8/255$ we evaluated all attacks using 100 samples (1 sample per label) and for $\epsilon=4/255$ we evaluated PGD and FGSM attacks using 5000 samples.
We observe that robustness is improved by 3\% and 5.7\% for $\epsilon=8/255$ and $\epsilon=4/255$ respectively and that natural accuracy is improved by 0.2\%. 
The HSJA attack seems to fail in finding adversarial examples, and that for $\epsilon=8/255$ the adaptive attack is the best attack. 

In \Figref{fig:adversarial_defense:subset_imagenet_adaptive_acc} we explore the robustness of our scheme against the adaptive attack. The results show that our scheme is more robust for all $\epsilon\geq0.016$, specifically the maximal improvement is 7.0\% for $\epsilon=0.016$.

\begin{figure}[tb]
\centering
  \includegraphics[width=0.65\textwidth]{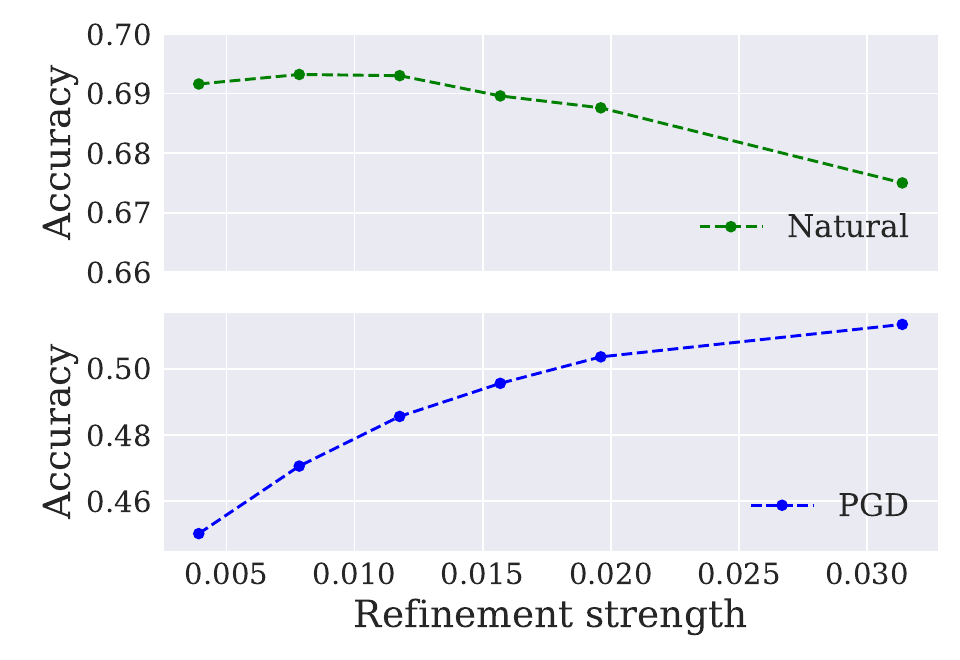}
  \caption{ImageNet accuracy vs the pNML refinement strength}
  \label{fig:adversarial_defense:imagenet_acc_vs_lambda}
\end{figure}

\begin{table}[tb]
\caption{CIFAR10 accuracy for different refinement parameters}
\small
\centering
\begin{tabular}{c|c|c}
\toprule
Accuracy & Iterations & Step size \\
\toprule
67.26\% & 1 & 0.03 \\
66.89\% & 2 & 0.015 \\
66.86\% & 4 & 0.01 \\
\bottomrule
\end{tabular}
\label{table:adversarial_defense:refinement_overshoot}
\end{table}

\subsection{Ablation study} \label{sec:adversarial_defense:ablation_study}
Choosing the refinement strength represents a trade-off between increasing the robustness against adversarial attacks and decreasing the accuracy of natural samples. 
This trade-off is explored in \Figref{fig:adversarial_defense:imagenet_acc_vs_lambda}.
As the refinement strength increase so is the robustness to PGD attack at the price of a small accuracy loss for the natural samples. 
Denote $\epsilon$ as the strength used during the adversarial training,
from multiple experiments conducted on multiple models, we conclude that usually, a refinement strength in the interval $[0.5\epsilon,\epsilon]$ would give a good balance between natural and adversarial accuracy. 

In \tableref{table:adversarial_defense:refinement_overshoot} we explore the overshoot mechanism (see \secref{subsec:adversarial_defense:adversarial_subspace_interpretation}). We adjust the refinement to become more precise by replacing FGSM refinement with a PGD refinement which uses more iterations and smaller step size. We test our method on CIFAR10 against PGD attack with the same settings described in \secref{sec:adversarial_defense:experiment_results}. The results show that as the refinement becomes more precise, CIFAR10 accuracy decreases. This result demonstrates empirically that the overshoot mechanism improves robustness. This provides further support to the claim of the instability of the adversarial subspace, which explains why FGSM refinement towards $y_{true}$ is more likely to succeed compared to refinement towards the $y_{target}$. 



\section{Concluding Remarks}
In this chapter, we presented the Adversarial pNML scheme for defending DNNs from adversarial attacks. We justified the scheme by considering the properties of the adversarial subspace. 
We showed empirically that our method increases robustness against adversarial attacks for MNIST, CIFAR10 and ImageNet datasets. The scheme is conceptually simple, requires only one hyper-parameter and flexible since it allows a trade-off between robustness and natural accuracy. Furthermore, any model can easily combine our scheme to enhance model robustness.

This work suggests several potential directions for future work. 
First, the pNML regret can form an adversarial attack detector.
Second, exploring other hypothesis classes. For instance, the entire model parameter class where the model weights are changed according to the different hypotheses.

\chapter{Active Learning} \label{chap:active_learning}  
\section{Introduction}
In supervised learning, a training set is provided to a learner, which optimizes its parameters by minimizing the error on this set.
The process of creating this training set requires annotation, where an expert labels the data points. This is a time-consuming and costly process, which results in only a small subset of the data being labeled which may not represent the true underlying model~\cite{ren2021survey}. 
Active learning, on the other hand, allows the learner to interact with a labeling expert by sequentially selecting samples for the expert to label based on previously observed labeled data. Therefore, reducing the number of examples needed to achieve a given accuracy level~\cite{wang2016cost}.

Recent strategies ~\cite{gal2017deep,houlsby2011bayesian,sener2017active,shayovitz2021universal}  aim to quantify the uncertainties of samples from the unlabeled pool and utilize them to select a sample for annotation. Their underlying assumptions are that the distribution of the unlabeled pool and the test set are similar and that the data follows some parametric distribution. However, this may not always be true, particularly in privacy-sensitive applications where real user data cannot be annotated~\cite{alabduljabbar2021tldr} and the unlabeled pool may contain irrelevant information. In such cases, choosing samples from the unlabeled pool may not necessarily improve performance on the test set.



As an alternative to making distributional assumptions, we build upon the \textit{individual setting}~\cite{merhav1998universal} and the pNML learner as its min-max regret solution.
In this section, we derive an active learning criterion that takes into account a trained model, the unlabeled pool, and the unlabeled test features. The criterion is designed to select a sample to be labeled in such a way that, when added to the training set with its worst-case label, it attains the minimal pNML regret for the test set. Additionally, we provide an approximate version of this criterion that enables faster inference for deep neural networks (DNNs).

We demonstrate that in the existence of out-of-distribution (OOD) samples, for the same accuracy level our criterion needs 66.2\%, 91.9\%, and 77.2\% of labeled samples compared to recent leading methods for CIFAR10~\cite{krizhevsky2014cifar}, EMNIST~\cite{cohen2017emnist}, and MNIST~\cite{deng2012mnist} datasets respectively. When we consider only in-distribution (IND) samples as labeling candidates, our approach needs 74.3\%, 89.0\%, and 80.9\% labeled samples on the aforementioned datasets.

\subsection{Preliminaries and notations}
In this chapter, there is an iterative acquisition of samples in an online manner. Therefore we use slightly different notations than the previous chapters:
$x^n=(x_1,x_2,...,x_n)$ is a sequence of samples. The variables $x \in \mathbb{X}$ and $y \in \mathbb{Y}$ represent the features and labels respectively with $\mathbb{X}$ and $\mathbb{Y}$ being the sets containing the features and label's alphabet respectively.
In the supervised learning framework, a training set consisting of $n$ pairs of examples is provided to the learner
\begin{equation}
    z^n = \lbrace(x_i, y_i)\rbrace_{i=1}^n
\end{equation}
where $x_i$ is the $i$-th data point and $y_i$ is its corresponding label. 
The goal of a learner is to predict an unknown test label $y$ given its test data, $x$, by assigning a probability distribution $q\left(\cdot|x,z^n\right)$ for each training set $z^n$. 

Recent research has focused on obtaining a diverse set of samples for training deep learning models with reduced sampling bias. \citet{sener2017active} proposed a method for constructing core-sets by solving the k-center problem. However, the method's search procedure is computationally expensive, as it requires constructing a large distance matrix from unlabeled samples.

A widely used criterion for active learning is Bayesian Active Learning by Disagreement (BALD) which was originally proposed by \citet{houlsby2011bayesian}. 
This method finds the unlabeled sample $\hat{x}_i$ that maximizes the mutual information between the model parameters $\theta$ and the labeling candidates $x_i$ 
given the training set $z^{n-1} $ \footnote{We use the notation $I(X;Y|z)$ to denote the mutual information between the random variables X and the random variable Y conditioned on a realization z of a random variable Z.}: $$\hat{x}_i = \argmax_{x_i}{I(\theta;Y_i|x_i,z^{n-1})}$$ 
The idea is to minimize the uncertainty about model parameters using Shannon's entropy. This criterion also appears as an upper bound on information based complexity of stochastic optimization \cite{raginsky2011information} and also for experimental design  \cite{mackay1992information,fedorov2013theory}. 
This approach was empirically investigated by \citet{gal2017deep}, where a Bayesian method for deep learning was proposed and several heuristic active learning acquisition functions were explored within this framework. 

BALD, however, has a fundamental disadvantage if the test distribution differs from the training set distribution since what is maximally informative for model estimation may not be maximally informative for test time prediction. 

\citet{shayovitz2021universal} derived a criterion named Universal Active Learning (UAL) that takes into account the unlabeled test set when optimizing the training set:
\begin{equation}
    \hat{x}_i = \argmin_{x_i}{I(\theta;Y|X,x_i,Y_t,z^{n-1})}
\end{equation}
where $X$ and $Y$ are the test feature and label random variables.
UAL is derived from a Capacity-Redundancy theorem \cite{Shayovitz2019} and implicitly optimizes an exploration-exploitation trade-off in feature selection.
However, UAL and BALD assume that the data is generated according to a distribution which belongs to a given parametric hypothesis class. This assumption cannot be verified on real world data thus limiting its application.

\section{Active learning for individual data}\label{sec:active_learning:IAL}
In active learning, the learner sequentially selects data instances $x_i$ based on some criterion and produces $n$ training examples $z^n$. 
The objective is to select a subset of the unlabelled pool and derive a probabilistic learner $q\left(y|x,z^n\right)$ that attains the minimal prediction error among all training sets of the same size. 

Most selection criteria are based on uncertainty quantification of data instances in order to quantify their informativeness. However, in the individual setting, there is no natural uncertainty measure since there is no distribution $f\left(y|x\right)$ governing the data. 

In the previous chapters, we mentioned that the learning problem in the individual setting is defined as the log-loss difference between a learner \textbf{q} and the reference learner (genie)
\begin{equation}
    R_{n}\left(q,y;x\right) = \log\frac{p\left(y|x,\hat{\theta}\right)}{q\left(y|x,z^n\right)}.
\end{equation}
We propose to use the min-max regret $R_n$ as an active learning criterion which essentially quantifies the prediction performance of the training set $z^n$ for a given unlabeled test feature $x$. A "good" $z^n$ minimizes the min-max regret for any test feature and thus provides good test set performance. Since $R_n$ is a point wise quantity, we suggest to look at the average over all test data. 

We propose the following criterion:
\begin{equation}\label{eq:active_learning:c_a_n_batch}
    C_{n} = \min_{x^n}\max_{y^n}\sum_{x}\log\left(\sum_y p\left(y|x,\hat{\theta}\right)\right)
\end{equation}
where $\hat{\theta} = \hat{\theta}\left(x,y,z^n\right)$.
The idea is to find a set of training points, $x^n$ that minimizes the averaged log normalization factor (across unlabeled test points), for the worst possible labels $y^n$. This criterion looks for the worst case scenario since there is no assumption on the data distribution and we assume individual sequences.

Since (\ref{eq:active_learning:c_a_n_batch}) is difficult to solve for a general hypothesis class, we define a greedy form which we denote as Individual Active Learning (IAL):
\begin{equation}\label{eq:active_learning:c_a_n}
    C_{n|n-1} = \min_{x_i}\max_{y_i}\sum_{x}\log\left(\sum_y p\left(y|x,\hat{\theta}\right)\right)
\end{equation}
Note that when computing (\ref{eq:active_learning:c_a_n}), the previously labeled training set, $z^{n-1}$, is assumed available for the learner. The objective in (\ref{eq:active_learning:c_a_n}) is to find a single point $x_i$ from the unlabelled pool as opposed to the objective in~\eqref{eq:active_learning:c_a_n_batch} that tries to find an optimal batch $x^n$.
 
\subsection{Deep individual active learning} \label{sec:active_learning:ial_gpc}
In this section, we derive an approximate inference of IAL for DNNs. 
DNN hypothesis class poses a challenging problem for information-theoretic active learning since its parameter space is infinite dimensional and the weights posterior distribution is analytically intractable. Moreover, direct application of deep active learning schemes is unfeasible for real world large scale data since it requires training the entire model for each possible training point.

Recall that $x$, $y$ and $p(\theta)$ are test feature, test label and prior on the weights respectively. The MAP estimation for $\theta$ is
\begin{equation}\label{eq:active_learning:map_gpc}
    \hat{\theta} = \arg\max_{\theta} p\left(y^n,y|x^{n},x,\theta\right)p(\theta).
\end{equation}
The prior acts as a regularizer over the latent vector $\theta$

Given a training set, the maximization of the likelihood function $p\left(y^n,y|x^{n},x,\theta\right)$ is performed by training the DNN with all the data and converging to a steady state maxima. 
In order to avoid re-training the entire network for all possible values of $x$, $y$, $x_n$ and $y_n$,  we utilize the independence between soft-max scores in the MAP estimation. Using Bayes, we observe that (\ref{eq:active_learning:map_gpc}) can be written as:
\begin{equation}\label{eq:active_learning:map_gpc1}
    \hat{\theta} = \arg\max_{\theta} p\left(y|x,\theta\right)p\left(y_n|x_n,\theta\right)p\left(\theta|y^{n-1},x^{n-1}\right)
\end{equation}
where $p\left(\theta|y^{n-1},x^{n-1}\right)$ is the posterior of $\theta$ based on the available data $z^{n-1}=(x^{n-1},y^{n-1})$. 

The posterior $p\left(\theta|y^{n-1},x^{n-1}\right)$ is not dependent on the test data $(x,y)$ and the evaluated labeling candidate $(x_n,y_n)$ and thus can be computed once and then used in the IAL selection process. For a DNN, this posterior is multi modal and intractable to compute directly. Therefore, we propose to approximate it by some simpler distribution which will allow easier computation of the maximum likelihood $\hat{\theta}$. 

\subsubsection{MC-Dropout as the posterior function}
There are many approximation approaches which find a representative distribution for the true posterior from some parametric family of simpler distributions ~\cite{daxberger2021laplace,wilson2016stochastic,zhang2018advances}. In this work, we have opted to use the method in \citet{gal2016dropout}, denoted as MC-Dropout, due to its simplicity and favorable performance. 

The MC Dropout algorithm, executes multiple dropout inference iterations and for each iteration, the final prediction probabilities are averaged and form an approximation of the predictive output probability. \citet{gal2016dropout} argued that DNNs with dropout applied before every weight layer are mathematically equivalent to approximate variational inference in the Gaussian process. Therefore, $p\left(\theta|y^{n-1},x^{n-1}\right)$ can be approximated in KL-sense by a distribution which is controlled by a dropout parameter and a given prior. Since this approximate distribution is still very complex, we propose to use a uniform distribution on the weights for each dropout iteration.

Using MC Dropout, the maximization in~\eqref{eq:active_learning:map_gpc1} does not include any training of the DNN, but only sampling weights based on the approximated posterior, performing inference with $x_i$ and $x$ as inputs for each $\theta$ in MC-Dropout ensemble and finding the $\theta$ which maximizes the product of all the softmax's and posterior. This is considerably faster than training the network for every point.

\subsubsection{Algorithm description}
The resulting algorithm denoted Deep Individual Active Learning (DIAL) is shown in Algorithm~\ref{alg:active_learning:IAL_algo} and is as follows. 

\begin{enumerate}
\item Train a model on the labeled training set $z^{n-1}$. 
\item Run MC-Dropout inference for $M$ iterations on all the unlabeled pool and test set.
\item Find the weight that maximizes DNN prediction of the test input and the unlabeled candidate input as in \eqref{eq:active_learning:map_gpc1}. 
\item Accumulate the pNML regret of the test point given these estimations. 
\item Find the unlabeled candidate for which the worst case averaged regret of the test set is minimal as in~\eqref{eq:active_learning:c_a_n}.
\end{enumerate}
For step 2, since the variational posterior associated with MC-Dropout is difficult to evaluate, we assume that it is uniform for all the sampled weights.

\begin{algorithm}[tb]
\caption{DIAL: Deep Individual Active Learning}\label{alg:active_learning:IAL_algo}
\begin{algorithmic}
\State \textbf{Input} Training set $ z^{n-1}$, unlabeled pool and test samples $\lbrace x_i \rbrace^N_{i=1}$ and $\lbrace x_k \rbrace^K_{k=1}$.
\State \textbf{Output} Next data point for labeling $\hat{x_i}$
\State Run MC-Dropout using using $z^{n-1}$ to get $\left\{\theta_m \right\}_{m=1}^M$
\State $\mathbf{S} = zeros(N,|\mathbb{Y}|)$ 
\For {$i \gets 1$ to $N$}
    \For{$y_i \in \mathbb{Y}$} 
        \For{$k \gets 1$ to $K$}
            \State $\Gamma = 0$
            \For{$y_k \in \mathbb{Y}$}
                \State $
                    \hat{\theta} = \argmax_{\theta_m} p\left(y_k|x_k,\theta_m\right)p\left(y_i|x_i,\theta_m\right)$
            \State $\Gamma = \Gamma + p\left(y_k | x_k,\hat{\theta} \right)$  
            \EndFor
            \State $\mathbf{S}\left(i,j\right) = \mathbf{S}\left(i,j\right) + \log\Gamma$
        \EndFor
   \EndFor
\EndFor
\State $\hat{x_i} = \argmin_{x_i}\max_{y_i}\mathbf{S}$
\end{algorithmic}
\end{algorithm}

\begin{figure}[tb]
\centering
\begin{subfigure}{0.32\textwidth}
\centerline{\includegraphics[width=0.99\textwidth]{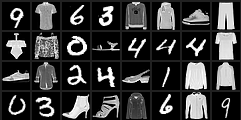}}
\caption{MNIST and OOD images}
\label{fig:active_learning:mnist_ood_images}
\end{subfigure}
\begin{subfigure}{0.32\textwidth}
\centerline{\includegraphics[width=0.99\textwidth]{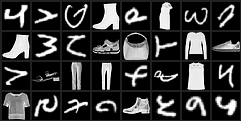}}
\caption{EMNIST and OOD images}
\label{fig:active_learning:emnist_ood_images}
\end{subfigure}
\begin{subfigure}{0.32\textwidth}
\centerline{\includegraphics[width=0.99\textwidth]{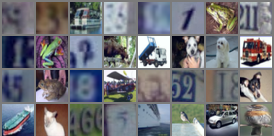}}
\caption{CIFAR10 and OOD images}
\label{fig:active_learning:cifar10_ood_images}
\end{subfigure}
\par\bigskip
\begin{subfigure}{0.32\textwidth}
\centerline{\includegraphics[width=0.99\textwidth]{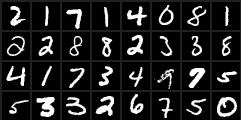}}
\caption{MNIST test images}
\label{fig:active_learning:mnist_images}
\end{subfigure}
\begin{subfigure}{0.32\textwidth}
\centerline{\includegraphics[width=0.99\textwidth]{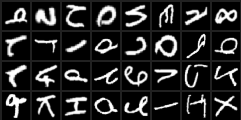}}
\caption{EMNIST test images}
\label{fig:active_learning:emnist_images}
\end{subfigure}
\begin{subfigure}{0.32\textwidth}
\centerline{\includegraphics[width=0.99\textwidth]{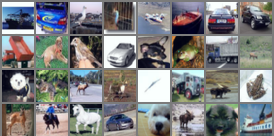}}
\caption{CIFAR10 test images}
\label{fig:active_learning:cifar10_images}
\end{subfigure}
\caption{Datasets that contain a mix of images with OOD samples}
\label{fig:active_learning:images}
\end{figure}

\section{Experiments}
We tested the proposed DIAL strategy in two scenarios: (i) where the initial training, unlabeled pool, and test data come from the same distribution, and (ii) when there are OOD samples present in the unlabeled pool. 

The reason for using the individual setting and DIAL as its associated strategy in the presence of OOD samples is that it does not make any assumptions about the data generation process, making the results applicable to a wide range of scenarios, including PAC~\cite{simon2015almost}, stochastic~\cite{merhav1998universal}, adversarial settings, as well as samples from unknown distributions.

\subsection{Datasets}
We considered the following datasets for training and evaluation of the different active learning methods.

\textbf{The MNIST dataset}~\cite{deng2012mnist} consists of 28x28 grayscale images of handwritten digits, with 60K images for training and 10K images for testing.

\textbf{The EMNIST dataset}~\cite{cohen2017emnist} is a variant of the MNIST dataset that includes a larger variety of images (upper and lower case letters, digits, and symbols). It consists of 240K images with 47 different labels.

\textbf{The CIFAR10 dataset}~\cite{krizhevsky2014cifar} consists of 60K 32x32 color images in 10 classes. The classes include objects such as airplanes, cars, birds, and ships.

\textbf{Fashion MNIST}~\cite{xiao2017fashion} is a dataset of images of clothing and accessories, consisting of 70K images. Each image is 28x28 grayscale pixels. 

\textbf{The SVHN dataset}~\cite{sermanet2012convolutional} contains 600K real-world images with digits and numbers in natural scene images collected from Google Street View.

\subsection{Baselines}
We built upon ~\citet{Huang2021deepal} open-source implementation of the following methods.

\textbf{The Random sampling} algorithm is the most basic approach in learning. It selects samples to label randomly, without considering any other criteria. This method can be useful when the data is relatively homogeneous and easy to classify, but it can be less efficient when the data is more complex or when there is a high degree of uncertainty.

\textbf{The Bayesian Active Learning by Disagreement (BALD)} method~\cite{gal2017deep} utilizes an acquisition function that calculates the mutual information between the model's predictions and the model's parameters. This function measures how closely the predictions for a specific data point are linked to the model's parameters, indicating that determining the true label of samples with high mutual information would also provide insight into the true model parameters.

\textbf{The Core-set} algorithm~\cite{sener2017active}  aims to find a small subset from a large labeled dataset such that a model learned from this subset will perform well on the entire dataset. The associated active learning algorithm chooses a subset that minimizes this bound, which is equivalent to the k-center problem.

\subsection{Experimental setup}
The first setting we consider consists of an initial training set, an unlabeled pool, and an unlabeled test set, all drown from the \textbf{same distribution}.
The experiment includes the following 4 steps:
\begin{enumerate}
    \item A model is trained on the small labeled data-set (initial training set). 
    \item One of the active learning strategies is utilized to select a small number of the most informative examples from the unlabeled pool.
    \item Querying the labels of the selected samples and adding them to the labeled data-set.
\item Retrain the model using the new training set.
\end{enumerate}
Steps 2-4 are repeated multiple times, with the model becoming more accurate with each iteration as it is trained on a larger labeled data-set.

In addition to the standard setting, we evaluate the performance in \textbf{the presence of OOD samples}. In this scenario, the initial training and test sets are drawn from the same distribution, but the unlabeled pool contains a mix of OOD samples. When an OOD unlabeled sample is selected for annotation, it is not used in training of the next iteration of the model. An effective strategy would recognize that OOD samples do not improve performance on the test set and avoid selecting them.

A visual representation of the scenario with OOD samples is illustrated in~\figref{fig:active_learning:images}. The top row depicts the unlabeled pool, which contains a mixture of both valid and OOD samples, while the bottom row shows the test set, which contains only valid samples. 

\begin{figure}[tb]
\centering
\begin{subfigure}{0.49\textwidth}
\centerline{\includegraphics[width=1.0\textwidth]{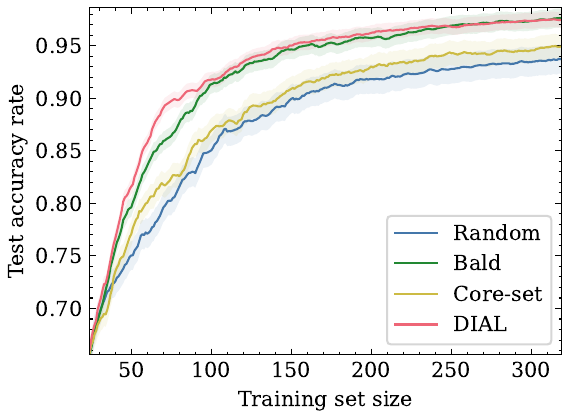}}
\caption{MNIST}
\label{fig:active_learning:mnist}
\end{subfigure}
\begin{subfigure}{0.49\textwidth}
\centerline{\includegraphics[width=1.0\textwidth]{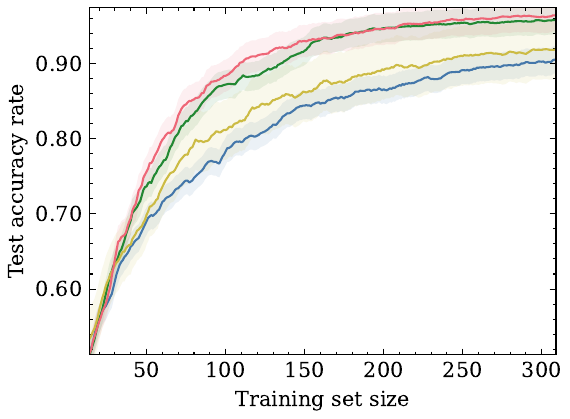}}
\caption{MNIST with OOD }
\label{fig:active_learning:mnist_ood}
\end{subfigure}
\caption{Accuracy rate as function of the training set size}
\end{figure}

\subsection{MNIST experimental results} \label{sec:active_learning:mnist_results}
Following~\citet{gal2017deep}, we considered a model consisting of two blocks of convolution, dropout, max-pooling, and ReLu, with 32 and 64 5x5 convolution filters. These blocks are followed by 2 fully connected layers that include dropout between them. The layers have 128 and 10 hidden units respectively. The dropout probability was set to 0.5 in all three locations. In each active learning round, a single sample was selected. 
We executed all active learning methods 6 times with different random seeds. For BALD and DIAL, we used 100 dropout iterations and employed the criterion on 512 random samples from the unlabeled pool.

\begin{table}[bt]
\centering
\small
\caption{MNIST with OOD training set}
\begin{tabular}{l|c|c|c}
\toprule
Methods & 85\% Acc. & 75\%  Acc. & 65\%  Acc. \\
\toprule
Random & 145 & 73 & 36 \\
Core-set & 117 & 61 & 33  \\
BALD & 83 & 51 & 32 \\
DIAL & \textbf{73 (-12.1\%)} & \textbf{48 (-5.9\%)} & \textbf{30 (-6.2\%)} \\
\bottomrule
\end{tabular}
\label{tab:active_learning:mnist_ood}
\end{table}

MNIST results are shown in~\figref{fig:active_learning:mnist}. DIAL is the top preforming method and have a better test set accuracy than BALD, Core-set, and Random. The largest efficiency is at accuracy rate of 0.95 where DIAL uses 169 samples while BALD attains this accuracy level with 209 samples.

To simulate the presence of OOD samples, we added the Fashion MNIST to the unlabeled pool such that the ratio of Fashion MNIST to MNIST is 1:1.
In this setting, DIAL outperforms all other baselines as shown in~\figref{fig:active_learning:mnist_ood}.
The largest improvement of DIAL over BALD is for accuracy level of 0.96 where DIAL requires 240 samples while BALD uses 307 samples.
The number of samples for additional accuracy rates is shown in~\tabref{tab:active_learning:mnist_ood}.

\begin{figure}[tb]
\centering
\begin{subfigure}{0.49\textwidth}
\centerline{\includegraphics[width=1.0\textwidth]{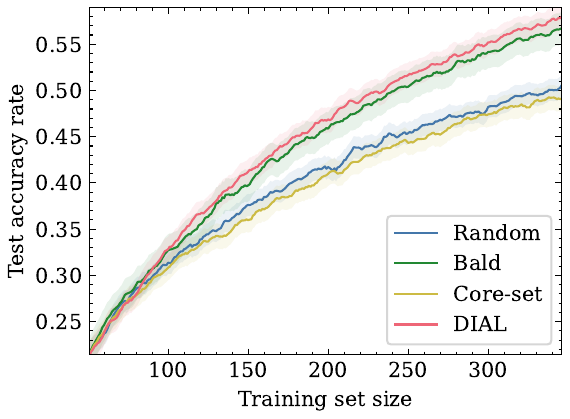}}
\caption{EMNIST}
\label{fig:active_learning:emnist}
\end{subfigure}
\begin{subfigure}{0.49\textwidth}
\centerline{\includegraphics[width=1.0\textwidth]{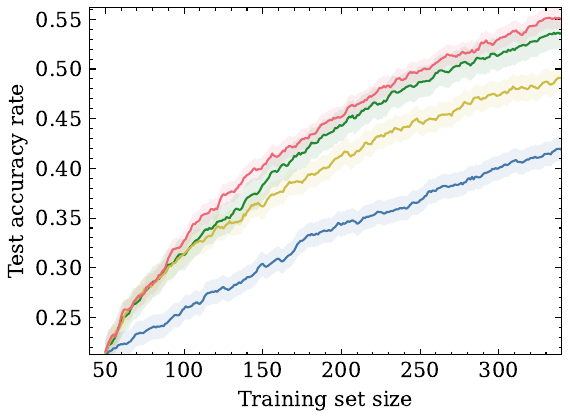}}
\caption{EMNIST with OOD}
\label{fig:active_learning:emnist_ood}
\end{subfigure}
\caption{Active learning performance for EMNIST}
\end{figure}

\subsection{EMNIST experimental results} \label{sec:active_learning:emnist_results}
We followed the same setting as the MNIST experiment with a slightly larger model than MNIST consisting of three blocks of convolution, dropout, max-pooling, and ReLu. 

The experimental results, shown in \figref{fig:active_learning:emnist}, indicate that DIAL is the top-performing method: For an accuracy rate of 0.56, it requires 8.3\% less training samples when compared to the second best method.

In the presence of OOD samples, the DIAL method outperforms all other baselines as shown in~\figref{fig:active_learning:emnist_ood} and Table~\ref{tab:active_learning:emnist_ood}. For 300 samples, DIAL achieves a test set accuracy rate of 0.52, while BALD, Core-set, and Random attain 0.51, 0.42, and 0.40 respectively. For the same accuracy rate of 0.53, DIAL needs 288 training samples, while BALD requires 309. Core-set and Random do not achieve this accuracy level for the test and training set sizes.

\begin{table}[tb]
\centering
\small
\caption{EMNIST with OOD training set}
\begin{tabular}{l|c|c|c}
\toprule
Methods & 40\% Acc. & 30\%  Acc. & 25\%  Acc. \\
\toprule
Random & 281 & 140 & 80 \\
Core-set & 221 & 96 & 62 \\
BALD & 154 & 85 & \textbf{59} \\
DIAL & \textbf{138 (-10.4\%)} & \textbf{84 (-1.2\%)} &  \textbf{59 (0\%)} \\
\bottomrule
\end{tabular}
\label{tab:active_learning:emnist_ood}
\end{table}

\subsection{CIFAR10 experimental results} \label{sec:active_learning:cifar10_results}
For the CIFAR10 data-set, we utilized ResNet-18~\cite{he2016deep} with acquisition size of 16 samples. 
We used 1K initial training set size and measured the performance of the active learning strategies up to a training set size of 3K.

CIFAR10 results are shown in~\figref{fig:active_learning:cifar10}. Overall, DIAL and Random preform the same and have a better test set accuracy than BALD and Core-set for all training set sizes greater than 2,100.

When the presence of OOD samples in the unlabeled pool is considered, as shown in \figref{fig:active_learning:cifar10_ood}, DIAL outperforms all other baselines. This can be explained by \figref{fig:active_learning:cifar10_ood_ratio} of the appendix, which shows the ratio of OOD samples to the training set size. The figure indicates that DIAL selects fewer OOD samples, which explains its good performance. It is worth realizing that for all OOD scenarios, DIAL was able to better select in-distribution samples without any expert telling it what is the distribution and only using the unlabeled test features. This shows that DIAL is a universal approach which can adapt to any distribution shift. 

\begin{figure}[tb]
\centering
\begin{subfigure}{0.49\textwidth}
\centerline{\includegraphics[width=0.89\textwidth]{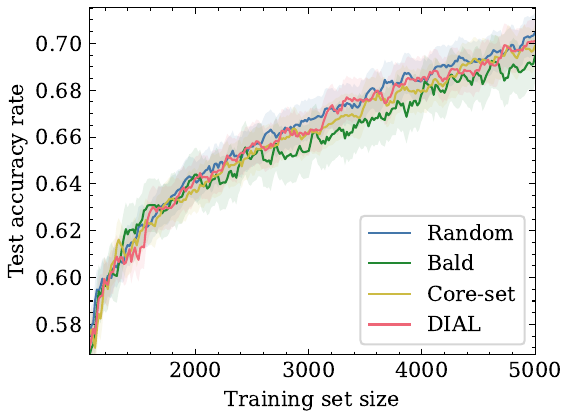}}
\caption{CIFAR10 with IND only}
\label{fig:active_learning:cifar10}
\end{subfigure}
\begin{subfigure}{0.49\textwidth}
\centerline{\includegraphics[width=0.89\textwidth]{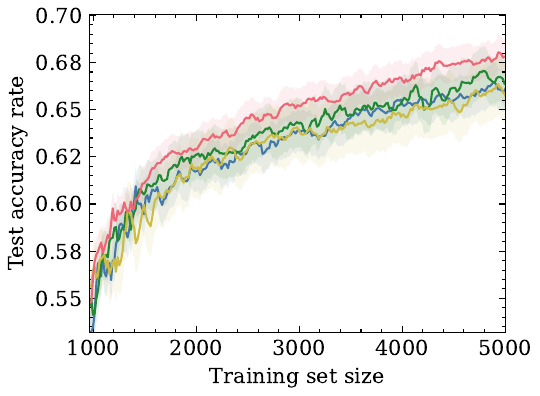}}
\caption{CIFAR10 with OOD}
\label{fig:active_learning:cifar10_ood}
\end{subfigure}
\caption{Active learning for CIFAR10}
\end{figure}

\Tabref{tab:active_learning:cifar10_ood} shows the number of samples required for different accuracy levels. For the same accuracy rate of 0.65, DIAL needs 21.3\%, 34.8\%, and 31.1\% less training samples than BLAD, Core-set, and Random strategies.

\begin{table}[tb]
\centering
\caption{CIFAR10 in the presence of OOD samples}
\small
\begin{tabular}{l|c|c|c}
\toprule
Methods & 65\% Acc. & 62\%  Acc. & 60\%  Acc. \\
\toprule
Random &  3604 & 1956 & 1444\\
Core-set & 3812 & 1844 & 1332\\
BALD & 3156 & 1636 & 1316\\
DIAL & \textbf{2484 (-21.3\%)} & \textbf{1556 (-4.9\%)} & \textbf{1188 (-9.7\%)} \\
\bottomrule
\end{tabular}

\label{tab:active_learning:cifar10_ood}
\end{table}

\section{Limitations} \label{sec:active_learning:limitations}
The proposed DIAL algorithm is a min-max strategy for the individual settings. However, DIAL may not be the most beneficial approach in scenarios where the unlabeled pool is very similar to the test set, where different selection strategies may preform similarly and with smaller complexity.
This limitation of DIAL is supported by the experimental results of section~\ref{sec:active_learning:cifar10_results}, where the DIAL algorithm performed similarly to random selection for the CIFAR10 dataset.

Another limitation of DIAL is that it has a higher overhead computation compared to other active learning methods such as BALD. This is because DIAL involves computing the regret on the test set, which requires additional computations and could become significant when the unlabeled pool or the test set are very large.

\begin{figure}[tb]
\centering
\centerline{\includegraphics[width=0.5\columnwidth]{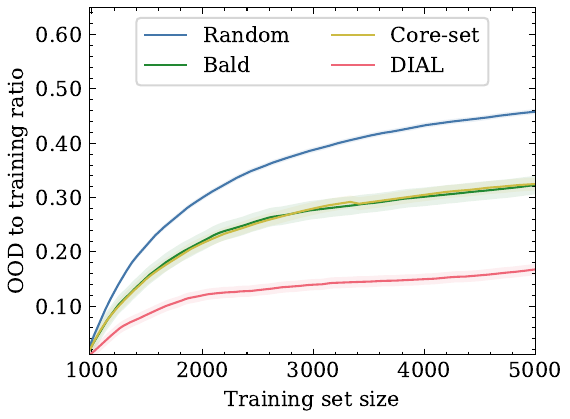}} 
\caption{The amount of chosen OOD samples for CIFAR10}
\label{fig:active_learning:cifar10_ood_ratio}
\end{figure}

\section{Concluding remarks} \label{sec:active_learning:conclusions}
In this chapter, we propose a min-max active learning criterion for the individual setting, which does not rely on any distributional assumptions. We have also developed an efficient method for computing this criterion for DNNs. Our experimental results demonstrate that the proposed strategy, referred to as DIAL, is particularly effective in the presence of OOD samples in the unlabeled pool. Specifically, our results show that DIAL requires 22.8\%, 11.0\%, and 33.8\% fewer samples to achieve a certain level of accuracy on the MNIST, EMNIST, and CIFAR10 datasets, respectively.

As future work, once can investigate batch acquisition criteria that take into account batch selection. This will allow us to consider the relationship between the selected samples and the overall composition of the batch, which may lead to even further improvements in performance.

\chapter{Future Research} \label{chap:future_research}
Throughout this work, we have shown that DNN based learners have achieved remarkable performance in classification. Despite this success, the theoretical understanding of the generalization capabilities of neural networks is still limited, making the DNN hypothesis class an intriguing area of research.

The main problem with the pNML procedure in DNN is that the model might fit exactly to each possible test label and, therefore, using the regret as a generalization measure might be useless.
We introduced in chapter~\ref{chap:overparameterized_linear_regression} the pNML with norm constraint and we can apply a similar approach for the DNN hypothesis family: Find a property that defines the complexity of the DNN model and use it to constrain the pNML learner.
Recently, researchers have proposed many ways to quantify the complexity of DNN models: the number of degrees of freedom~\cite{gao2016degrees}, prequential code~\cite{blier2018description} and intrinsic dimension~\cite{li2018measuring}.
Using these methods as a constraint in the pNML scheme might lead to a meaningful regret and a better understanding of DNN generalization performance.

An additional potential research direction is exploring the relationship between stability and pNML regret.
Stability is a classical approach for proving generalization bounds:
A stable learning algorithm is one for which when the training data is modified slightly the prediction does not change much.
This approach has been used to obtain relatively strong generalization bounds for several convex optimization algorithms~\cite{shalev2010learnability}.
The pNML regret seems to relate to the stability criteria. During the pNML procedure, we add the test sample data to the training set with all possible labels. 
Small pNML regret usually means that there is only one label that the model is certain about and that the test label value does not change the model drastically.
Finding the relation between stability and the pNML will improve our understanding of learning algorithm generalization capabilities.

In all our derivation and suggested algorithm, we assumed that the hypothesis class was predefined. However, when dealing with real world data, it can be difficult to determine what hypothesis class should be used. In fact, the selection of the hypothesis class itself can have a significant impact on the accuracy and effectiveness of the model. In order to address this issue, one approach is known as ``twice universality'' was briefly explored in~\citet{bibas2019deep}. It involves executing the pNML algorithm over a number of hierarchical model families (such as DNN with different number of layers). By doing so, we can explore the space of possible hypothesis classes and select the one that best fits the given dataset. This approach might be particularly useful in cases where the true hypothesis class is unknown or difficult to determine a priori. By allowing for more flexibility in the choice of hypothesis class, the twice universality approach can help to improve the performance of a model given that task and dataset.

Finally, we explored the pNML learner in the context of linear regression and neural networks. The pNML is a versatile method that can be applied to other hypothesis classes as well: A straight forward research direction is to calculate the pNML for additional hypothesis class. Linear regression with L1 regularization, decision trees, and logistic regression are all popular machine learning models that are widely used in many applications, but their generalization performance is not yet well understood. Applying the pNML algorithm to these models can help to better understand their generalization capabilities and identify the optimal model for a given dataset.

\bibliographystyle{plainnat}
\bibliography{references}

\appendix
\chapter{Appendix for chapter 2}

\section{Genie prediction upper bound}
\label{sec:norm_constrained_pnml:genie_upper_bound}
The ridge regression ERM solution is given by
\begin{equation}
\theta_N = (X_N^\top X_N + \lambda I)^{-1} X_N^\top Y_N, \quad P_{N} \triangleq \left( X_N^\top X_N + \lambda I \right)^{-1}.
\end{equation}
Using the recursive form, we can update the learnable parameters with the following formula:
\begin{equation} \label{eq:norm_constrained_pnml:rls}
\theta_{N+1} = \theta_N + \frac{P_N}{1 + x^\top P_N x} x \left(y - x^\top \theta_N \right) .
\end{equation}
The SVD decomposition of the training set data matrix
\begin{equation}
X_N^\top = 
\begin{bmatrix}
u_1 & \hdots & u_M
\end{bmatrix}
\begin{bmatrix}
\textit{diag}\left(h_1,\hdots, h_N \right) \\
\mathbf{0}
\end{bmatrix}
\begin{bmatrix}
v_1^\top \\ \vdots \\ v_N^\top
\end{bmatrix}
= UHV^\top.
\end{equation}
Notice that we are dealing with the over-parameterized region such that $M \geq N$:
\begin{equation} \label{eq:norm_constrained_pnml:1_plus_x_P_x_upper_bound_two}
1 + x^\top P_N x = 
1 +
\sum_{i=1}^N \frac{\left( u_i^\top x \right)^2}{h_i^2 + \lambda} 
+
\sum_{i=N+1}^M \frac{\left( u_i^\top x \right)^2}{\lambda}
\leq
K_0 + \frac{1}{\lambda} \norm{x_\bot}^2.
\end{equation}
where $K_0$ is the under-parameterized pNML normalization factor:
\begin{equation} \label{eq:norm_constrained_pnml:underparam_norm_facotr}
K_0 = \int_{-\infty}^{\infty} \probthetagenietag dy'
= 1 + x^\top \left(X_N^\top X_N\right)^{-1} x.
\end{equation}

The genie prediction, using the recursive formulation of \eqref{eq:norm_constrained_pnml:rls}, is
\begin{equation} \label{eq:norm_constrained_pnml:genie_with_rls}
\begin{split}
\probthetagenie
&=
\frac{1}{\sqrt{2\pi\sigma^2}}
\exp\left\{
-\frac{1}{2\sigma^2} \left(y - x^\top \hat{\theta}(z^N;x,y)\right)^2
\right\}
\\ &=
\frac{1}{\sqrt{2\pi\sigma^2}}
\exp\left\{
-\frac{1}{2\sigma^2} \left(y - x^\top \theta_N - \frac{x^\top  P_N x}{1 + x^\top P_N x} \left(y - x^\top \theta_N \right) \right)^2
\right\}   
\\ &=
\frac{1}{\sqrt{2\pi\sigma^2}}
\exp\left\{
-\frac{\left(y - x^\top \theta_N \right)^2}{2\sigma^2\left(1 + x^\top P_N x\right)^2}
\right\}.
\end{split}
\end{equation}
Plug in the upper bound from \eqref{eq:norm_constrained_pnml:1_plus_x_P_x_upper_bound_two} proves the lemma:
\begin{equation}
\probthetagenie
\leq
\frac{1}{\sqrt{2\pi\sigma^2}}
\exp\left\{
- \frac{\left(y-x^\top \theta_N \right)^2}{2\sigma^2 K_0^2 \left(1 + \frac{\norm{x_\bot}^2}{K_0 \lambda} \right)^2}
\right\}.
\end{equation}

\section{The regularization factor lower bound}
\label{sec:norm_constrained_pnml:lambda_lower_bound}
We add the test sample $(x,y)$ to the training set.
The corresponding genie weights are
\begin{equation}
\thetagenie
=
\left(X_{N+1}^\top X_{N+1} + \lambda I \right)^{-1} X_{N+1}^\top
\begin{bmatrix} y_1 & \hdots & y_N & y \end{bmatrix}^\top.
\end{equation}
The MN least squares solution of the $N+1$ samples is $
\theta_{N+1}^*$.
Using Taylor series expansion of with respect to $\lambda$ and the MN recursive formulation
\begin{equation}
\begin{split}
\norm{\thetagenie}^2 
&\geq 
\norm{\theta_N^{* 2}}- 2 \theta_N^{*\top} X_{N+1}^+ X_{N+1}^{+ \top} \theta_N^{*} \lambda
\\ &=
\norm{\theta_N^{* 2}}- 2 \left( \theta_N^{* \top} X_N^+ X_N^{+ \top} \theta_N^*
+
\frac{\left(y - x^\top \theta_N^*\right)^2}{||x_{\bot}||^2} 
x^\top X_N^+ X_N^{+ \top} x  \right).
\end{split}
\end{equation}
The second equality is derived in \appref{sec:norm_constrained_pnml:taylor_series_second_term}. 
Utilizing \Theoref{theorem:mn_norm}
the following inequality is obtained
\begin{equation}
\begin{split}
\norm{\thetagenie}^2 
&\geq
\norm{\theta_N^*}^2 + \frac{1}{\norm{x_\bot}^2}\left(y - x^\top \theta_N^*\right)^2 
\\ & \qquad \qquad
- 2 \left(\theta_N^{* \top} X_N^+ X_N^{+ \top} \theta_N^*
+
\frac{\left(y - x^\top \theta_N^*\right)^2}{||x_{\bot}||^2} 
x^\top X_N^+ X_N^{+ \top} x \right)
\lambda.
\end{split}
\end{equation}
Plug it the norm constrain $\norm{\thetagenie}^2 = \norm{\theta_N^*}^2$:
\begin{equation} \label{eq:norm_constrained_pnml:lambda_lower_bound}
\lambda \geq
\frac{1}{2}
\frac{\frac{1}{\norm{x_\bot}^2}\left(y - x^\top \theta_N^*\right)^2 }{\theta_N^{* \top} X_N^+ X_N^{+ \top} \theta_N^*
+
\frac{1}{||x_{\bot}||^2} 
x^\top X_N^+ X_N^{+ \top} x \left(y - x^\top \theta_N^*\right)^2}.
\end{equation}

\section{The pNML regret upper bound for over-parameterized linear regression}
\label{sec:norm_constrained_pnml:pnml_regret_upper_bound}
Denote $\delta \geq 0$ we relax the constraint
\begin{equation}
\norm{\theta_N^*}^2 = \norm{\theta_{N+1}}^2 \leq (1+\delta) \norm{\theta_N^*}^2.
\end{equation}
We get a perfect fit when to the following constraint is satisfied.
\begin{equation}
(1+\delta) \norm{\theta_N^*}^2
\geq 
\norm{\theta_N^*}^2 + \frac{1}{\norm{x_\bot}^2}\left(y' - x^\top \theta_N^* \right)^2
\end{equation}
\begin{equation}
x^\top \theta_N^* - \sqrt{\delta \norm{x_\bot}^2 \norm{\theta_N^*}^2} \leq y' \leq  x^\top \theta_N^* + \sqrt{\delta \norm{x_\bot}^2 \norm{\theta_N^*}^2}
\end{equation}

We split the integral of the normalization factor into two parts: one with a perfect fit and the other we upper bound with the genie upper bound (\Lemmaref{lemma:genie_upper_bound})
\begin{equation}
\begin{split}
K &\leq  
2 \int_{x^\top \theta_N^*}^{y^*} \frac{1}{\sqrt{2 \pi \sigma^2}} dy' +
2 \int_{y^*}^{\infty} 
\frac{1}{\sqrt{2\pi\sigma^2}}
\exp\left\{
- \frac{\left(y'-x^\top \theta_N \right)^2}{2\sigma^2 K_0^2 \left(1 + \frac{\norm{x_\bot}^2}{K_0 \lambda} \right)^2}
\right\}
dy'.
\\
&\leq  
2 \int_{x^\top \theta_N^*}^{y^*} \frac{1}{\sqrt{2 \pi \sigma^2}} dy' +
\int_{-\infty}^{\infty} 
\frac{1}{\sqrt{2\pi\sigma^2}}
\exp\left\{
- \frac{\left(y'-x^\top \theta_N \right)^2}{2\sigma^2 K_0^2 \left(1 + \frac{\norm{x_\bot}^2}{K_0 \lambda} \right)^2}
\right\}
dy'.
\\
&=
2 \sqrt{\frac{2 \delta}{\pi \sigma^2}  \norm{x_\bot}^2 \norm{\theta_N^*}^2}
+
K_0 \left(1 + \frac{\norm{x_\bot}^2}{K_0 \lambda}\right),
\end{split}
\end{equation}
where we fixed $\lambda$ by its lower bound \eqref{eq:norm_constrained_pnml:lambda_lower_bound} at the point 
$y^*=x^\top \theta_N^* + \sqrt{\delta \norm{x_\bot}^2 \norm{\theta_N^*}^2}$. 
\begin{equation}
\begin{split}
K
&\leq  
2 \sqrt{\frac{2 \delta}{\pi \sigma^2}  \norm{x_\bot}^2 \norm{\theta_N^*}^2}
+
K_0 + 2 \norm{x_\bot}^2 \frac{\theta_N^{* \top} X_N^+ X_N^{+ \top} \theta_N^*
+
\delta \norm{\theta_N^*}^2
x^\top X_N^+ X_N^{+ \top} x }
{\delta \norm{\theta_N^*}^2}
\\ &=
\sqrt{\frac{2 \delta}{\pi \sigma^2}  \norm{x_\bot}^2 \norm{\theta_N^*}^2}
+
K_0 
+ 
2\norm{x_\bot}^2 x^\top X_N^+ X_N^{+ \top} x +
\frac{2 \norm{x_\bot}^2}{\delta \norm{\theta_N^*}^2}
\theta_N^{* \top} X_N^+ X_N^{+ \top} \theta_N^*
\\ &=
K_0 + 2\norm{x_\bot}^2 x^\top X_N^+ X_N^{+ \top} x
\\ & \qquad \qquad
+
\frac{2 \norm{x_\bot}^2}{\norm{\theta_N^*}^2}
\theta_N^{* \top} X_N^+ X_N^{+ \top} \theta_N^*
\left[
\frac{1}{\theta_N^{* \top} X_N^+ X_N^{+ \top} \theta_N^*}
\sqrt{\frac{\norm{\theta_N^*}^6}{2 \pi \sigma^2 \norm{x_\bot}^2}}
\sqrt{\delta}
+
\frac{1}{\delta}
\right].
\end{split}
\end{equation}
To find a tight upper bound, we choose $\delta$ that minimizes the right side
\begin{equation}
\begin{split}
K &\leq 
K_0 + 2\norm{x_\bot}^2 x^\top X_N^+ X_N^{+ \top} x
+
3\sqrt[3]{
\frac{1}{\pi \sigma^2}
\norm{x_\bot}^2
\theta_N^{* \top} X_N^+ X_N^{+ \top} \theta_N^*
}.
\end{split}
\end{equation}
Plugging in $K_0$, the theorem result is obtained.

\section{Taylor series second term} \label{sec:norm_constrained_pnml:taylor_series_second_term}
We prove that
\begin{equation}
\theta_N^{* \top} X_{N+1}^+ X_{N+1}^{+ \top} \theta_N^{*} =
\theta_N^{* \top} X_N^+ X_N^{+ \top} \theta_N^*
+
\frac{\left(y - x^\top \theta_N^*\right)^2}{||x_{\bot}||^2} 
x^\top X_N^+ X_N^{+ \top} x 
\end{equation}

We use the MN recursive formulation
\begin{equation} \label{eq:norm_constrained_pnml:second_term_first_derv}
\begin{split}
& \theta_N^{* \top} X_{N+1}^+ X_{N+1}^{+ \top} \theta_N^{*}
=
\left(\theta_N^* + \frac{x_\bot}{||x_{\bot}||^2}  (y -x^\top \theta_N^* ) \right)^\top X_{N+1}^+ X_{N+1}^{+ \top} \left(\theta_N^* + \frac{x_\bot}{||x_{\bot}||^2}  (y -x^\top \theta_N^*) \right)
\\ & \quad = 
\theta_N^{* \top} X_{N+1}^+ X_{N+1}^{+ \top} \theta_N^* 
+
\frac{\left(y -x^\top \theta_N^*\right)^2}{||x_{\bot}||^2} 
x_\bot^\top X_{N+1}^+ X_{N+1}^{+ \top} x_\bot
+
\frac{2(y -x^\top \theta_N^*)}{||x_{\bot}||^2}  
x_\bot^\top X_{N+1}^+ X_{N+1}^{+ \top} \theta_N^*
\end{split}
\end{equation}
and
\begin{equation}
X_{N+1}^+ X_{N+1}^{+ \top}  
=
X_N^+ X_N^{+ \top} 
+
\frac{x^\top X_N^+
X_N^{+ \top}  x}{||x_{\bot}||^4} 
x_\bot x_\bot^\top
-
\frac{1}{||x_{\bot}||^2} 
X_N^+ X_N^{+ \top} x x_\bot^\top
-
\frac{1}{||x_{\bot}||^2} x_\bot x^\top X_N^+ X_N^{+ \top}.
\end{equation}

Substitute $X_{N+1}^+ X_{N+1}^{+ \top}$ to \eqref{eq:norm_constrained_pnml:second_term_first_derv}:
\begin{equation}
\begin{split}
\theta_N^{* \top} X_{N+1}^+ X_{N+1}^{+ \top} \theta_N^{*}
&=
\theta_N^{* \top} X_N^+ X_N^{+ \top} \theta_N^* 
+
\frac{x^\top X_N^+ X_N^{+ \top}  x}{||x_{\bot}||^4}  \left(\theta_N^{* \top} x_\bot\right)^2
\\ &
-\frac{2}{||x_{\bot}||^2} x^\top X_N^+ X_N^{+ \top} \theta_N^* \left(\theta_N^{* \top} x_\bot \right)
\\ &
+
\frac{\left(y -x^\top \theta_N^*\right)^2}{||x_{\bot}||^2} 
\left[
x_\bot^\top
X_N^+ X_N^{+ \top} 
x_\bot
+
x^\top X_N^+ X_N^{+ \top}  x
-
2 x^\top X_N^+ X_N^{+ \top} x_\bot
\right]
\\ & 
+
\frac{2(y^* -x^\top \theta_N^*)}{||x_{\bot}||^2}
\bigg[
x_\bot^\top X_N^+ X_N^{+ \top} \theta_N^*
+
\frac{1}{||x_{\bot}||^2}
x^\top X_N^+ X_N^{+ \top}  x \left(\theta_N^{* \top} x_\bot \right)
\\ & \qquad \qquad \qquad 
-
\frac{1}{||x_{\bot}||^2} 
x_\bot^\top 
X_N^+ X_N^{+ \top} x \left(\theta_N^{* \top} x_\bot \right)
-
x_\bot^\top X_N^+ X_N^{+ \top} \theta_N^* 
\bigg].
\end{split}
\end{equation}
Most terms are zero since
\begin{equation}
x_\bot^\top \theta_N^* =  x^\top \left(I - X_N^\top X_N^{+ \top} \right) X_N^+ Y_N 
= Y_N^\top  \left(X_N^+ - X_N^+ \right)x = 0
\end{equation}
\begin{equation}
x_\bot^\top X_N^+ X_N^{+ \top} 
= x^\top \left(I - X_N^\top X_N^{+ \top} \right) X_N^+ X_N^{+ \top}
=  x^\top \left(X_N^+ X_N^{+ \top} - X_N^+ X_N^{+ \top} \right) 
= x^\top \mathbf{0}.
\end{equation}
And that proves the result.

\chapter{Appendix for chapter 4}
\section{The spectrum of real dataset} \label{sec:single_layer_nn:spectrum}
We provide a visualization of the training data spectrum when propagated to the last layer of a DNN.

We feed the training data through the model up to the last layer to create the training embeddings. 
Next, we compute the correlation matrix of the training embeddings and perform an SVD decomposition.
We plot the eigenvalues for different training sets in~\figref{fig:svd}. 

\Figref{fig:DenseNet_svd} shows the eigenvalues of DenseNet-BC-100 model when ordered from the largest to smallest.
For the SVHN training set, most of the energy is located in the first 50 eigenvalues and then there is a significant decrease of approximately $10^3$.
The same phenomenon is also seen in~\figref{fig:DenseNet_svd} that shows the eigenvalues of ResNet-40 model.
In our derived regret, if the test sample is located in the subspace that is associated with small eigenvalues (for example indices 50 or above for DenseNet trained with SVHN) then $x^\top g$ is large and so is the pNML regret.

For both DensNet and ResNet models, the values of the eigenvalues of CIFAR-100 seem to be spread more evenly compared to CIFAR-10, and the CIFAR-10 are more uniform than the SVHN. 
How much the eigenvalues are spread can indicate the variability of the set: SVHN is a set of digits that is much more constrained than CIFAR-100 which has 100 different classes.

\begin{figure}[tb]
\centering
\begin{subfigure}[t]{0.49\linewidth}
\includegraphics[width=1.0\textwidth]{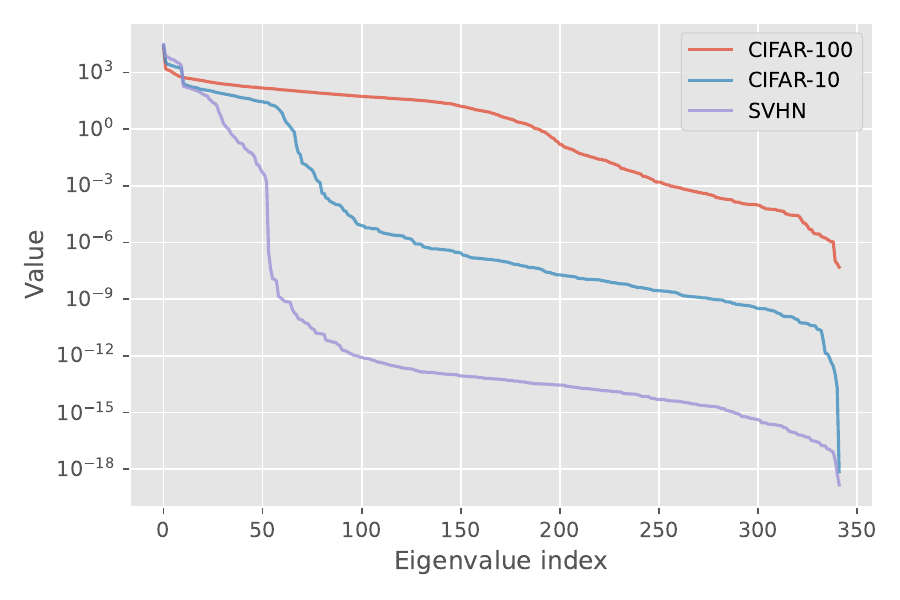}
\caption{DenseNet-BC-100  \label{fig:DenseNet_svd}}   
\end{subfigure}
\begin{subfigure}[t]{0.49\linewidth}
\includegraphics[width=1.0\textwidth]{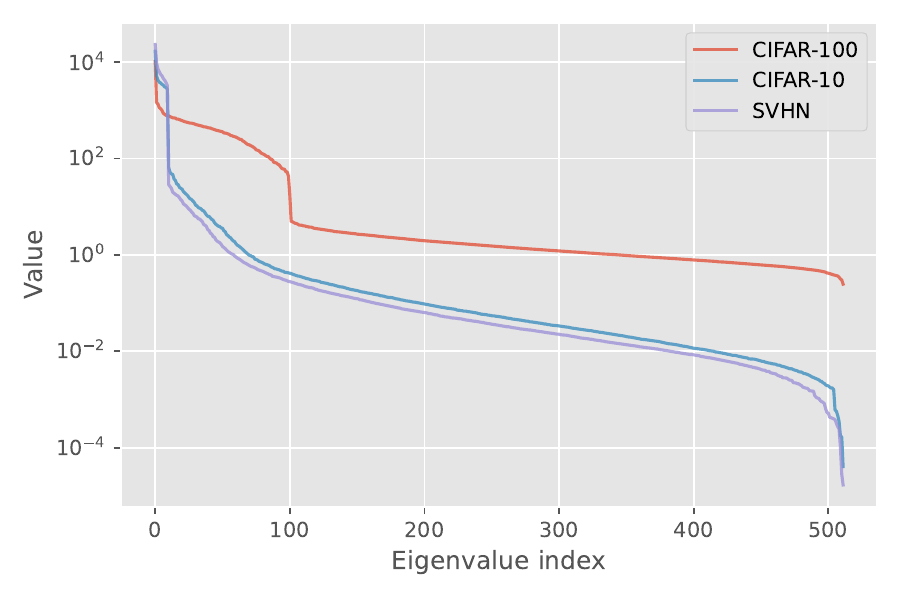}
\caption{ResNet-40 \label{fig:ResNet_svd}}   
\end{subfigure}
\caption{The spectrum of the training embeddings.}
\label{fig:svd}
\end{figure}

\section{Gram vs. Gram+pNML}
\label{sec:single_layer_nn:gram_vs_gram_plus_pnml}
We further explore the benefit of the pNML regret in detecting OOD samples over the Gram approach.
We focus on the DenseNet model with CIFAR-100 as the training set and LSUN (C) as the OOD set.

\Figref{fig:ind_hist} shows the 2D histogram of the IND set based on the pNML regret values and Gram scores. 
In addition, we plotted the best threshold for separating the IND and OOD of these sets.
pNML regret values less than 0.0024 and Gram scores below 0.0017 qualify as IND samples by both the pNML and Gram scores. 
Gram and Gram+pNML do not succeed to classify 1205 and 891 out of a total 10,000 IND samples respectively.

\Figref{fig:ood_hist} presents the 2D histogram of the LSUN (C) as OOD set. 
For regret values greater than 0.0024 and Gram score lower than 0.0017, the pNML succeeds to classify as IND but the Gram fails: 
There are 473 samples that the pNML classifies as OOD but the Gram fails, in contrast to 76 samples classified as such by the Gram and not by the pNML regret.
Most of the pNML improvement is in assigning a high score to OOD samples while there is not much change in the rank of the IND ones.

\begin{figure}[tb]
    \centering
\begin{subfigure}[t]{0.49\linewidth}
    \includegraphics[width=1.0\textwidth]{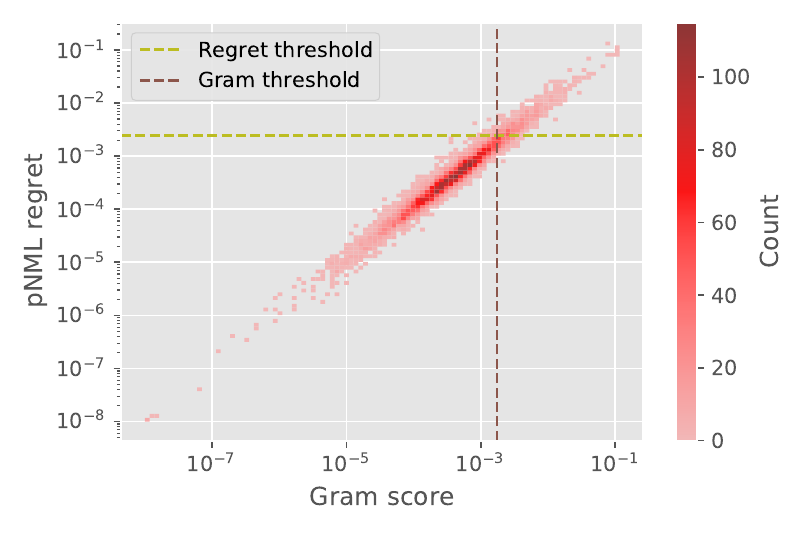}
    \caption{IND 2D histogram  \label{fig:ind_hist}}   
\end{subfigure}
\begin{subfigure}[t]{0.49\linewidth}
    \includegraphics[width=1.0\textwidth]{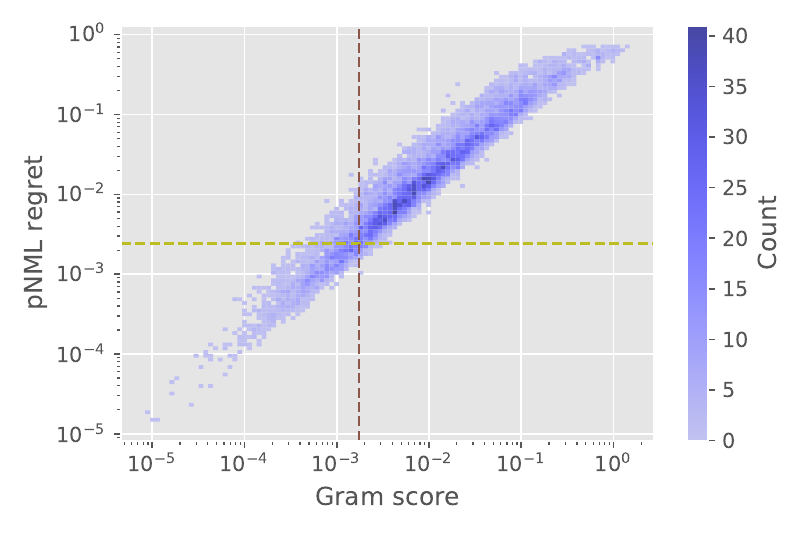}
    \caption{OOD 2D histogram \label{fig:ood_hist}}   
\end{subfigure}
\caption{2D histogram of the pNML regret and the Gram score}
\label{fig:histograms}
\end{figure}

\begin{table}[tb]
\centering
\caption{DenseNet-BC-100 model TNR at TPR95\% comparison}
\label{tab:single_layer_nn:TNRatTPR95_densenet}
\resizebox{\textwidth}{!}{%
\begin{tabular}{clcccc}
\toprule
IND & OOD &        Baseline/+pNML &            ODIN/+pNML &            Gram/+pNML &            OECC/+pNML \\

\midrule
\multirow{8}{*}{CIFAR-100} & iSUN &  14.8 / \textbf{81.2} &  37.4 / \textbf{82.8} &  95.8 / \textbf{97.9} &  97.5 / \textbf{99.2} \\
     & LSUN (R) &  16.4 / \textbf{82.7} &  41.6 / \textbf{84.5} &  97.1 / \textbf{98.7} &  98.4 / \textbf{99.6} \\
     & LSUN (C) &  28.3 / \textbf{65.7} &  58.2 / \textbf{65.4} &  65.3 / \textbf{76.3} &  74.6 / \textbf{83.4} \\
     & Imagenet (R) &  17.3 / \textbf{86.4} &  43.0 / \textbf{87.9} &  95.6 / \textbf{98.0} &  96.5 / \textbf{99.0} \\
     & Imagenet (C) &  24.3 / \textbf{77.2} &  52.5 / \textbf{78.6} &  88.8 / \textbf{93.8} &  92.6 / \textbf{96.9} \\
     & Uniform &    0.0 / \textbf{100} &    0.0 / \textbf{100} &    100 / \textbf{100} &    100 / \textbf{100} \\
     & Gaussian &    0.0 / \textbf{100} &    0.0 / \textbf{100} &    100 / \textbf{100} &    100 / \textbf{100} \\
     & SVHN &  26.2 / \textbf{79.2} &  56.8 / \textbf{79.0} &  89.3 / \textbf{93.7} &  89.0 / \textbf{90.7} \\
\midrule
\multirow{8}{*}{CIFAR-10} & iSUN &  63.3 / \textbf{93.2} &  94.0 / \textbf{94.3} &  99.1 / \textbf{99.8} &   99.7 / \textbf{100} \\
     & LSUN (R) &  66.9 / \textbf{94.2} &  \textbf{96.2} / 95.8 &  99.5 / \textbf{99.9} &   99.8 / \textbf{100} \\
     & LSUN (C) &  52.0 / \textbf{79.9} &  74.6 / \textbf{80.2} &  88.7 / \textbf{94.4} &  95.7 / \textbf{99.6} \\
     & Imagenet (R) &  59.4 / \textbf{93.4} &  92.5 / \textbf{94.6} &  98.8 / \textbf{99.6} &  99.3 / \textbf{99.9} \\
     & Imagenet (C) &  57.0 / \textbf{87.1} &  86.9 / \textbf{88.3} &  96.8 / \textbf{98.7} &  98.6 / \textbf{99.8} \\
     & Uniform &   76.4 / \textbf{100} &    100 / \textbf{100} &    100 / \textbf{100} &    100 / \textbf{100} \\
     & Gaussian &   88.1 / \textbf{100} &    100 / \textbf{100} &    100 / \textbf{100} &    100 / \textbf{100} \\
     & SVHN &  40.4 / \textbf{92.2} &  77.0 / \textbf{95.0} &  96.0 / \textbf{98.2} &  98.5 / \textbf{99.9} \\
\midrule
\multirow{9}{*}{SVHN} & iSUN &  78.3 / \textbf{93.6} &  78.5 / \textbf{96.3} &  99.6 / \textbf{99.9} &    100 / \textbf{100} \\
     & LSUN (R) &  77.1 / \textbf{91.7} &  77.0 / \textbf{95.2} &   99.7 / \textbf{100} &    100 / \textbf{100} \\
     & LSUN (C) &  73.5 / \textbf{89.7} &  68.5 / \textbf{90.0} &  93.4 / \textbf{97.2} &   99.5 / \textbf{100} \\
     & Imagenet (R) &  79.7 / \textbf{93.6} &  79.0 / \textbf{95.8} &  99.2 / \textbf{99.8} &    100 / \textbf{100} \\
     & Imagenet (C) &  78.9 / \textbf{92.8} &  77.6 / \textbf{94.5} &  98.0 / \textbf{99.3} &   99.9 / \textbf{100} \\
     & Uniform &   66.1 / \textbf{100} &   71.7 / \textbf{100} &    100 / \textbf{100} &    100 / \textbf{100} \\
     & Gaussian &  88.7 / \textbf{99.7} &   95.6 / \textbf{100} &    100 / \textbf{100} &    100 / \textbf{100} \\
     & CIFAR-10 &  69.1 / \textbf{81.0} &  66.6 / \textbf{88.5} &  75.1 / \textbf{86.8} &   98.9 / \textbf{100} \\
     & CIFAR-100 &  68.7 / \textbf{81.4} &  65.7 / \textbf{88.5} &  80.3 / \textbf{90.1} &   99.1 / \textbf{100} \\
\bottomrule
\end{tabular}

}
\end{table}

\begin{table}[tbh]
\centering
\caption{DenseNet-BC-100 model Detection Acc. comparison}
\label{tab:single_layer_nn:DetectionAcc._densenet}
\resizebox{\textwidth}{!}{%
\begin{tabular}{clcccc}
\toprule
IND & OOD &        Baseline/+pNML &            ODIN/+pNML &            Gram/+pNML &            OECC/+pNML \\

\midrule
\multirow{8}{*}{CIFAR-100} & iSUN &  64.0 / \textbf{89.9} &  76.5 / \textbf{90.3} &  95.6 / \textbf{97.0} &  96.5 / \textbf{98.0} \\
     & LSUN (R) &  65.0 / \textbf{90.5} &  77.7 / \textbf{91.0} &  96.3 / \textbf{97.4} &  97.2 / \textbf{98.5} \\
     & LSUN (C) &  72.6 / \textbf{85.3} &  83.4 / \textbf{85.2} &  83.7 / \textbf{87.5} &  87.0 / \textbf{90.2} \\
     & Imagenet (R) &  65.7 / \textbf{91.6} &  77.3 / \textbf{92.1} &  95.5 / \textbf{97.0} &  96.0 / \textbf{97.8} \\
     & Imagenet (C) &  69.0 / \textbf{89.0} &  80.8 / \textbf{89.3} &  92.4 / \textbf{94.5} &  94.0 / \textbf{96.1} \\
     & Uniform &   64.2 / \textbf{100} &   85.0 / \textbf{100} &    100 / \textbf{100} &   99.9 / \textbf{100} \\
     & Gaussian &   58.8 / \textbf{100} &   66.9 / \textbf{100} &    100 / \textbf{100} &    100 / \textbf{100} \\
     & SVHN &  75.5 / \textbf{90.3} &  86.0 / \textbf{90.3} &  92.3 / \textbf{94.4} &  92.1 / \textbf{93.0} \\
\midrule
\multirow{8}{*}{CIFAR-10} & iSUN &  89.2 / \textbf{94.2} &  94.6 / \textbf{94.8} &  98.0 / \textbf{99.0} &  98.7 / \textbf{99.6} \\
     & LSUN (R) &  90.2 / \textbf{94.7} &  \textbf{95.6} / 95.5 &  98.6 / \textbf{99.3} &  98.9 / \textbf{99.7} \\
     & LSUN (C) &  86.9 / \textbf{89.5} &  \textbf{89.7} / 89.4 &  92.1 / \textbf{94.8} &  95.5 / \textbf{98.8} \\
     & Imagenet (R) &  88.5 / \textbf{94.3} &  94.0 / \textbf{94.9} &  97.9 / \textbf{98.8} &  98.3 / \textbf{99.2} \\
     & Imagenet (C) &  88.0 / \textbf{91.9} &  \textbf{92.3} / 92.2 &  96.2 / \textbf{97.7} &  97.4 / \textbf{99.0} \\
     & Uniform &   94.8 / \textbf{100} &   99.7 / \textbf{100} &    100 / \textbf{100} &    100 / \textbf{100} \\
     & Gaussian &   95.3 / \textbf{100} &   99.8 / \textbf{100} &    100 / \textbf{100} &    100 / \textbf{100} \\
     & SVHN &  83.2 / \textbf{94.0} &  88.1 / \textbf{95.1} &  95.8 / \textbf{97.3} &  97.4 / \textbf{99.3} \\
\midrule
\multirow{9}{*}{SVHN} & iSUN &  89.7 / \textbf{94.6} &  87.7 / \textbf{95.7} &  98.3 / \textbf{99.1} &   99.8 / \textbf{100} \\
     & LSUN (R) &  89.2 / \textbf{93.8} &  87.2 / \textbf{95.1} &  98.6 / \textbf{99.2} &   99.9 / \textbf{100} \\
     & LSUN (C) &  88.0 / \textbf{92.8} &  83.6 / \textbf{92.8} &  94.3 / \textbf{96.4} &  98.5 / \textbf{99.8} \\
     & Imagenet (R) &  90.2 / \textbf{94.4} &  88.2 / \textbf{95.5} &  97.9 / \textbf{98.9} &   99.7 / \textbf{100} \\
     & Imagenet (C) &  89.8 / \textbf{94.2} &  87.6 / \textbf{94.8} &  96.7 / \textbf{98.1} &   99.5 / \textbf{100} \\
     & Uniform &  87.9 / \textbf{98.8} &  85.2 / \textbf{99.4} &   99.9 / \textbf{100} &    100 / \textbf{100} \\
     & Gaussian &  93.6 / \textbf{98.4} &  95.4 / \textbf{99.1} &    100 / \textbf{100} &    100 / \textbf{100} \\
     & CIFAR-10 &  86.5 / \textbf{91.0} &  83.5 / \textbf{92.7} &  89.0 / \textbf{92.0} &  97.4 / \textbf{99.8} \\
     & CIFAR-100 &  86.5 / \textbf{91.0} &  83.1 / \textbf{92.8} &  90.4 / \textbf{93.2} &  97.7 / \textbf{99.8} \\
\bottomrule
\end{tabular}

}
\end{table}

\begin{table}[tbh]
\centering
\caption{ResNet-34 model TNR at TPR95\% comparison}
\label{tab:single_layer_nn:TNRatTPR95_resnet}
\resizebox{\textwidth}{!}{%
\begin{tabular}{clcccc}
\toprule
IND & OOD &        Baseline/+pNML &            ODIN/+pNML &            Gram/+pNML &            OECC/+pNML \\

\midrule
\multirow{8}{*}{CIFAR-100} & iSUN &  16.6 / \textbf{26.1} &  \textbf{45.4} / 44.1 &  94.7 / \textbf{95.7} &  97.2 / \textbf{98.0} \\
     & LSUN (R) &  18.4 / \textbf{28.4} &  \textbf{45.5} / 44.6 &  96.6 / \textbf{97.1} &  98.3 / \textbf{99.0} \\
     & LSUN (C) &  18.2 / \textbf{30.1} &  44.0 / \textbf{51.2} &  64.6 / \textbf{72.9} &  80.3 / \textbf{89.8} \\
     & Imagenet (R) &  20.2 / \textbf{31.8} &  \textbf{48.7} / 47.6 &  94.8 / \textbf{96.2} &  95.5 / \textbf{95.8} \\
     & Imagenet (C) &  23.9 / \textbf{33.6} &  44.4 / \textbf{48.1} &  88.3 / \textbf{91.6} &  90.6 / \textbf{91.6} \\
     & Uniform &  10.1 / \textbf{89.1} &  98.4 / \textbf{98.5} &    100 / \textbf{100} &    100 / \textbf{100} \\
     & Gaussian &   0.0 / \textbf{13.7} &   4.5 / \textbf{66.8} &    100 / \textbf{100} &    100 / \textbf{100} \\
     & SVHN &  19.9 / \textbf{52.0} &  63.8 / \textbf{75.0} &  80.3 / \textbf{89.0} &  86.8 / \textbf{89.2} \\
\midrule
\multirow{8}{*}{CIFAR-10} & iSUN &  44.5 / \textbf{78.5} &  73.0 / \textbf{86.3} &  99.4 / \textbf{99.9} &   99.8 / \textbf{100} \\
     & LSUN (R) &  45.1 / \textbf{79.8} &  73.5 / \textbf{87.5} &  99.6 / \textbf{99.9} &   99.9 / \textbf{100} \\
     & LSUN (C) &  48.0 / \textbf{72.6} &  63.1 / \textbf{76.1} &  90.2 / \textbf{95.9} &  96.3 / \textbf{98.9} \\
     & Imagenet (R) &  44.0 / \textbf{72.8} &  71.8 / \textbf{81.9} &  98.9 / \textbf{99.6} &  99.6 / \textbf{99.8} \\
     & Imagenet (C) &  45.9 / \textbf{71.4} &  66.5 / \textbf{78.0} &  97.0 / \textbf{98.8} &  98.9 / \textbf{99.7} \\
     & Uniform &   71.4 / \textbf{100} &    100 / \textbf{100} &    100 / \textbf{100} &    100 / \textbf{100} \\
     & Gaussian &   90.2 / \textbf{100} &    100 / \textbf{100} &    100 / \textbf{100} &    100 / \textbf{100} \\
     & SVHN &  32.2 / \textbf{69.1} &  81.9 / \textbf{90.8} &  97.6 / \textbf{99.2} &  99.3 / \textbf{99.7} \\
\midrule
\multirow{9}{*}{SVHN} & iSUN &  77.0 / \textbf{85.6} &  79.1 / \textbf{90.6} &  99.5 / \textbf{99.9} &    100 / \textbf{100} \\
     & LSUN (R) &  74.4 / \textbf{82.9} &  76.6 / \textbf{88.3} &  99.6 / \textbf{99.9} &    100 / \textbf{100} \\
     & LSUN (C) &  76.1 / \textbf{86.3} &  78.5 / \textbf{86.4} &  94.5 / \textbf{98.4} &  99.3 / \textbf{99.9} \\
     & Imagenet (R) &  79.0 / \textbf{88.0} &  80.8 / \textbf{92.5} &  99.4 / \textbf{99.8} &    100 / \textbf{100} \\
     & Imagenet (C) &  80.4 / \textbf{88.4} &  82.4 / \textbf{91.5} &  98.6 / \textbf{99.7} &   99.9 / \textbf{100} \\
     & Uniform &  85.2 / \textbf{95.6} &  86.1 / \textbf{99.3} &    100 / \textbf{100} &    100 / \textbf{100} \\
     & Gaussian &  84.8 / \textbf{94.9} &  90.9 / \textbf{99.4} &    100 / \textbf{100} &    100 / \textbf{100} \\
     & CIFAR-10 &  78.3 / \textbf{87.2} &  79.9 / \textbf{90.4} &  86.1 / \textbf{97.2} &  98.4 / \textbf{99.8} \\
     & CIFAR-100 &  76.9 / \textbf{85.8} &  78.5 / \textbf{89.1} &  87.6 / \textbf{96.9} &  98.4 / \textbf{99.8} \\
\bottomrule
\end{tabular}

}
\end{table}

\begin{table}[tbh]
\centering
\caption{ResNet-34 model Detection Acc. comparison}
\label{tab:single_layer_nn:DetectionAcc._resnet}
\resizebox{\textwidth}{!}{%
\begin{tabular}{clcccc}
\toprule
IND & OOD &        Baseline/+pNML &            ODIN/+pNML &            Gram/+pNML &            OECC/+pNML \\

\midrule
\multirow{8}{*}{CIFAR-100} & iSUN &  70.1 / \textbf{76.0} &  78.6 / \textbf{79.3} &  95.0 / \textbf{95.4} &  96.2 / \textbf{96.9} \\
     & LSUN (R) &  69.8 / \textbf{76.5} &  78.1 / \textbf{79.8} &  96.0 / \textbf{96.2} &  96.9 / \textbf{97.6} \\
     & LSUN (C) &  69.4 / \textbf{76.0} &  75.7 / \textbf{79.9} &  84.3 / \textbf{87.4} &  89.3 / \textbf{92.8} \\
     & Imagenet (R) &  70.8 / \textbf{76.6} &  80.2 / \textbf{80.2} &  95.0 / \textbf{95.7} &  95.4 / \textbf{95.5} \\
     & Imagenet (C) &  72.5 / \textbf{78.2} &  78.7 / \textbf{80.2} &  92.1 / \textbf{93.6} &  93.2 / \textbf{93.6} \\
     & Uniform &  81.7 / \textbf{93.5} &  96.7 / \textbf{96.8} &    100 / \textbf{100} &    100 / \textbf{100} \\
     & Gaussian &  60.5 / \textbf{83.7} &  81.7 / \textbf{92.2} &    100 / \textbf{100} &    100 / \textbf{100} \\
     & SVHN &  73.2 / \textbf{82.9} &  88.1 / \textbf{89.0} &  89.5 / \textbf{92.6} &  91.8 / \textbf{92.7} \\
\midrule
\multirow{8}{*}{CIFAR-10} & iSUN &  85.0 / \textbf{90.4} &  86.9 / \textbf{92.0} &  98.2 / \textbf{99.1} &  98.8 / \textbf{99.0} \\
     & LSUN (R) &  85.3 / \textbf{90.8} &  87.1 / \textbf{92.4} &  98.7 / \textbf{99.3} &  99.1 / \textbf{99.2} \\
     & LSUN (C) &  86.2 / \textbf{90.0} &  87.2 / \textbf{88.7} &  92.8 / \textbf{95.6} &  95.7 / \textbf{97.2} \\
     & Imagenet (R) &  84.9 / \textbf{89.0} &  86.3 / \textbf{90.4} &  97.9 / \textbf{98.8} &  98.5 / \textbf{98.7} \\
     & Imagenet (C) &  85.3 / \textbf{89.4} &  86.3 / \textbf{89.9} &  96.3 / \textbf{97.7} &  97.5 / \textbf{98.3} \\
     & Uniform &  93.5 / \textbf{98.8} &  99.3 / \textbf{99.9} &    100 / \textbf{100} &    100 / \textbf{100} \\
     & Gaussian &  95.5 / \textbf{99.7} &   99.8 / \textbf{100} &    100 / \textbf{100} &    100 / \textbf{100} \\
     & SVHN &  85.1 / \textbf{90.3} &  89.1 / \textbf{93.0} &  96.8 / \textbf{98.1} &  98.1 / \textbf{98.4} \\
\midrule
\multirow{9}{*}{SVHN} & iSUN &  89.7 / \textbf{92.8} &  89.2 / \textbf{93.5} &  98.2 / \textbf{99.1} &  99.7 / \textbf{99.9} \\
     & LSUN (R) &  88.9 / \textbf{92.1} &  88.2 / \textbf{92.7} &  98.6 / \textbf{99.2} &  99.8 / \textbf{99.9} \\
     & LSUN (C) &  89.7 / \textbf{92.2} &  89.2 / \textbf{92.2} &  94.8 / \textbf{97.3} &  98.0 / \textbf{98.9} \\
     & Imagenet (R) &  90.4 / \textbf{93.4} &  90.0 / \textbf{94.2} &  98.0 / \textbf{99.1} &  99.5 / \textbf{99.8} \\
     & Imagenet (C) &  91.0 / \textbf{93.3} &  90.6 / \textbf{93.8} &  97.1 / \textbf{98.7} &  99.2 / \textbf{99.6} \\
     & Uniform &  92.9 / \textbf{95.7} &  92.3 / \textbf{97.4} &   99.9 / \textbf{100} &    100 / \textbf{100} \\
     & Gaussian &  92.9 / \textbf{95.4} &  93.0 / \textbf{97.5} &    100 / \textbf{100} &    100 / \textbf{100} \\
     & CIFAR-10 &  90.0 / \textbf{93.1} &  89.4 / \textbf{93.4} &  92.2 / \textbf{96.2} &  96.9 / \textbf{98.5} \\
     & CIFAR-100 &  89.6 / \textbf{92.5} &  89.0 / \textbf{93.1} &  92.4 / \textbf{96.1} &  97.0 / \textbf{98.5} \\
\bottomrule
\end{tabular}

}
\end{table}

\section{Additional out of distribution metrics}
\label{sec:single_layer_nn:ood_results}

The additional OOD metrics, TNR at 95\% FPR and Detection Accuracy, for the DensNet model are shown in \tableref{tab:single_layer_nn:TNRatTPR95_densenet} and \tableref{tab:single_layer_nn:DetectionAcc._densenet} respectively and for the ResNet are presented in \tableref{tab:single_layer_nn:TNRatTPR95_resnet} and \tableref{tab:single_layer_nn:DetectionAcc._resnet}.
We improve the compared methods for all IND-OOD sets except for 6 experiments of ODIN method with the TNR at 95\% metric.



\chapter{Appendix for chapter 7}
\section{Training parameters}
We now detail the training parameters used to train the different models.
\paragraph{MNIST.} For the standard model we used a network that consists of two convolutional layers with 32 and 64 filters respectively, each followed by $2\times2$ max-pooling, and a fully connected layer of size 1024. 
We trained the standard model for 100 epochs with natural training set. We used SGD with a learning rate of $0.01$, a momentum value of 0.9, a weight decay of 0.0001, and a batch size of 50.

\paragraph{CIFAR10.} We used wide-ResNet 28-10 architecture \citep{zagoruyko2016wide} for the standard and WRN models.
We trained the standard and WRN models over 204 epochs using SGD optimizer with a batch size of 128 and a learning rate of 0.001,
reducing it to 0.0001 and 0.00001 after 100 and 150 epochs respectively. We also used a momentum
value of 0.9 and a weight decay of 0.0002. WRN was trained using adversarial trainset generated by PGD attack on the natural trainset with 7 steps of $2/255$ with a total $\epsilon=8/255$.

\paragraph{ImageNet.} For the standard model we used a pre-trained ResNet50 \citep{he2016deep}. Similarly to the other models, we adjust the standard model to only output the first 100 logits.

\section{Alternative approximations for the adaptive attack}

The end-to-end flow of the adaptive defense is shown in \figref{fig:adversarial_defense:adv_pnml_scheme}: We denote $(x, y_1)$ as an input that belong to label $y_1$, $w_0$ is the model parameters, and $L(w_0,x,y_i)$ is the model loss w.r.t a specific label $y_i$ where $i \in [1,N]$. $x_{refine}^i$ is the refinement result for the $i$-th hypothesis and $p_i/C$ is the probability of the corresponding hypothesis. $L_{refine}$ is the loss for the first hypothesis. The adaptive adversary manipulate the input by taking steps in the direction of the gradients $\frac{\partial L_{refine}}{\partial x}$.

Recall that the adaptive attack approximates the refinement stage with a unity operator on the backward pass (see \secref{subsec:adversarial_defense:adaptive_adversary}). This approach, in effect, disregard anything that comes before the $sign(\cdot)$ operator during backpropagation. 
An alternative approach is to use some kind of a differentiable function to approximate the $sign(\cdot)$ operator and backpropagate through the entire computational graph. The first obstacle is to find a differentiable function that approximates the $sign(\cdot)$ operator well. The first option that comes to mind is to use a $sigmoid(\cdot)$ function, but since the input values are distributed across a wide range, the $sigmoid(\cdot)$ causes a vanishing gradients effect, which misses the goal of this approximation.

\begin{figure}[tb]
\centering
\includegraphics[width=1.0\linewidth]{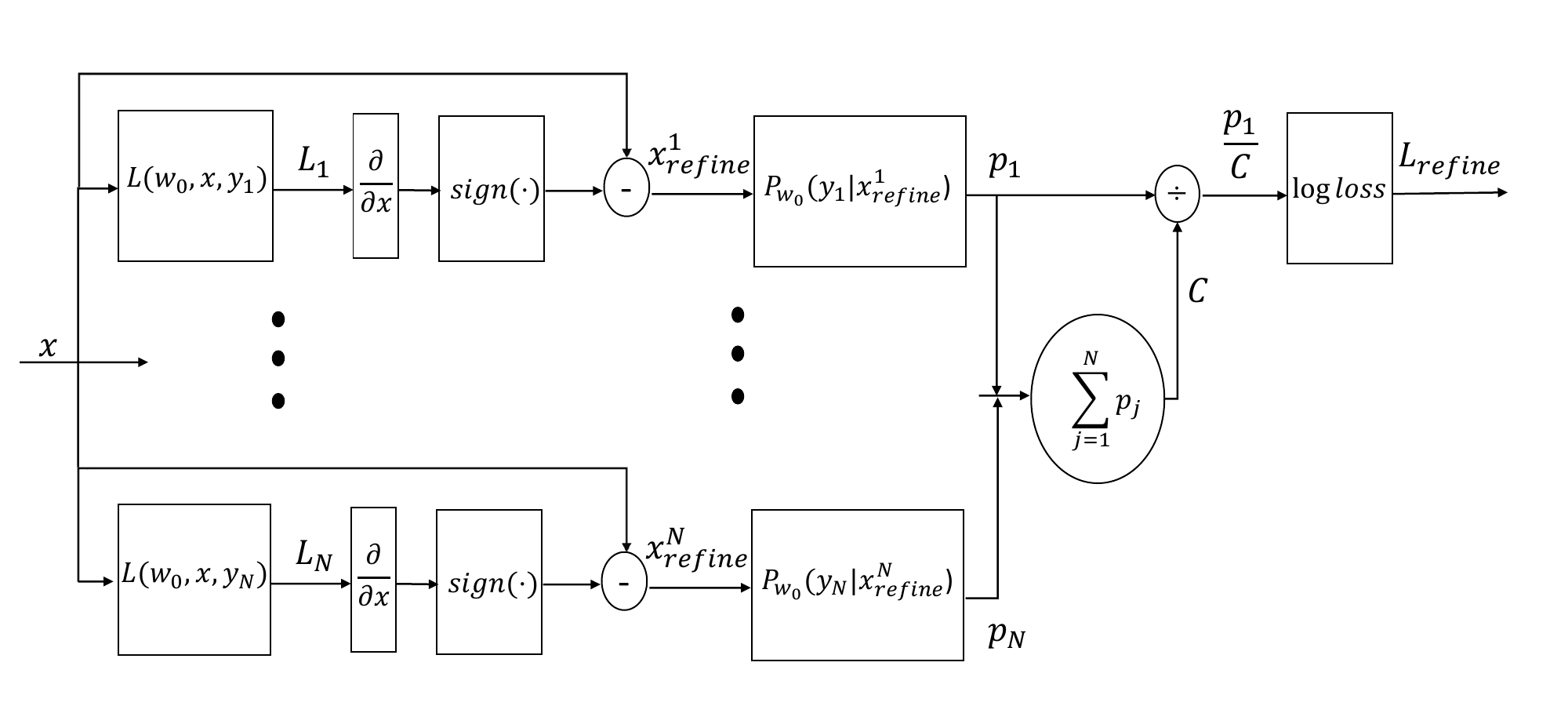}
\caption{End-to-end model illustration presenting label 1 hypothesis log-loss.}
\label{fig:adversarial_defense:adv_pnml_scheme}
\end{figure}

Another approach is to disregard the $sign(\cdot)$ operator on the backward pass, i.e., backpropagate the gradients without changing them. We examine this case:
\begin{equation}
    \frac{\partial L_{refine}}{\partial x} = \sum_{i=1}^{N} \frac{\partial log(\frac{p_1}{\sum_{j=1}^{N} p_j})}{\partial p_i} \frac{\partial p_i}{\partial x} = 
    \frac{1}{p_1}\frac{\partial p_1}{\partial x} - \frac{1}{\sum_{j=1}^{N} p_j}\sum_{i=1}^{N}\frac{\partial p_i}{\partial x},
\end{equation}
\begin{equation}
\frac{\partial p_i}{\partial x} = 
\frac{\partial p_{w_o}(y_i|x^i_{refine})}{\partial x} = 
\frac{\partial p_{w_o}(y_i|x^i_{refine})}{\partial x^i_{refine}}\frac{\partial x^i_{refine}}{\partial x},
\end{equation}

\begin{equation} \label{eq:adversarial_defense:hessian}
\frac{\partial x^i_{refine}}{\partial x} = 
\frac{\partial (x - \frac{\partial L_i}{\partial x})}{\partial x} =
1 - \frac{1}{\partial x}\left(\frac{\partial L_i}{\partial x}\right).
\end{equation}

Note that \eqref{eq:adversarial_defense:hessian} is dependent on the Hessian matrix of the loss $L_i$ w.r.t $x$. Computing this value is computationally hard for DNN's and it is usually outside the scope of adversarial robustness tests - only first-order adversaries are considered \citep{madry2017towards}. This emphasizes that the only viable, gradient-based, adaptive attack is the one used in our paper.

\begin{figure}[tb]
\centering
\includegraphics[width=0.55\linewidth]{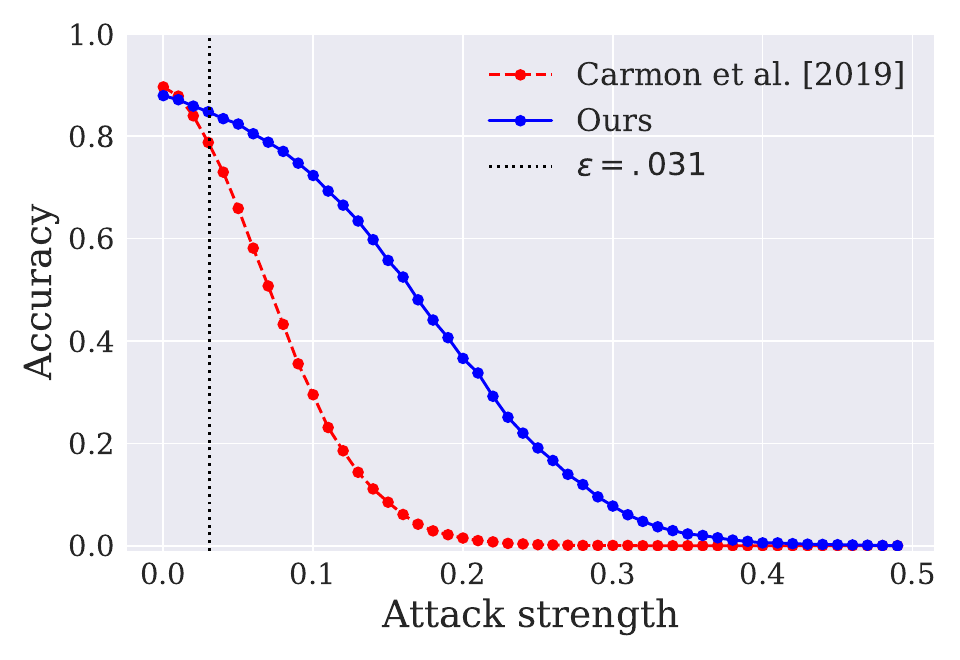}
\caption{CIFAR10 accuracy against HSJA attack}
\label{fig:adversarial_defense:cifar_hsja}
\end{figure}

\begin{table}[tb]
\small
\caption{Accuracy for the different number of iterations and step sizes}
\label{table:adversarial_defense:attack_stability}
\begin{subtable}[t]{0.4\textwidth}
\caption{Adaptive attack}
\label{table:adaptive_stability}
\flushright
\begin{tabular}{c|*{3}{c}}
\toprule
Iterations \textbackslash Step size & 0.005 & 0.01 & 0.02\\
\midrule
20 & 68.2 & 67.8 & 67.8\\
40 & 67.7 & \textbf{67.3} & 68.1\\
60 & 67.7 & 67.5 & 68.0\\
\bottomrule
\end{tabular}
\end{subtable}
\hspace{\fill}
\begin{subtable}[t]{0.48\textwidth}
\caption{PGD attack}
\label{table:PGD_stability}
\flushright
\begin{tabular}{c|*{3}{c}}
\toprule
Iterations \textbackslash Step size & 0.005 & 0.01 & 0.02 \\
\midrule
20 & 67.2 & 67.5 & 67.0\\
40 & 67.5 & 67.0 & 67.2\\
60 & 67.3 & \textbf{66.8} & 67\\
\bottomrule
\end{tabular}
\end{subtable}
\end{table}

\section{Additional CIFAR10 results}
\paragraph{HSJA for different $\epsilon$ values.} In figure \ref{fig:adversarial_defense:cifar_hsja} we demonstrate that our scheme is more robust against black-box HSJA for various $\epsilon$ values. Specifically, the maximal improvement is 49.1\% for $\epsilon=0.13$. We note that in comparison to the white-box attack, HSJA is much less efficient against our scheme for large $\epsilon$ values (see \figref{fig:adversarial_defense:cifar_acc_vs_attack_strength} for comparison). Nevertheless, the improvement of our scheme against black-box attack supports the claim that its robustness enhancement is not only the result of masked gradients.

\paragraph{Fine-tuned attacks.} In \tableref{table:adversarial_defense:attack_stability} we evaluate the accuracy for PGD and adaptive attacks with different hyperparameters.
 We evaluate the accuracy using 600 samples from the testset. For each attack we used $\epsilon=0.031$, 5 random starting points, 3 different values of number of iterations  $[20,40,60]$ and 3 different values of the step size  $[0.005, 0.01, 0.02]$, for a total of 9 hyperparameter combinations.

The best adaptive attack is obtained when the hyperparameters are the center values: the step is $0.01$ and number of iterations is $40$. 
For PGD attack we observe that increasing the number of iterations to 60 slightly improves the attack. Overall, the accuracy of both attacks remains roughly the same for different hyperparameters selection. This result indicates that the attack used to evaluate CIFAR10 in \secref{sec:adversarial_defense:experiment_results} is indeed strong since the robustness is stable across different combinations of hyperparameters.

\cleardoublepage
\newpage
\thispagestyle{empty}
\mbox{}

\includepdf[pages=-]{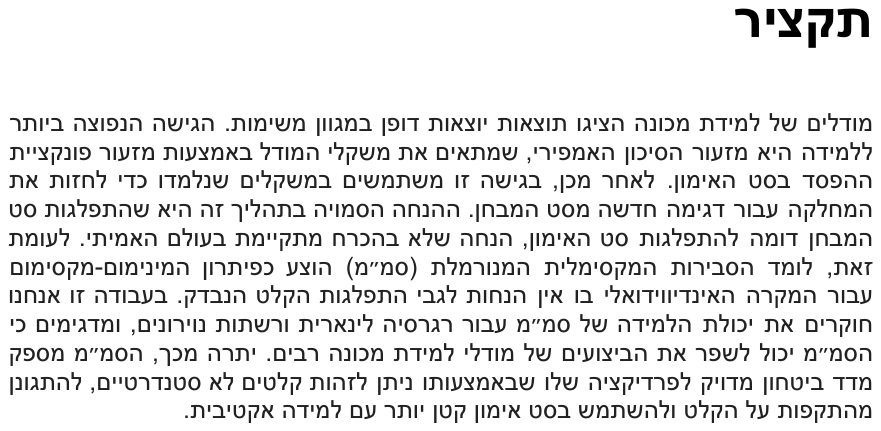}
\end{document}